\documentclass{article}

\PassOptionsToPackage{numbers}{natbib}
\usepackage[preprint]{neurips_2026}

\usepackage[utf8]{inputenc} 
\usepackage[T1]{fontenc}    
\usepackage{hyperref}       
\usepackage{url}            
\usepackage{booktabs}       
\usepackage{amsfonts}       
\usepackage{nicefrac}       
\usepackage{microtype}      
\usepackage[table]{xcolor}  
\usepackage{colortbl}
\usepackage{tikz}
\usepackage{amsmath}
\usepackage{amssymb}
\usepackage{mathtools}
\usepackage{amsthm}
\usepackage{enumitem}
\usepackage{graphicx}
\usepackage{subcaption}
\usepackage{caption}

\usepackage[capitalize,noabbrev,nameinlink]{cleveref}
\usepackage{float}          
\usepackage{pifont}
\usepackage{pdfpages}
\usepackage{multirow}
\usepackage{tabularx}
\usepackage{makecell}
\usepackage{threeparttable}
\usepackage{wrapfig}
\usepackage{tcolorbox}
\usepackage{comment}
\usepackage{rotating}

\usepackage{algorithm}
\usepackage{algpseudocode}  

\usepackage{amsmath,amsfonts,bm}

\def\eqref#1{equation~\ref{#1}}

\def\1{\bm{1}}

\DeclareMathAlphabet{\mathsfit}{\encodingdefault}{\sfdefault}{m}{sl}
\SetMathAlphabet{\mathsfit}{bold}{\encodingdefault}{\sfdefault}{bx}{n}

\newcommand{\KL}{D_{\mathrm{KL}}}

\newcommand{\ie}{\textit{i.e.}}
\newcommand{\eg}{\textit{e.g.}}
\newcommand{\myparagraph}[1]{\vspace{-0.12in}\paragraph{#1}}

\theoremstyle{plain}
\newtheorem{theorem}{Theorem}[section]
\newtheorem{proposition}[theorem]{Proposition}

\newtheorem{corollary}[theorem]{Corollary}
\theoremstyle{definition}
\newtheorem{definition}[theorem]{Definition}
\newtheorem{assumption}[theorem]{Assumption}
\theoremstyle{remark}

\definecolor{Red}{rgb}{0.768, 0.054, 0.054}
\definecolor{Blue}{rgb}{0.152, 0.294, 0.925}
\definecolor{Green}{rgb}{0,0.4,0.7}
\hypersetup{
    pdftitle={PCoreSet: Effective Active Learning through Knowledge Distillation from Vision-Language Models},
    pdfauthor={Seongjae Kang, Dong Bok Lee, Hyungjoon Jang, Dongseop Kim, Sung Ju Hwang},
    colorlinks=true,
    citecolor=teal,
    linkcolor=Red,
    urlcolor=Green,
}

\newcolumntype{L}{>{\columncolor{gray!10}}c}
\newcolumntype{G}{>{\columncolor{yellow!15}}c}
\newcolumntype{H}{>{\columncolor{green!15}}c}

\Crefname{table}{Tab.}{Tabs.}
\Crefname{figure}{Fig.}{Figs.}
\Crefname{equation}{Eq.}{Eqs.}
\Crefformat{section}{#2\S#1#3}
\Crefname{thm}{Theorem}{Theorems}
\Crefname{assumption}{Assumption}{Assumptions}
\Crefname{theorem}{Theorem}{Theorems}
\Crefname{algorithm}{Alg.}{Algs.}
\creflabelformat{equation}{#2#1#3}
\crefrangelabelformat{equation}{#3#1#4--#5#2#6}

\title{\textbf{\texttt{PCoreSet}}: Effective Active Learning through Knowledge Distillation from Vision-Language Models}

\author{%
  Seongjae Kang$^{\clubsuit,\dagger}$\quad
  Dong Bok Lee$^{\clubsuit,\dagger}$\quad
  Hyungjoon Jang$^{\spadesuit}$\\
  \textbf{Dongseop Kim}$^{\spadesuit}$\quad
  \textbf{Sung Ju Hwang}$^{\clubsuit,\diamondsuit}$\\
  $^{\clubsuit}$KAIST\quad $^{\spadesuit}$VUNO Inc.\quad $^{\diamondsuit}$DeepAuto.ai\\
  \texttt{\{tjdwo2744,markhi,sjhwang\}@kaist.ac.kr}\\
  \texttt{\{hyungjoon.jang,dongseop.kim\}@vuno.co}\\
  \footnotesize{$^{\dagger}$Equal contribution}
}

\begin{document}

\maketitle

\begin{abstract}

Knowledge distillation (KD) is a widely used framework for training compact, task-specific models by transferring the knowledge from teacher models.
However, its application to \textit{active learning} (AL), which aims to minimize annotation costs through iterative sample selection, remains underexplored.
This gap stems from the fact that KD typically assumes access to sufficient labeled data, whereas AL operates in data-scarce scenarios where task-specific teacher models are often unavailable.
In this paper, we first introduce \textbf{ActiveKD}, a framework that integrates AL with KD by leveraging the zero- and few-shot capabilities of large vision-language models (VLMs).
A key aspect of ActiveKD is the \textit{structured prediction bias} of VLMs---\ie, their predictions form clusters in the probability space.
We regard this structure as an inductive bias of the teacher model, capturing generalizable output patterns beneficial to student learning.
To exploit this bias, we propose \textbf{\texttt{P}}robabilistic \textbf{\texttt{CoreSet}} (\textbf{\texttt{PCoreSet}}), a selection strategy that maximizes coverage in the probability space rather than the feature space.
\textbf{\texttt{PCoreSet}} strategically selects probabilistically diverse unlabeled samples, facilitating more efficient transfer of teacher knowledge under limited annotation budgets.
Extensive evaluations on 11 datasets show that \textbf{ActiveKD} consistently improves every selection method over the no-distillation baseline (\eg, +29.07\% on ImageNet, averaged over methods). Under \textbf{ActiveKD}, several widely-used AL selectors \emph{underperform Random} and none consistently beats it; in contrast, \textbf{\texttt{PCoreSet}} is the only selector that consistently outperforms Random across 5 students and 3 teachers, always achieving the best performance except for the first 2 AL rounds. Our code is available at \url{https://github.com/erjui/PCoreSet}.


\end{abstract}






\vspace{-0.1in}
\section{Introduction}
\label{sec:introduction}


Recent advances in deep neural networks have focused on developing powerful generalist models capable of solving diverse tasks~\citep{achiam2023gpt, chatgpt, liu2023visual, liu2024improved} or creating robust transferable models through intensive pretraining~\citep{koroteev2021bert, radford2019language, he2016deep, dosovitskiy2020image, radford2021learning, jia2021scaling}.
However, real-world applications often necessitate training compact task-specific models, with a significant challenge of obtaining task-specific labeled data due to high annotation costs.
Active learning (AL) addresses this by iteratively selecting the most informative samples for oracle annotation~\citep{ren2021survey, li2024survey}, particularly in pool-based scenarios where informative samples can be identified from large unlabeled data.
The availability of unlabeled data allows semi-supervised learning (SSL) to be applied to AL settings during training~\citep{assran2021semi, cai2022semi, zheng2023simmatchv2, DHO}, accompanying researches at the intersection of AL and SSL~\citep{gao2020consistency, lim2023active, rangnekar2023semantic, singh2024semi}.

\begin{figure}[t]
    \centering
    \includegraphics[width=\textwidth]{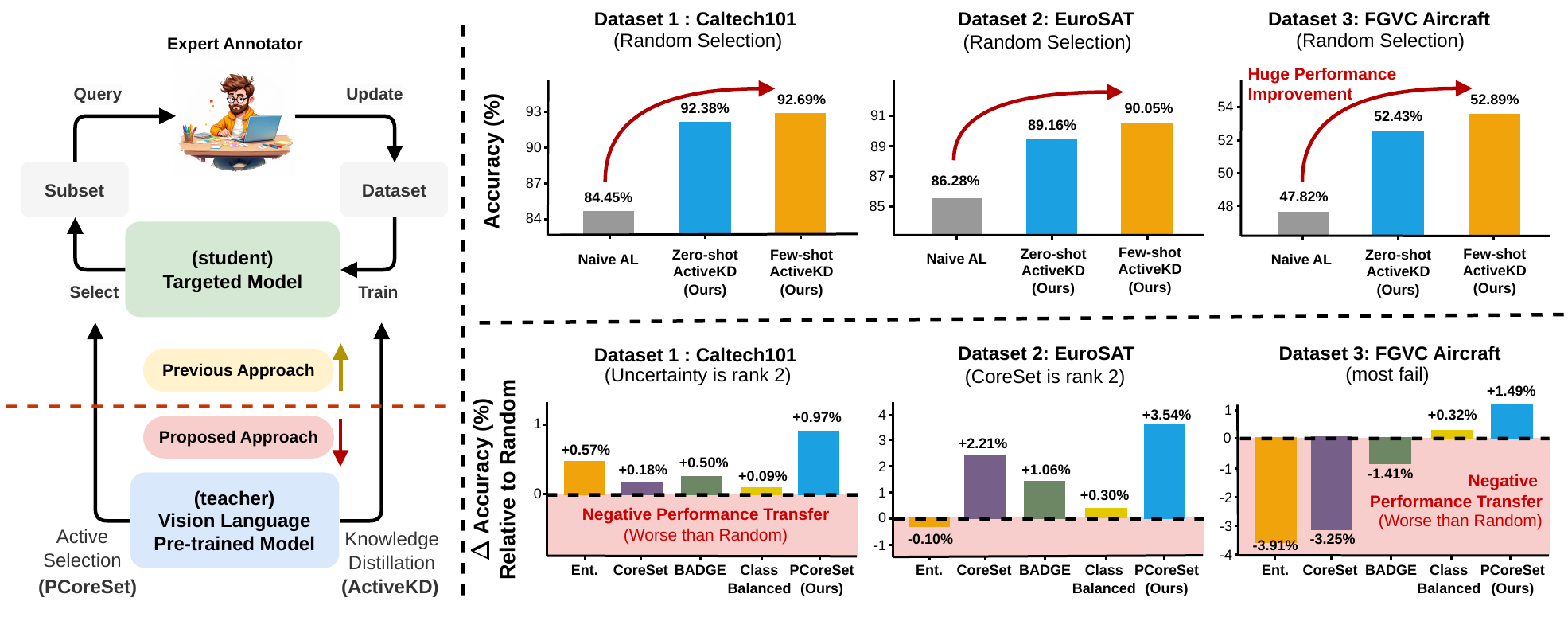}
    \vspace{-0.2in}
    \caption{\small
    \textbf{Overview of \textbf{ActiveKD} and \texttt{PCoreSet}.}
    \textbf{(Left)} Our two contributions: \textbf{ActiveKD}, a framework that integrates KD from VLM teachers into pool-based active learning, and \textbf{\texttt{PCoreSet}}, a selection strategy that maximizes coverage in the probability space.
    \textbf{(Top right)} \textbf{ActiveKD} substantially improves every selection method over the no-distillation baseline, so \textbf{ActiveKD} is beneficial \emph{regardless} of the active selector.
    \textbf{(Bottom right)} Under \textbf{ActiveKD}, (i) widely used AL selectors (\eg, Entropy, Coreset, BADGE) often \emph{underperform Random}, (ii) no existing selector is a clear winner, and (iii) \textbf{\texttt{PCoreSet}} \textbf{consistently outperforms all selectors across rounds}. Since all methods share an identical training pipeline and differ \emph{only} in data selection, this gain is significant by active-learning standards.
    Detailed numbers in \Cref{tab:performance_comparison} and \Cref{fig:main_results}.}
    \label{fig:problem}
    \vspace{-0.3in}
\end{figure}

Knowledge distillation (KD), meanwhile, is a widely used framework for training models by transferring knowledge from large pretrained teacher models~\citep{hinton2015distilling}.
Beyond its original purpose, KD can be naturally integrated into SSL framework as teacher models can provide its knowledge in the form of soft labels for unlabeled data~\citep{chen2020big, he2021semi, du2023semi, yang2025knowledge}.
However, the application of KD in AL remains \textit{underexplored}, mainly due to a mismatch in assumptions: AL acquires labels through selective annotation, while KD typically assumes powerful task-specific teacher models trained on sufficient labeled data---rarely available in the data-scarce settings common in AL.


Recently, vision-language models (VLMs) have emerged as powerful teachers in KD frameworks, as their rich representations---acquired through large-scale, unsupervised pretraining---enable effective knowledge transfer across both modalities and tasks~\citep{vemulapalli2023knowledge, wu2024cascade, mistretta2024improving}.
Crucially, VLMs are \emph{generalist} teachers---usable zero-shot without any task labels, or quickly adaptable from only a handful of labels---sidestepping the assumption mismatch noted above and making KD naturally compatible with the small-label regime of AL.
Recent work~\citep{DHO} has further demonstrated that transferring VLM knowledge into task-specific models is effective in label-scarce scenarios such as few-shot or low-shot semi-supervised settings; motivated by this, we introduce \textbf{ActiveKD}, a framework that integrates AL with KD by leveraging the zero- and few-shot capabilities of VLMs.
At each round, the student is trained on labeled data and unlabeled data with VLM soft predictions (see \Cref{fig:problem}).

In ActiveKD settings, we identify an aspect of VLM teachers:
they exhibit \textit{structured prediction biases} acquired during pretraining and shaped by language prompts~\citep{bang2024active}, resulting in predictions that form distinct clusters in the probability space as illustrated in \Cref{fig:activeKD}.
We observe that this bias propagates through distillation, influencing both active selection and student performance; rather than treating it as a limitation, we interpret this structure as an \textit{inductive bias} of the teacher---capturing generalizable output patterns that, complemented by labeled data, can benefit student learning.
To exploit this, we propose a \textit{simple and effective} method, \textbf{\texttt{P}}robabilistic \textbf{\texttt{CoreSet}} (\textbf{\texttt{PCoreSet}}), which selects coresets in the probability space rather than the feature space~\citep{sener2017active}.
Specifically, it maximizes diversity in the categorical probability space by targeting underrepresented regions with limited budgets.
Because KD makes the student inherit the teacher's structured bias (formalized in \Cref{prop:bias-propagation}), samples that fall outside the teacher's confident modes are precisely those the student cannot fully learn from distillation alone. By picking such deviating samples for annotation, \textbf{\texttt{PCoreSet}} supplies labeled supervision exactly where the distillation signal is weakest, accelerating the propagation of the teacher's full structure to the student (\Cref{fig:combined_pcoreset_teacher}).


To evaluate the empirical effectiveness of \textbf{ActiveKD} and \textbf{\texttt{PCoreSet}}, we conduct extensive experiments on 11 datasets spanning diverse domains and tasks detailed in \Cref{sec:datasets}.
As summarized in \Cref{fig:combined}, \textbf{ActiveKD consistently improves} performance across selection methods on all datasets (\eg, \textbf{+29.07\%} on ImageNet, averaged over methods), with further gains when using few-shot teachers.
Moreover, under ActiveKD, several widely-used AL selectors (\eg, Entropy, Coreset, BADGE) \textit{underperform Random} and no existing selector consistently beats it; in contrast, \textbf{\texttt{PCoreSet}} is \textbf{the only selector that consistently outperforms Random} across 5 students and 3 teachers, with either 8 or 16 rounds.

Our contributions and findings are summarized as follows:

\begin{itemize}
[itemsep=1mm,parsep=1pt,topsep=2pt,leftmargin=*]
\vspace{-0.05in}

\item We first introduce \textit{active knowledge distillation} (\textbf{ActiveKD}), a framework that integrates knowledge distillation into active learning by leveraging the zero- and few-shot capabilities of vision-language models to train efficient task-specific models with limited labeled data.

\item We identify \textit{structured prediction biases} in VLM teachers that are propagated to the student model after KD.
We thus propose \textbf{\texttt{P}}robabilistic \textbf{\texttt{CoreSet}} (\textbf{\texttt{PCoreSet}}), which utilizes this \textit{inductive bias} to select samples in the underrepresented regions of the probability space.

\item Extensive experiments on 11 datasets show that ActiveKD substantially boosts every selector (\eg, \textbf{+29.07\% on ImageNet}) and that \textbf{\texttt{PCoreSet}} is \textbf{the only selector that consistently beats Random} under ActiveKD across 5 students and 3 teachers---a gain significant by AL standards.
 
\vspace{-0.05in}
\end{itemize}

\section{Related Works}

\paragraph{Vision-Language Models (VLMs).}
Recent advances in machine learning have focused on training large foundation models that either transfer pretrained knowledge to specific tasks~\citep{he2016deep, dosovitskiy2020image, radford2021learning, jia2021scaling, radford2019language, koroteev2021bert} or directly apply to target tasks~\citep{wei2021finetuned, liu2023visual, liu2024improved, radford2021learning, jia2021scaling} via task instructions or zero-shot prediction using language prompts.
In this paper, we utilize VLMs~\citep{radford2021learning, jia2021scaling, silva2024closer} as teachers, which learn a shared embedding space for images and texts through a contrastive objective.
A key advantage of such language-aligned training is the ability to perform zero-shot predictions~\citep{radford2021learning, jia2021scaling} \textit{without task-specific training} and few-shot predictions~\citep{zhou2022learning, zhou2022conditional, khattak2023maple, zhu2023prompt, khattak2023self, zhao2024learning, roy2023consistency, zhang2024dept, lafon2025gallop, gao2024clip, zhang2021tip, yu2023task} \textit{with few examples}, qualifying VLMs as generalist visual recognition models.
Despite their strong task-agnostic capabilities, effectively leveraging task-specific data remains crucial for developing compact models that perform well on downstream tasks, especially under labeled-data scarcity.

\vspace{-0.02in}
\textbf{Active learning} \citep[\textbf{AL};][]{ren2021survey, zhan2022comparative, li2024survey} addresses labeled-data scarcity by strategically selecting the most informative samples from unlabeled data for annotation, maximizing model performance with minimal labels.
Previous approaches include uncertainty-based methods~\citep{lewis1995sequential, joshi2009multi, holub2008entropy, houlsby2011bayesian, gal2017deep, kirsch2019batchbald, rakesh2021efficacy}, diversity-based techniques~\citep{sener2017active, parvaneh2022active, yehuda2022active, hacohen2022active}, and hybrid approaches~\citep{ash2019deep, kirsch2019batchbald, hacohen2023select, giouroukis2025dual} that combine multiple criteria for sample selection.
Several studies have focused on the class imbalance problem~\citep{aggarwal2020active, bengar2022class, huang2024class}, where unbalanced datasets can lead selection algorithms to pick class-imbalanced samples.
Recent research has integrated AL with prompt tuning~\citep{lester2021power, jia2022visual, zhou2022conditional} of VLMs, leveraging foundational knowledge with efficient parameter updates~\citep{bang2024active, safaei2024active, kim2024active}.
PCB~\citep{bang2024active} addressed skewed predictions of VLMs by balancing class distributions, while CB+SQ~\citep{kim2024active} enhanced this with class-guided clustering and adaptive thresholds.
In contrast, we regard structured predictions as an \emph{inductive bias}: rather than balancing only the class distribution of selected labels~\citep{bang2024active, kim2024active}, we view the teacher's full \emph{probability distribution} as the structured signal to balance, yielding finer-grained coverage than label-frequency balancing alone.
Compared with feature-space coresets~\citep{sener2017active}, which operate on high-dimensional student representations that drift across AL rounds, our probability-space objective lives in the fixed-dimensional class simplex, whose axes are semantically anchored to the target categories regardless of how the teacher and student evolve.



\myparagraph{Knowledge Distillation from VLMs.}
Knowledge Distillation (KD) \citep{hinton2015distilling} transfers knowledge from large teachers to compact students.
With the emergence of vision foundation models through self-supervised learning \citep{chen2020simple, he2020momentum, grill2020bootstrap, chen2021exploring, caron2021emerging, he2022masked} and vision-language pretraining \citep{radford2021learning, jia2021scaling}, researchers have explored KD methods to leverage knowledge embedded within these models, moving beyond conventional KD that require training student models on identical datasets with substantial data.
Early work focused on self-supervised models \citep{fang2021seed, abbasi2020compress, xu2021bag, wang2022attention, navaneet2022simreg, singh2024simple}, while the advent of VLMs inspired further research, including distillation of compact VLMs from larger counterparts \citep{wu2023tinyclip, sun2023dime, udandarao2024active, vasu2024mobileclip, yang2024clip}, unsupervised distillation from VLM predictions \citep{vemulapalli2023knowledge, wu2024cascade, mistretta2024improving}, and few-shot semi-supervised distillation of VLMs~\citep{DHO}.
Some research has explored using teacher models instead of human annotation to address incorrect predictions \citep{baykal2022robust} or generate data efficiently \citep{wang2020neural, liu2024evolving}.
Our work focuses on human-in-the-loop scenarios, aiming to identify the most informative samples for annotation within a KD framework, by leveraging the zero- or few-shot capabilities of VLMs, to maximize learning efficiency under practical constraints of AL settings.
In contrast to fixed-set KD from VLMs~\citep{DHO}, AL introduces a sequential dynamic in which both the labeled subset and the student evolve jointly across rounds: each round's distillation must remain effective under fewer labels than the next, while each round's selection must anticipate the student that will be trained on the augmented set.
Our framework supports two teacher modes---a zero-shot teacher that stays fixed across rounds, and a few-shot teacher re-adapted at each round on the labels acquired so far---with a selection criterion tailored to the data distribution induced by ActiveKD in either case.

\vspace{-0.03in}

\begin{figure*}[t]
  \centering
  \includegraphics[width=0.85\textwidth]{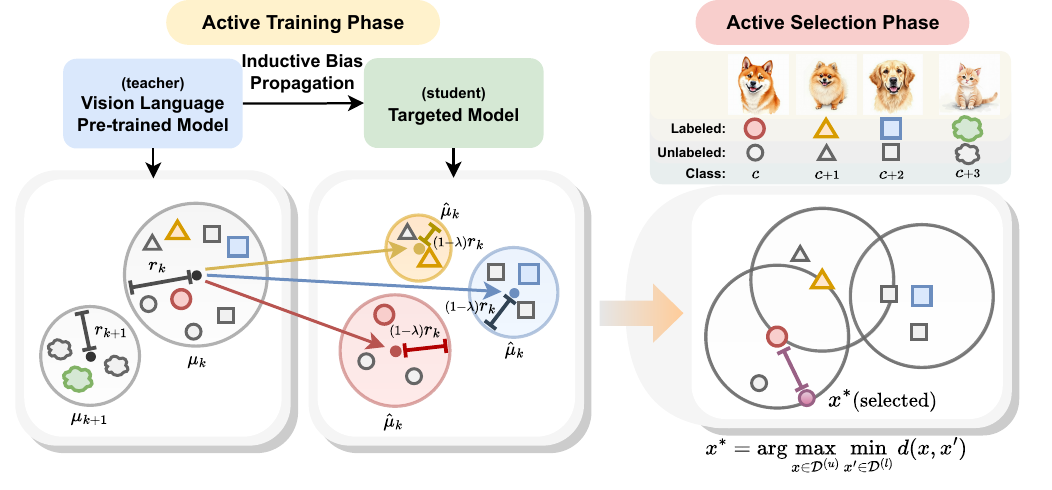}
  \vspace{-0.1in}
  \caption{
      \textbf{(Left):} Teacher model \textit{prediction biases} ($(\mu_1, r_1), \ldots, (\mu_k, r_k)$) are transferred to student models via distillation, where $\hat{\mu}_k=\mu_k+y_c$ and $\hat{r}_k=(1-\lambda)r_k$ ($\mu$ denotes centroids, $r$ denotes radii).
      \textbf{(Right):} \textbf{\texttt{PCoreSet}} selects samples maximizing distance to labeled points in probability simplex $\Delta^{C-1}$, uncovering underrepresented regions.
  }
  \vspace{-0.15in}
  \label{fig:activeKD}
\end{figure*}

\begin{figure*}[t]
\centering
\begin{minipage}[t]{0.51\textwidth}\vspace{0pt}%
{\setlength{\intextsep}{0pt}\setlength{\textfloatsep}{0pt}%
\begin{algorithm}[H]
\captionsetup{skip=3pt}
\small
\caption{Active\underline{KD} Framework}\label{alg:activekd}
\begin{algorithmic}[1]
\Require Labeled dataset $\mathcal{D}^{(l)}$, unlabeled pool $\mathcal{D}^{(u)}$, \underline{teacher VLM $f$}, a selection algorithm $A$, query size $Q$, number of rounds $R$
\Ensure Final model $f_r$
\For{$r = 1$ to $R$}
    \State Initialize the $r$-th round model $f_r$
    \State \underline{(Optional for few-shot teacher) adapt $f$ on $\mathcal{D}^{(l)}$}
    \State Train $f_r$ with $\mathcal{L}_{\text{CE}}$ in \Cref{eq:ce_loss} and \underline{$\mathcal{L}_{\text{KD}}$ in \Cref{eq:kd_loss}}
    \State $\{x_q^*\}_{q=1}^Q \leftarrow A(\mathcal{D}^{(l)}, \mathcal{D}^{(u)}, f_r; Q)$
    \State Obtain labels $\{y_q^*\}_{q=1}^Q$ for $\{x_q^*\}_{q=1}^Q$
    \State $\mathcal{D}^{(l)} \leftarrow \mathcal{D}^{(l)} \cup \{(x_q^*, y_q^*)\}_{q=1}^Q$
    \State $\mathcal{D}^{(u)} \leftarrow \mathcal{D}^{(u)} \setminus \{x_q^*\}_{q=1}^Q$
\EndFor
\State \Return $f_r$
\end{algorithmic}
\end{algorithm}}%
\end{minipage}\hfill
\begin{minipage}[t]{0.46\textwidth}\vspace{0pt}%
{\setlength{\intextsep}{0pt}\setlength{\textfloatsep}{0pt}%
\begin{algorithm}[H]
\captionsetup{skip=3pt}
\small
\caption{\textbf{\texttt{PCoreSet}} Algorithm}\label{alg:pcoreset}
\begin{algorithmic}[1]
\Require Labeled dataset $\mathcal{D}^{(l)}$, unlabeled pool $\mathcal{D}^{(u)}$, model $f_r$, query size $Q$, distance $d(x,x'):=\lVert f_r(x)-f_r(x')\rVert_2$
\Ensure Coreset $\mathcal{S}\subset\mathcal{D}^{(u)}$ with $\lvert\mathcal{S}\rvert=Q$
\State Initialize $\mathcal{S}\gets \mathcal{D}^{(l)}$
\State Initialize $D[x]\gets \min_{s\in\mathcal{S}} d(x,s)$ for all $x\in\mathcal{D}^{(u)}$
\While{$\lvert\mathcal{S}\rvert<Q$}
    \State $x^* \gets \arg\max_{x\in \mathcal{D}^{(u)}\setminus\mathcal{S}} D[x]$
    \State $\mathcal{S}\gets \mathcal{S}\cup\{x^*\}$
    \For{each $x\in \mathcal{D}^{(u)}\setminus\mathcal{S}$}
        \State $D[x]\gets \min\!\big(D[x],\, d(x,x^*)\big)$
    \EndFor
\EndWhile
\State \Return $\mathcal{S}$
\end{algorithmic}
\end{algorithm}}%
\end{minipage}
\vspace{-0.2in}
\end{figure*}

\section{Method}
\label{sec:method}

\subsection{Preliminaries}
\paragraph{Backgrounds on VLMs.}
We utilize Vision-Language Models (VLMs) like CLIP \citep{radford2021learning} and ALIGN \citep{jia2021scaling} as teacher models. These models jointly optimize an image encoder $f_\mathcal{X}: \mathcal{X} \rightarrow \mathbb{R}^d$ and a text encoder $f_\mathcal{T}: \mathcal{T} \rightarrow \mathbb{R}^d$ to map corresponding image-text pairs into a shared embedding space $\mathbb{R}^d$.
This cross-modal alignment enables zero-shot transfer through natural language supervision. For $C$-way classification tasks, we create textual prompts (\eg, ``a photo of a [CLASS]'') to generate class-specific text descriptions $\{t_1, t_2, \ldots, t_C\}$. The output probability vector of categorical distribution over $C$ classes is computed as the following:
\begin{equation}
f(x) = \sigma\Big(\tfrac{1}{\zeta}\big[\mathtt{CosSim}(f_\mathcal{X}(x), f_\mathcal{T}(t_1)), \ldots, \mathtt{CosSim}(f_\mathcal{X}(x), f_\mathcal{T}(t_C))\big]^\top\Big) \in \Delta^{C-1}.
\label{eq:teacher}
\end{equation}
Here, $\Delta^{C-1}$ is the $(C{-}1)$-dimensional probability simplex, $\sigma: \mathbb{R}^C \to \Delta^{C-1}$ denotes the softmax function, and $\zeta \in \mathbb{R}_{>0}$ is the temperature scaling factor~\citep{hinton2015distilling}. Cosine similarity is defined as $\mathtt{CosSim}(x, y) = \frac{x^\top y}{\|x\|_2 \|y\|_2}$.
The final predicted class is given by $\arg\max_{c \in \{1,\ldots,C\}} [f(x)]_c$.

\vspace{-0.03in}
\myparagraph{Problem formulation.}
In this paper, we consider a pool-based active learning (AL), a framework for building a $C$-way classifier $f_r:\mathcal{X}\mapsto\Delta^{C-1}$ for each round $r\in\{1,\ldots,R\}$ while minimizing annotation costs.
We start with a labeled dataset $\mathcal{D}^{(l)} = \{(x_n^{(l)}, y_n)\}_{n=1}^{N}$ and an unlabeled dataset $\mathcal{D}^{(u)} = \{x_m^{(u)}\}_{m=1}^{M}$, where $y_n\in\{0,1\}^C$ is the one-hot encoding label of $x_n$ and typically $N \ll M$.
We first train a model $f_r$ using the current labeled dataset $\mathcal{D}^{(l)}$ for each round $r$.
Then, we select subset of unlabeled data that will be requested for annotation, \ie, $\{x_q^*\}_{q=1}^{Q} \leftarrow A(\mathcal{D}^{(l)},\mathcal{D}^{(u)}, f_r;Q)\subset\mathcal{D}^{(u)}$, where $Q$ is the number of query datapoints and $A$ is an algorithm for selection.
These selected queries are then annotated by an oracle (typically human experts) to obtain their true labels $\{y_q^*\}_{q=1}^Q$.
The labeled dataset is updated as $\mathcal{D}^{(l)} \leftarrow \mathcal{D}^{(l)} \cup \{(x_q^*, y_q^*)\}_{q=1}^Q$, and the unlabeled pool is reduced by $\mathcal{D}^{(u)} \leftarrow \mathcal{D}^{(u)} \setminus \{x_q^*\}_{q=1}^Q$. See \Cref{alg:activekd} without \underline{underlines} for an overview.


\begin{figure*}[t]
    \centering
    \begin{subfigure}[b]{0.24\textwidth}
        \centering
        \includegraphics[width=\textwidth]{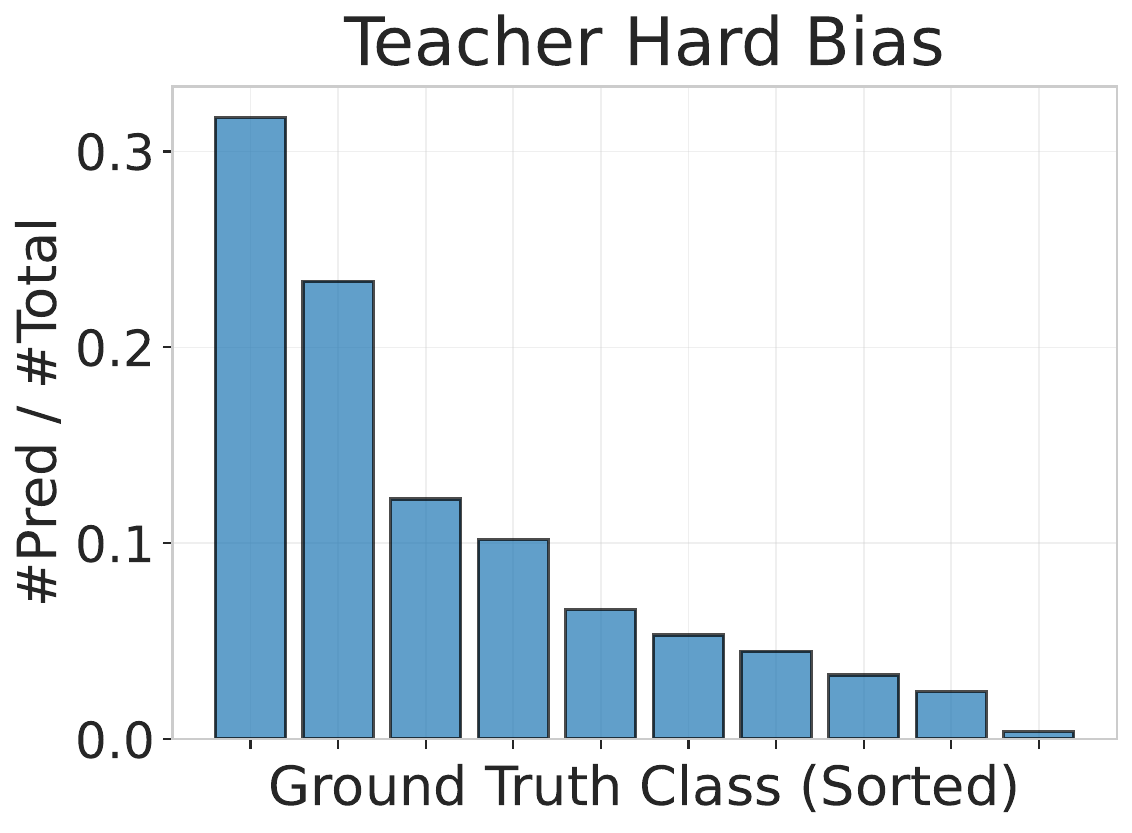}
        \vspace{-0.2in}
        \caption{}
        \label{fig:teacher_hard_bias}
    \end{subfigure}
    \hfill
    \begin{subfigure}[b]{0.45\textwidth}
        \centering
        \includegraphics[width=\textwidth]{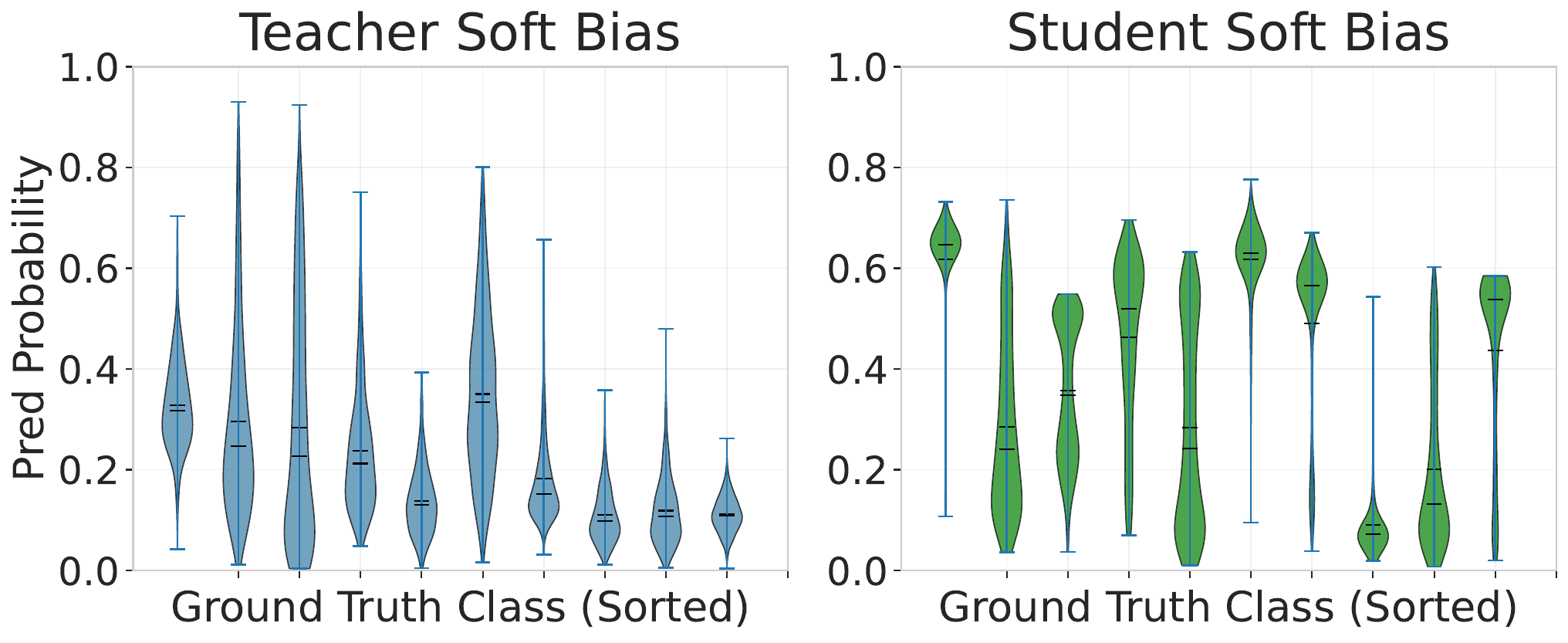}
        \vspace{-0.2in}
        \caption{}
        \label{fig:soft_bias}
    \end{subfigure}
    \hfill
    \begin{subfigure}[b]{0.26\textwidth}
        \centering
        \includegraphics[width=\textwidth]{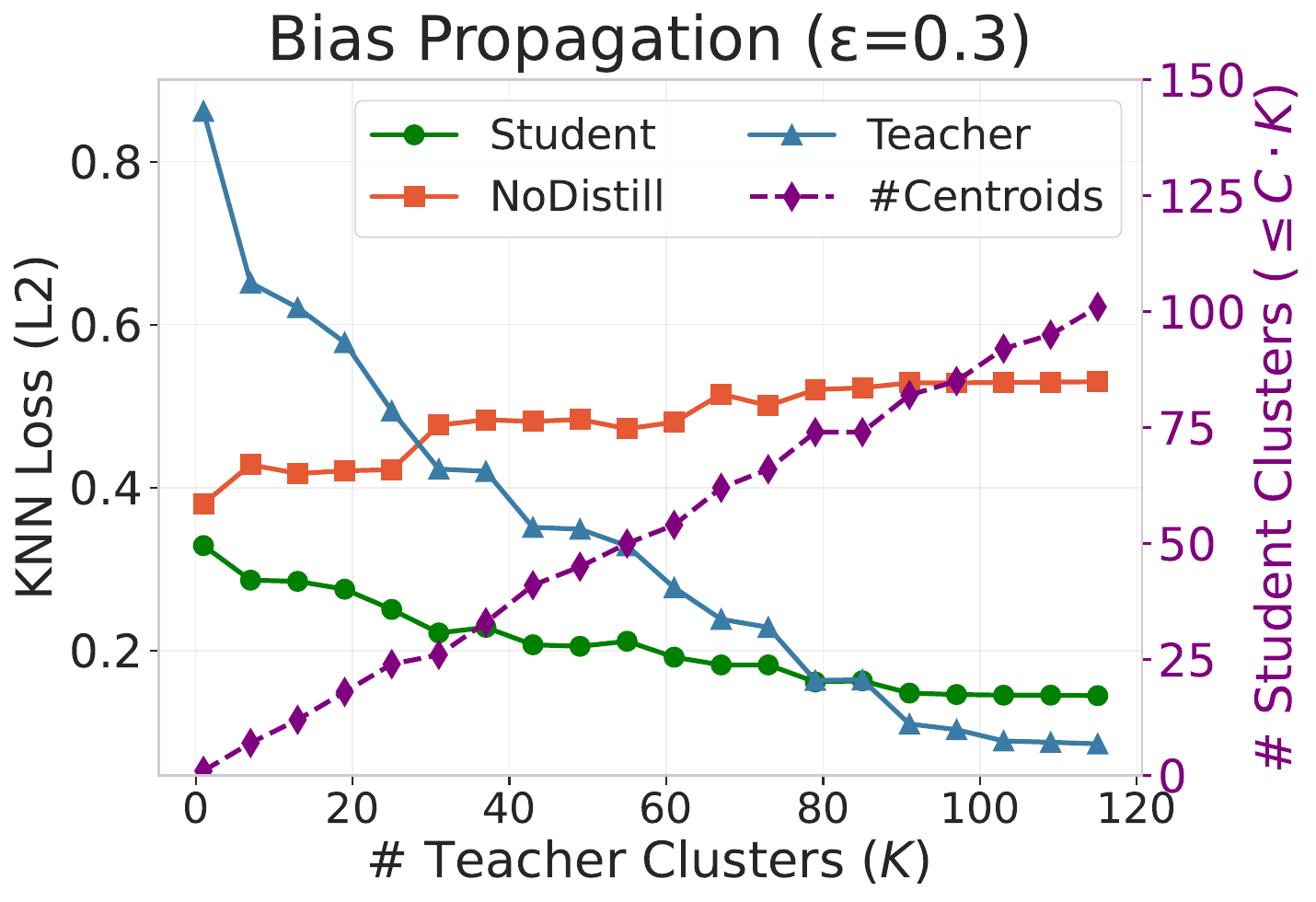}
        \vspace{-0.2in}
        \caption{}
        \label{fig:propagation}
    \end{subfigure}
    \vspace{-0.05in}
    \caption{
        \textbf{Visualization} of prediction bias and its propagation.
        \textbf{(a)}: \textbf{hard prediction bias} of the teacher; \textbf{(b)}: \textbf{soft prediction bias} of the \textbf{teacher (left)} and the \textbf{student (right)}  after distillation; and
        \textbf{(c)}: \textbf{the bias propagated from the teacher to the student}, quantified by KNN loss ($\ell_2$) across different numbers of clusters ($K$).
    }
    \label{fig:teacher_bias}
    \vspace{-0.15in}
\end{figure*}


\vspace{-0.05in}
\subsection{ActiveKD}

\paragraph{Active knowledge distillation (ActiveKD) framework.}
In this paper, we propose the \textbf{ActiveKD} framework that leverages knowledge distillation into active learning.
Specifically, we train the student model for each round using both supervised learning and knowledge distillation~\citep{hinton2015distilling, DHO}:
\begin{align}
\mathcal{L}_{\text{CE}} &= \tfrac{1}{N}\sum_{n} \ell\left(f_r(x_n), y_n\right), \label{eq:ce_loss}\\
\mathcal{L}_{\text{KD}} &= \tfrac{1}{N}\sum_{n} \KL\left[ f(x_n^{(l)}) \Vert f_r(x_n^{(l)})\right] + \tfrac{1}{M} \sum_{m} \KL\left[ f(x_m^{(u)}) \Vert  f_r(x_m^{(u)})\right],\label{eq:kd_loss}
\end{align}
where the final loss is \( \lambda\mathcal{L}_{\text{CE}} + (1 - \lambda)\mathcal{L}_{\text{KD}} \), $\ell$ denotes $\KL$ represent 
cross-entropy and Kullback-Leibler divergence, respectively.
For each round $r$, we use a mini-batch version of the above objective with stochastic gradient descent to optimize the parameters of $f_r$ by leveraging the teacher prediction $f(\cdot)$.
See \underline{underlines} in \Cref{alg:activekd} for additional parts of ActiveKD.


\myparagraph{Structured prediction bias propagation.}

VLMs \citep[\eg, CLIP ResNet-50;][]{radford2021learning} inherently exhibit class imbalance~\citep{bang2024active} due to their pre-training data and prompt design, as shown in \Cref{fig:teacher_hard_bias}.
Beyond imbalance, we observe that \textbf{teacher predictions form distinct clusters within constrained regions of the probability simplex}.
The violin plots in \Cref{fig:soft_bias}-\textbf{(left)} show the distributions for ground-truth (GT) classes, indicating that teacher predictions concentrate in specific regions rather than uniformly covering the space, creating \emph{blind spots}.
This \textbf{clustering behavior propagates to student models during knowledge distillation}, as students inherit similar biased patterns (\Cref{fig:soft_bias}-\textbf{(right)}).
To formalize this, we define the \textit{structured prediction bias} as follows:

\begin{definition}[Structured prediction bias]\label{def:bias}
Teacher predictions exhibit \textit{structured prediction bias} if there exist $K \in \mathbb{N}$, centroids $\{\mu_k\}_{k=1}^{K} \subset \Delta^{C-1}$, and radii $\{r_k\}_{k=1}^{K} \subset \mathbb{R}_{>0}$, such that:
\begin{equation}
\forall x \in \mathcal{X}, \quad f(x) \in \bigcup_{k=1}^{K} \left\{ p \in \Delta^{C-1} : \|p - \mu_k\|_2 \leq r_k \right\}.
\end{equation}
\end{definition}
\vspace{-0.07in}

In other words, all teacher predictions lie within a finite union of $\ell_2$-balls in the probability simplex.
We now show that \textbf{student predictions trained via KD also exhibit similar structured bias}:

\begin{proposition}[Bias propagation through KD]\label{prop:bias-propagation}
Let the teacher model $f$ exhibit \textit{structured prediction bias} as defined in \Cref{def:bias}, with $\{\mu_k\}_{k=1}^K$ and $\{r_k\}_{k=1}^K$.
Assume the student model $f_r$ is trained via KD from $f$ using the loss
$\lambda \mathcal{L}_{\text{CE}} + (1 - \lambda) \mathcal{L}_{\text{KD}}$, and satisfies
$\|f_r(x) - f^*(x)\|_1 \leq \epsilon$ for all $x \in \mathcal{X}$, where $f^*(x) = \lambda y + (1-\lambda) f(x)$ and $y \in \{0,1\}^C$ denotes the label of $x$.
Then student predictions $f_r(x)$ also exhibit \textit{structured prediction bias} defined in \Cref{def:bias}.
Specifically, $\forall x \in \mathcal{X}$, there exists $k \in \{1,\ldots,K\}$ such that:
\begin{equation}
f_r(x) \in \left\{ p \in \Delta^{C-1} : \|p - \hat{\mu}_k(x)\|_2 \leq \hat{r}_k \right\},
\end{equation}
where the propagated centriod is defined as $\hat{\mu}_k(x) = \lambda y + (1 - \lambda)\mu_k$,
and the adjusted radius is $\hat{r}_k = (1 - \lambda)r_k + \epsilon$.
Since $y$ is one-hot and $\mu_k$ is fixed, the set of possible centers $\{\hat{\mu}_k(x)\}$ is finite and contained in $\Delta^{C-1}$,
thus satisfying the condition of \Cref{def:bias} with at most $C \cdot K$ clusters.
\end{proposition}

\vspace{-0.05in}

We defer the proof of \Cref{prop:bias-propagation} to \Cref{sec:theoretical_analysis_bias}. \Cref{prop:bias-propagation} is illustrated in \Cref{fig:activeKD}-(left), where a teacher cluster $(\mu_k, r_k)$ propagates to corresponding student clusters $(\hat{\mu}_k, \hat{r}_k)$ across three classes.

\myparagraph{Empirical validation of \Cref{prop:bias-propagation}.}
We compare:
1) a baseline model using $\mathcal{L}_{\text{CE}}$ on $\mathcal{D}^{(l)}$; and
2) a student model trained via KD with loss $\lambda \mathcal{L}_{\text{CE}} + (1 - \lambda) \mathcal{L}_{\text{KD}}$ on $\mathcal{D}^{(l)}$ and $\mathcal{D}^{(u)}$, under the 1-shot setting.
We cluster teacher predictions on $\mathcal{D}^{(u)}$ via $k$-means~\citep{kmeans}, yielding centroids $\{\mu_k\}_{k=1}^K$.
For each $x \in \mathcal{D}^{(u)}$, we compute the propagated student centroid $\hat{\mu}_k = \lambda y + (1-\lambda)\mu_k$, where $y$ is the one-hot GT label, and measure the average $\ell_2$ distance between $f_r(x)$ and $\hat{\mu}_k$ (threshold $\epsilon = 0.3$).
As shown in \Cref{fig:propagation}, the \textbf{distilled student shows consistently lower distances} than the baseline across varying $K$, \textbf{validating \Cref{prop:bias-propagation}}.
The number of active propagated centroids also closely matches that of the teacher, further supporting the existence of \textit{structured prediction bias} of student model.

\subsection{\textbf{\texttt{P}}robabilistic \textbf{\texttt{C}}oreset (\textbf{\texttt{PCoreSet}})}

\paragraph{Motivation.}
During KD, the student model leverages both labels and teacher guidance, to generalize rather than merely mimicking teachers.
As a result, the \textbf{prediction structure of the teacher acts as a prior} that shapes the hypothesis space of the student.
Our key insight is that when KD is effective (with bounded error $\epsilon$), the student inherits the clustered structure of teacher (\Cref{prop:bias-propagation}).
Samples that deviate from this structure are particularly informative, as they fall outside regions captured by \textit{inductive bias}.
Rather than discarding this structure, we \textbf{leverage it to guide AL}: selecting samples that deviate from the established structure and expand the student's predictive capacity.

\vspace{-0.02in}

\myparagraph{\texttt{PCoreSet} selection.} 

To this end, we propose \textbf{\texttt{P}}robabilistic \textbf{\texttt{C}}oreset (\textbf{\texttt{PCoreSet}}), a \textit{simple yet effective} selection strategy inspired by coreset selection~\citep{sener2017active}.
While conventional coreset methods aim to maximize coverage in the feature space, \textbf{\texttt{PCoreSet}} instead maximizes coverage in the probability simplex $\Delta^{C-1}$ by targeting underrepresented probability regions.
This enables the student to inherit the teacher's \textit{inductive bias} more completely while actively exploring beyond it.
Formally, given a labeled dataset $\mathcal{D}^{(l)}$, we greedily select $x^* \in \mathcal{D}^{(u)}$ as:
\begin{equation}
x^* = \underset{x \in \mathcal{D}^{(u)}}{\arg\max} \; \underset{x' \in \mathcal{D}^{(l)}}{\min} \; d(x, x'),
\end{equation} where $d(x, x') := \|f_r(x) - f_r(x')\|_2$ measures distance in the probability space. We then request the label $y^*$ for $x^*$ and update the labeled set: $\mathcal{D}^{(l)} \leftarrow \mathcal{D}^{(l)} \cup \{(x^*, y^*)\}$. See \Cref{alg:pcoreset} for details.

\vspace{-0.02in}
\myparagraph{Error of teacher bias.} Although $\texttt{PCoreSet}$ maximally leverages the structured bias of VLM teachers, which can be \textbf{erroneous} in zero-/few-shot settings, any resulting mistakes can be corrected through the \textbf{GT labels} ($\mathcal{D}^{(l)}$), ensuring reliable learning under ActiveKD.

\vspace{-0.02in}
\textbf{The computational complexity} of \textbf{\texttt{PCoreSet}} is \(\mathcal{O}(C\cdot M\cdot N)\), whereas the feature-space coreset~\citep{sener2017active} has complexity \(\mathcal{O}(H\cdot M\cdot N)\), where \(H\) is the feature dimensionality. Thus, in most cases (\(C \ll H\)), \textbf{\texttt{PCoreSet}} is more efficient.
While \textbf{\texttt{PCoreSet}} is less efficient than the \(\mathcal{O}(C\cdot M)\) complexity of \textbf{Entropy}~\citep{holub2008entropy}, Entropy exhibits degradation as discussed in \Cref{sec:experimental_results}.
\section{Experiments}
\label{sec:experiments}







\subsection{Experimental Setup}
\label{sec:experiments_setup}

\paragraph{Datasets.}
We evaluate our approach across 11 diverse datasets: generic image recognition benchmarks~\citep{russakovsky2015imagenet,fei2004learning}, and fine-grained and domain-specific benchmarks~\citep{krause20133d,nilsback2008automated,maji2013fine,bossard2014food,xiao2010sun,cimpoi2014describing,helber2019eurosat,soomro2012ucf101}. To reduce computational cost, we subsample the unlabeled ImageNet~\citep{russakovsky2015imagenet} pool to 100,000 images. See more details in \cref{sec:datasets}.

\myparagraph{Active learning frameworks.}
1) \textbf{No Distill}: standard AL without KD, where the model is trained only on labeled data $\mathcal{D}^{(l)}$ using cross-entropy loss $\mathcal{L}_{\text{CE}}$;
2) \textbf{ActiveKD (Zero-Shot)}: our proposed \textbf{ActiveKD} framework with a fixed zero-shot VLM teacher~\citep{radford2021learning} that provides soft targets for KD on both $\mathcal{D}^{(l)}$ and $\mathcal{D}^{(u)}$; and
3) \textbf{ActiveKD (Few-Shot)}: as in 2) \textbf{ActiveKD (Zero-Shot)} but with a few-shot VLM teacher~\citep{silva2024closer} fine-tuned on $\mathcal{D}^{(l)}$ using both the newly selected samples and existing labeled data.




\myparagraph{Active selection baselines.}
We compare \textbf{\texttt{PCoreSet}} with five baselines. 1) \textbf{Random}: uniform sampling from the unlabeled pool $\mathcal{D}^{(u)}$; 2) \textbf{Entropy}~\citep{holub2008entropy}: selects points with highest predictive entropy; 3) \textbf{CoreSet}~\citep{sener2017active}: maximizes diversity in feature space; 4) \textbf{BADGE}~\citep{ash2019deep}: combines uncertainty and diversity via gradient embeddings; and 5) \textbf{ClassBalanced}: our reproduction of PCB~\citep{bang2024active}, which promotes class diversity in the queried set under pseudo-labels.

\myparagraph{Implementation details.}
All AL methods share the \textbf{same experimental setup} when evaluated under the same AL framework:
We use a query size of $Q=C$ per round (1-shot), with $R=16$ rounds for all datasets except ImageNet, which uses $R=8$.
We adopt DHO~\citep{DHO} as the KD method in all experiments to leverage the unlabeled pool.
For ImageNet, we use a self-supervised ResNet-50~\citep{caron2021emerging} as the student model. For the other 10 datasets, we use ResNet-18~\citep{he2016deep}, MobileNetV2~\citep{sandler2018mobilenetv2}, TinyViT~\citep{wu2022tinyvit}, and ViT-B/16~\citep{dosovitskiy2020image}, all pretrained on ImageNet.
We use CLIP ResNet-50~\citep{radford2021learning}, OpenCLIP~\citep{cherti2023reproducible}, and ViT-L/14 as the zero-shot teacher and CLAP~\citep{silva2024closer} as the few-shot teacher.
We exclude BADGE~\citep{ash2019deep} on ImageNet due to memory limitations when computing gradients over its 1,000 classes.
We report the mean performance with 95\% confidence intervals across 5 random seeds.
We defer further details to \Cref{sec:implementation}.



\subsection{Main Results}
\label{sec:experimental_results}

\begin{table*}[t!]
    \centering
    \caption{\textbf{Final-round accuracy on ImageNet and the 10-dataset average under three AL frameworks.} Mean and 95\% CI over 5 seeds; gray values in parentheses denote the gain over No Distill. \textbf{\texttt{PCoreSet}} consistently exceeds Random under \textbf{ActiveKD} (+0.47 on ImageNet, +1.13 on the 10-dataset avg.).}
    \label{tab:performance_comparison}
    \vspace{-0.1in}
    \scriptsize
    \setlength{\tabcolsep}{4pt}
    \resizebox{\textwidth}{!}{%
    \begin{tabular}{l ccc | ccc}
        \toprule
        & \multicolumn{3}{c|}{\textit{ImageNet (ResNet-50 student \& ResNet-50 teacher)}} & \multicolumn{3}{c}{\textit{Avg. over 10 Datasets (ResNet-18 student \& ResNet-50 teacher)}} \\
        \cmidrule(lr){2-4} \cmidrule(lr){5-7}
        \textbf{Methods} & \textbf{No Distill} & \textbf{ActiveKD (Zero-Shot)} & \textbf{ActiveKD (Few-Shot)} & \textbf{No Distill} & \textbf{ActiveKD (Zero-Shot)} & \textbf{ActiveKD (Few-Shot)} \\
        \midrule
        Random        & 33.36{\tiny$\pm$0.45} & 60.69{\tiny$\pm$0.16}\,\textcolor{gray}{\scriptsize (+27.33)} & 60.49{\tiny$\pm$0.12}\,\textcolor{gray}{\scriptsize (+27.13)}
                      & 63.10{\tiny$\pm$4.33} & 76.31{\tiny$\pm$0.88}\,\textcolor{gray}{\scriptsize (+13.21)} & 77.48{\tiny$\pm$0.86}\,\textcolor{gray}{\scriptsize (+14.38)} \\
        Entropy       & 30.76{\tiny$\pm$0.24} & 60.43{\tiny$\pm$0.16}\,\textcolor{gray}{\scriptsize (+29.67)} & 58.87{\tiny$\pm$0.14}\,\textcolor{gray}{\scriptsize (+28.11)}
                      & 64.06{\tiny$\pm$3.47} & 75.58{\tiny$\pm$0.86}\,\textcolor{gray}{\scriptsize (+11.52)} & 77.47{\tiny$\pm$0.90}\,\textcolor{gray}{\scriptsize (+13.41)} \\
        Coreset       & 26.61{\tiny$\pm$0.90} & 60.58{\tiny$\pm$0.07}\,\textcolor{gray}{\scriptsize (+33.97)} & 59.01{\tiny$\pm$0.42}\,\textcolor{gray}{\scriptsize (+32.40)}
                      & \textbf{64.99}{\tiny$\pm$3.01} & 76.20{\tiny$\pm$0.91}\,\textcolor{gray}{\scriptsize (+11.21)} & 77.63{\tiny$\pm$0.57}\,\textcolor{gray}{\scriptsize (+12.64)} \\
        BADGE         & -- & -- & --
                      & 63.14{\tiny$\pm$4.66} & 76.18{\tiny$\pm$0.74}\,\textcolor{gray}{\scriptsize (+13.04)} & 77.50{\tiny$\pm$0.67}\,\textcolor{gray}{\scriptsize (+14.36)} \\
        ClassBalanced & \textbf{34.44}{\tiny$\pm$0.60} & 61.07{\tiny$\pm$0.11}\,\textcolor{gray}{\scriptsize (+26.63)} & 61.13{\tiny$\pm$0.07}\,\textcolor{gray}{\scriptsize (+26.69)}
                      & 63.57{\tiny$\pm$4.33} & 76.28{\tiny$\pm$0.87}\,\textcolor{gray}{\scriptsize (+12.71)} & 77.56{\tiny$\pm$0.66}\,\textcolor{gray}{\scriptsize (+13.99)} \\
        \rowcolor{gray!12}
        \textbf{\texttt{PCoreSet} (ours)}
                      & 33.41{\tiny$\pm$0.41} & \textbf{61.16}{\tiny$\pm$0.14}\,\textcolor{gray}{\scriptsize (+27.75)} & \textbf{61.57}{\tiny$\pm$0.24}\,\textcolor{gray}{\scriptsize (+28.16)}
                      & 64.94{\tiny$\pm$3.25} & \textbf{77.44}{\tiny$\pm$0.84}\,\textcolor{gray}{\scriptsize (+12.50)} & \textbf{78.81}{\tiny$\pm$0.42}\,\textcolor{gray}{\scriptsize (+13.87)} \\
        \bottomrule
    \end{tabular}
    }
    \vspace{-0.1in}
\end{table*}

\begin{figure*}[t!]
    \centering
    \begin{subfigure}[b]{0.56\textwidth}
        \centering
        \includegraphics[width=\textwidth]{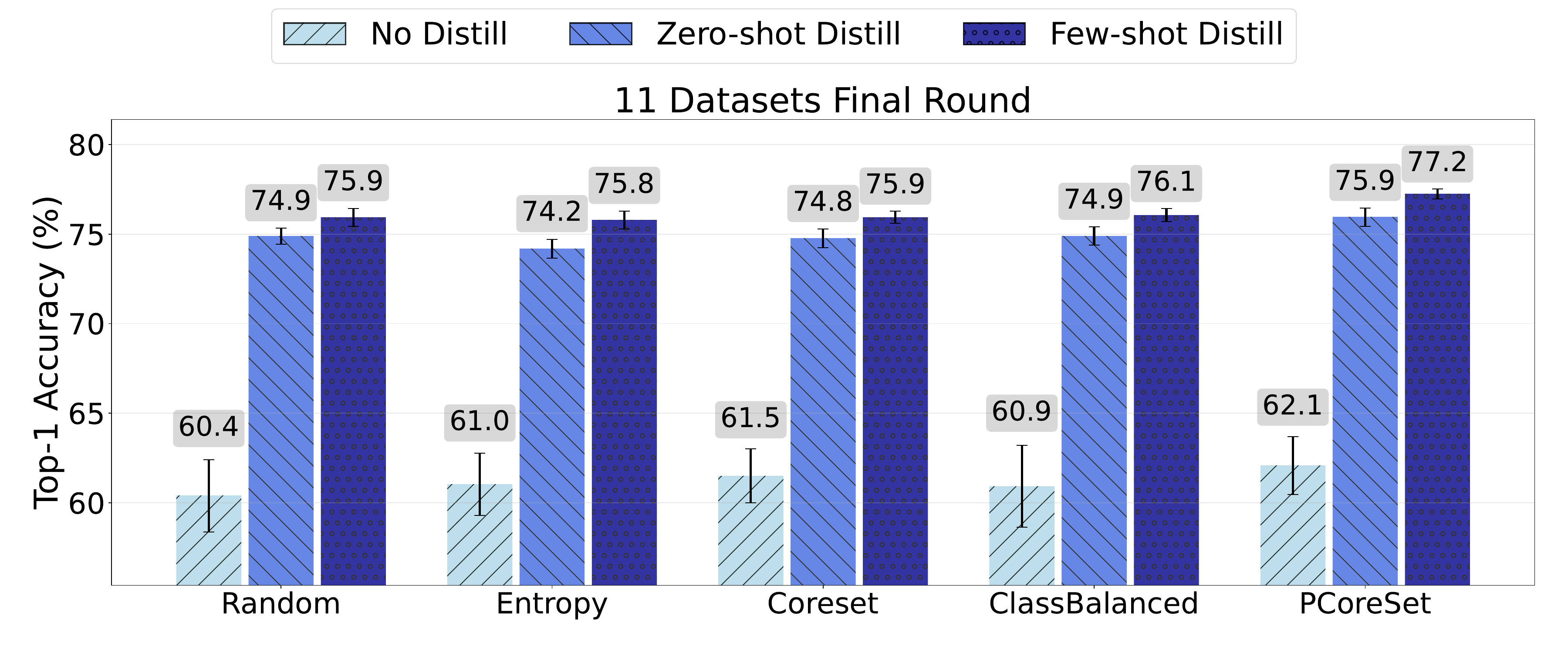}
    \end{subfigure}
    \hfill
    \begin{subfigure}[b]{0.42\textwidth}
        \centering
        \includegraphics[width=\textwidth]{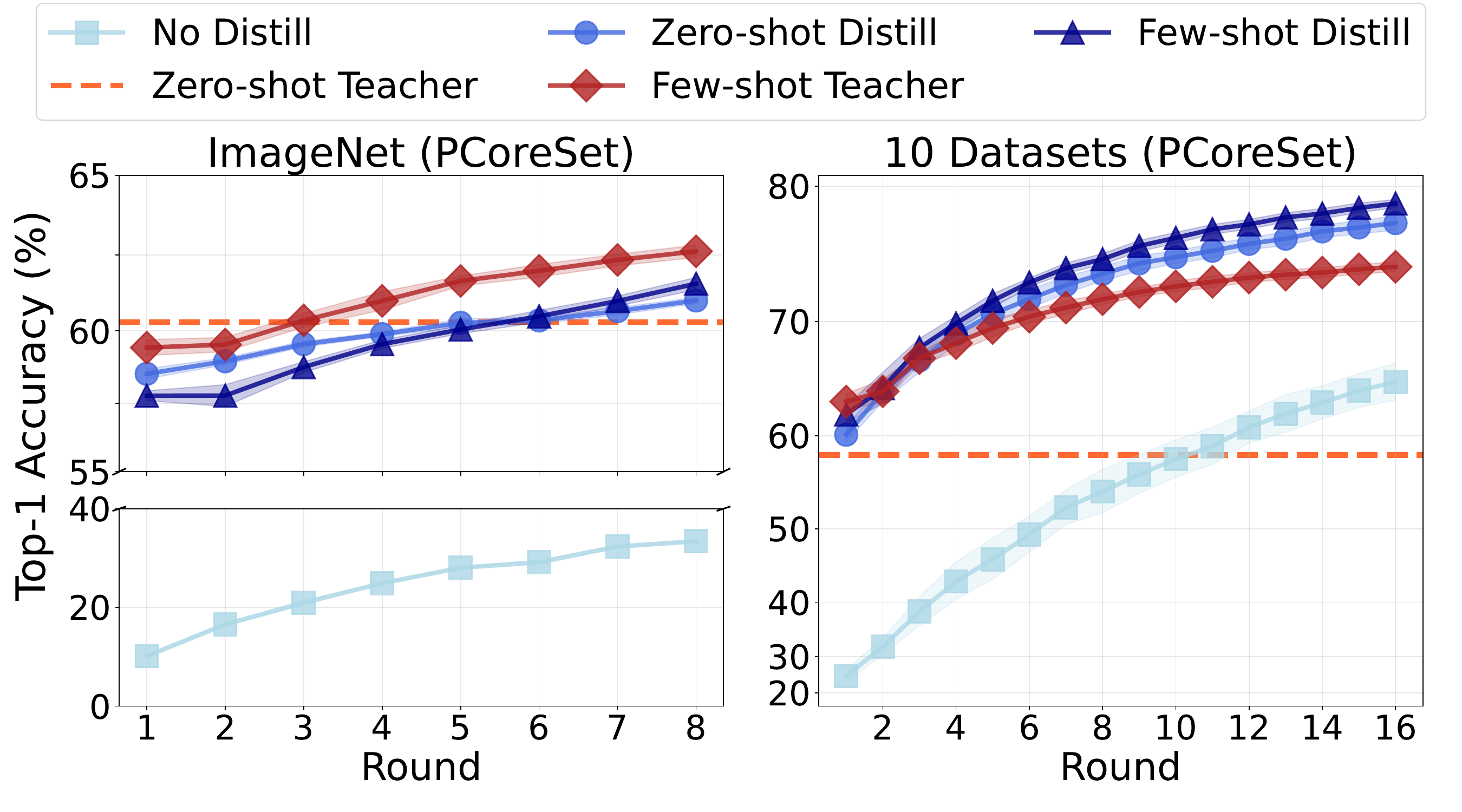}
    \end{subfigure}
    \vspace{-0.1in}
    \caption{
        \small
        \textbf{(Left):} The proposed \textbf{ActiveKD framework consistently improves the final-round accuracy of all selection methods}, while \textbf{\texttt{PCoreSet}} \textbf{further outperforms baselines} when combined with ActiveKD.
        \textbf{(Right):} ActiveKD consistently improves \textbf{\texttt{PCoreSet}} across active learning rounds (No Distill vs. Zero-shot Distill), with further gains when using few-shot teachers (Few-shot Distill). The proposed ActiveKD + \textbf{\texttt{PCoreSet}} \textbf{even surpasses} VLM teachers, outperforming \textbf{zero-shot teachers} on the average of 10 datasets after the \textbf{first AL round} and \textbf{few-shot teachers} after the \textbf{second}.
    }
    \vspace{-0.15in}
    \label{fig:combined}
\end{figure*}

\paragraph{Effectiveness of \textbf{ActiveKD} (Zero-Shot).}
We first evaluate \textbf{ActiveKD} with ResNet-18/ResNet-50 students and a ResNet-50 zero-shot teacher.
\Cref{tab:performance_comparison} shows that \textbf{ActiveKD} (\textbf{Zero-Shot}) \textbf{consistently improves No Distill across all datasets and selection strategies}: \eg, it improves accuracy by +27.33\% on ImageNet with \textit{Random} selection (33.36\%~$\to$~60.69\%) and by +29.07\% on average across strategies.
This improvement extends across 10 datasets with an average gain of +13.21\% (63.10\%~$\to$~76.31\%) for \textit{Random} selection, and +12.37\% on average across all strategies. These substantial gains demonstrates the effectiveness of our proposed ActiveKD under limited supervision.

\myparagraph{Effectiveness of \textbf{ActiveKD} (Few-Shot).}
\textbf{ActiveKD (Few-Shot)} also consistently improves over No Distill across all datasets and strategies; relative to \textbf{ActiveKD (Zero-Shot)} it yields +1.17\% for \textit{Random} and +1.41\% on average across the 10 datasets, demonstrating \textbf{the benefit of strong few-shot teachers}. However, the improvements of \textbf{ActiveKD (Few-Shot)} are not consistent on ImageNet; \eg, the gains are only +0.06\% with \textit{ClassBalanced} and \textbf{+0.41\% with \textbf{\texttt{PCoreSet}}}. We attribute this to the large number of classes in ImageNet ($C{=}1000$), which increases labeled samples per round, reducing the added value of teacher signals compared with other 10 datasets having $C\approx100$. For the same reason, Random \textit{often outperforms} Entropy/Coreset on ImageNet---also noted by \citet{emam2021active,bang2024active}---making Random a strong baseline.



\begin{figure*}[t!]
    \centering
    \includegraphics[width=0.95\textwidth]{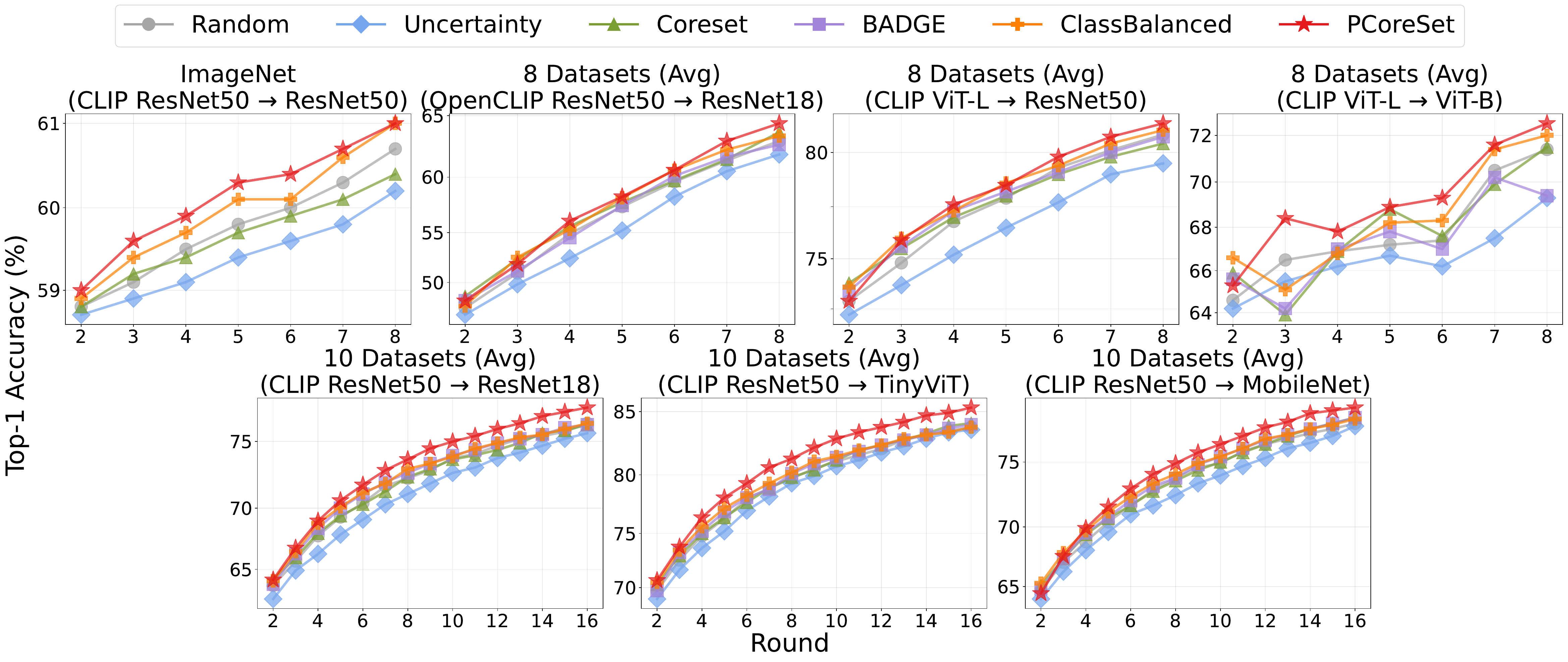}
    \vspace{-0.1in}
    \caption{
        \small\textbf{Results on ImageNet and the average over 8 or 10 datasets across 5 student and 3 teacher architectures} under ActiveKD (Zero-Shot) with either 8 or 16 rounds. Under ActiveKD, several widely-used selectors (Entropy, Coreset, BADGE) often \textit{underperform Random}; \textbf{\texttt{PCoreSet}} is the only selector that \textbf{consistently outperforms Random}, always achieving the best performance except for the first 2 AL rounds.}
    \label{fig:main_results}
\end{figure*}

\begin{figure*}[t!]
    \centering
    \vspace{-0.1in}
    \includegraphics[width=0.95\textwidth]{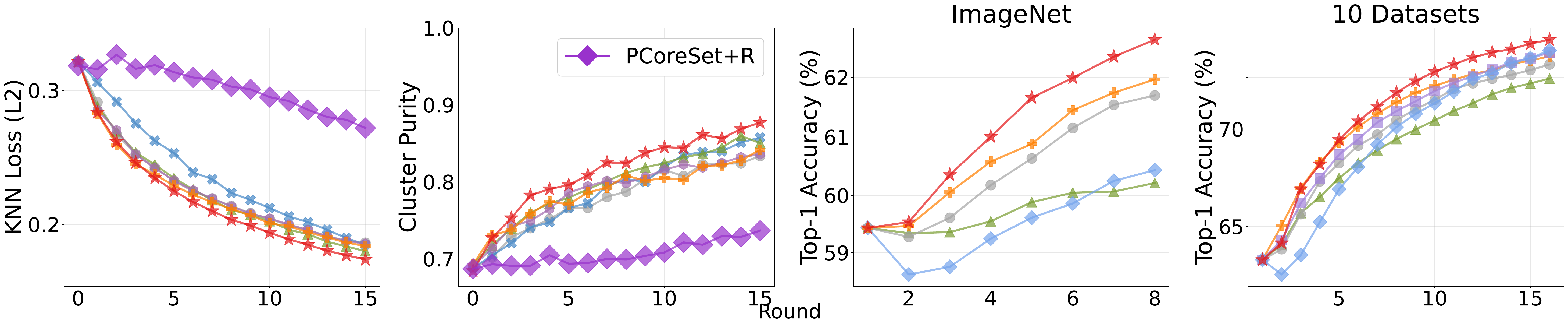}
    \vspace{-0.1in}
    \caption{\textbf{Left two panels:} KNN loss (L2) and cluster purity of \textbf{\texttt{PCoreSet}} and other selection strategies across 10 datasets; \textbf{right two panels:} performance of few-shot teachers on 11 datasets across rounds.}
    \label{fig:combined_pcoreset_teacher}
    \vspace{-0.15in}
\end{figure*}

\myparagraph{Effectiveness of \textbf{\texttt{PCoreSet}} under \textbf{ActiveKD (Zero-Shot)}.}
We now compare \textbf{\texttt{PCoreSet}} with selection baselines under \textbf{the same ActiveKD (Zero-Shot)}.
\Cref{fig:main_results} reports results on ImageNet and averages over 10 or 8 datasets across 8 or 16 AL rounds.
In total, we consider \textbf{73 settings}, computed as 
$(8-1)\times 4 + (16-1)\times 3$, excluding the first round where all selection methods share the same $\mathcal{D}^{(l)}$ with the same seed. 
The factors 4 and 3 denote the number of configurations for the 8-round and 16-round setups, including ImageNet and cross-dataset averages.
We observe slight degradation of \textbf{\texttt{PCoreSet}} in the first two rounds: \textbf{\texttt{PCoreSet}} relies on structured teacher bias, but at this stage there may be insufficient GT labels to reliably correct errors introduced by imperfect VLM teachers.
Beyond the first two rounds, however, \textbf{\texttt{PCoreSet}} \textbf{always achieves the best performance}, becoming increasingly effective as AL progresses. Importantly, no other selector consistently beats Random under ActiveKD---several (Entropy, Coreset, BADGE) often fall below it---so PCoreSet's consistent win is not marginal but \textbf{significant} by AL standards: all methods share an identical training pipeline and differ only in \textbf{data selection}.
See \Cref{sec:additional_experiments} for in-depth comparison across datasets and additional analysis.

\subsection{Analysis}



\paragraph{The acceleration of structured bias propagation.}
To assess whether \textbf{\texttt{PCoreSet}} promotes \emph{structured bias propagation}, we track the same metric, \ie, KNN loss ($\ell_2$) and clster purity, used in \Cref{fig:propagation} across AL rounds. 
We also include a counterfactual baseline, \textbf{\texttt{PCoreSet}}+R (Reverse), which suppresses probability-space coverage  by replacing $\arg\max$ with $\arg\min$ in line~4 of \Cref{alg:pcoreset}. As shown in the \textbf{two leftmost panels} of \Cref{fig:combined_pcoreset_teacher}, \textbf{\texttt{PCoreSet}} achieves consistently \textbf{the lowest KNN loss and the highest cluster purity}, confirming that it promotes \emph{structured bias propagation} as intended.
In contrast, \textbf{\texttt{PCoreSet}}+R shows the highest KLL loss and the lowest cluster purity, which corroborates that maximizing sample diversity in the probability space encourages bias propagation.

\myparagraph{A virtuous cycle with few-shot teachers.}
Beyond accelerating structured bias propagation, we evaluate the few-shot teacher under ActiveKD. Because acquisition is driven by the student predictions, we can test whether the selected samples by each selection method also benefit the teacher. As shown in the \textbf{two rightmost panels} of \Cref{fig:combined_pcoreset_teacher}, \textbf{\texttt{PCoreSet}} consistently outperforms competing selection strategies for the teacher. This supports that the \emph{structured prediction bias} of teacher models propagates to the student: by targeting regions underrepresented in the probability space of student models, \textbf{\texttt{PCoreSet}} also covers underrepresented teacher modes, \textbf{improving both models and inducing a virtuous cycle of better acquisition and stronger distillation}.

\begin{figure*}[t!]
    \centering
    \includegraphics[width=0.9\textwidth]{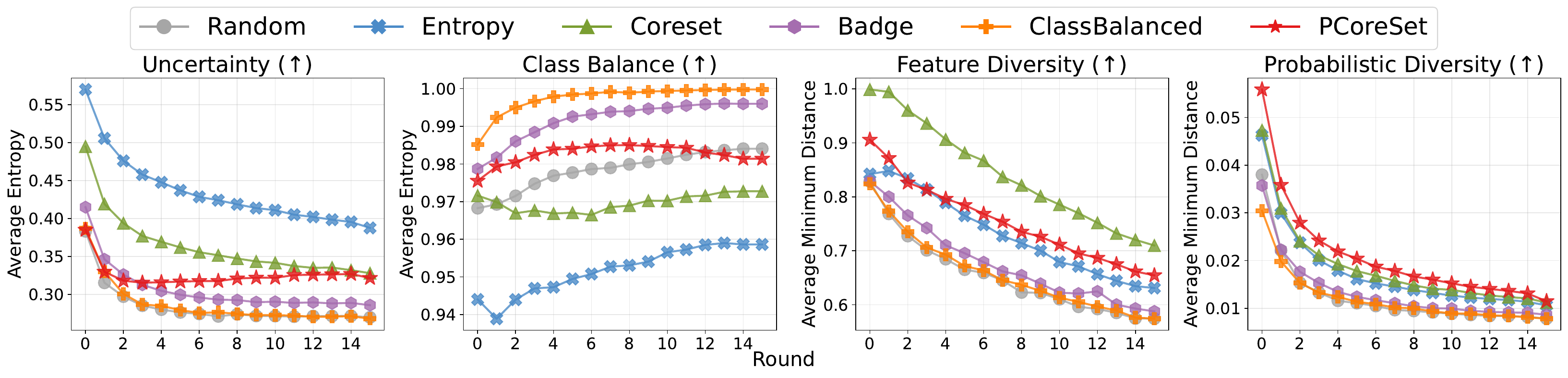}
    \vspace{-0.1in}
    \caption{
        \small
        Comparison of four selection criteria: 1) Uncertainty, 2) Class balance, 3) Feature space diversity, and 4) Probability space diversity. The results are averaged over the 10 datasets excluding ImageNet.
    }
    \label{fig:metric}
    \vspace{-0.1in}
\end{figure*}

\begin{figure*}[t!]
    \centering
    \includegraphics[width=0.9\textwidth]{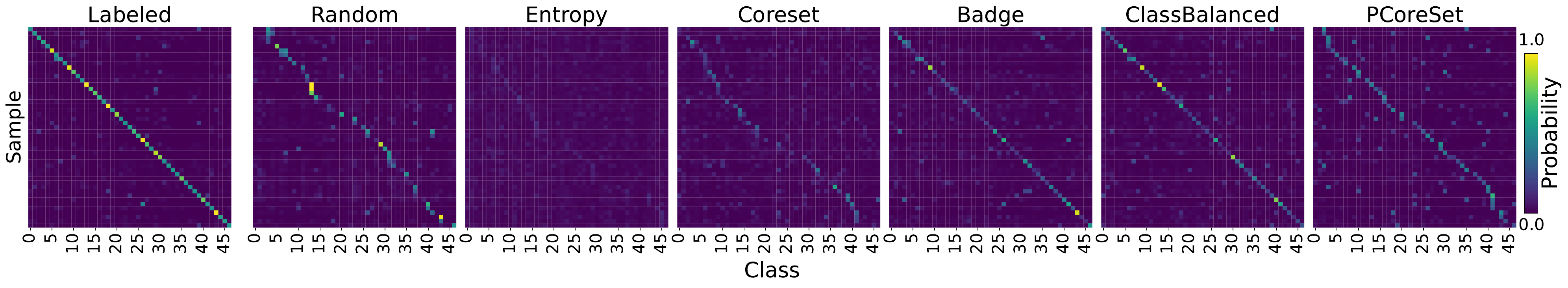}
    \vspace{-0.1in}
    \caption{
        \small
        The heatmap of probability vectors in the first AL round using the DTD dataset.
        See \Cref{sec:heatmap_of_selected_samples} for results on other datasets.
    }
    \label{fig:heatmap}
    \vspace{-0.15in}
\end{figure*}

\vspace{-0.01in}
\myparagraph{Active selection criteria.}
We analyze the progression of four strategies designed to be maximized throughout AL rounds: 
1) uncertainty, 
2) class balance, 
3) feature space diversity, and 
4) probability space diversity (\textbf{\texttt{PCoreSet}}). 
Criteria 1) and 2) are measured using normalized Shannon entropy, while 3) and 4) are evaluated using normalized average minimum distances. 
We report the average of each metric across 10 datasets, excluding ImageNet.
As shown in \Cref{fig:metric}, each method performs best on its respective objective: \textit{Entropy} maximizes uncertainty, \textit{ClassBalanced} optimizes the entropy of the class distribution, and \textit{Coreset} maximizes feature space diversity. 
As expected, our \textbf{\texttt{PCoreSet}} method \textbf{achieves the highest diversity in probability space}, while also ranking third in uncertainty and class balance, and second in feature diversity. 
Although \textit{Coreset} is effective in covering the feature space and offers moderate probabilistic diversity, it performs poorly in class balance, which may be the reason that it underperforms than \textit{ClassBalanced} baseline.


\vspace{-0.01in}
\myparagraph{Sample diversity of \textbf{\texttt{PCoreSet}}.}
To visualize the sample diversity of \textbf{\texttt{PCoreSet}}, we visualize probability heatmaps for the initial labeled set and for the samples selected after round~1 by each active-learning strategy (\Cref{fig:heatmap}).
The heatmaps show that \textbf{\texttt{PCoreSet}} selects samples with highly diverse probability profiles, maximizing distributional coverage relative to the initial set.
We believe this probabilistic and visual diversity helps us \textbf{select samples underrepresented in the probability space}, thereby more efficiently propagating the teacher's \emph{structured prediction bias} during distillation (see \Cref{sec:qualitative_results} for qualitative examples).

\vspace{-0.1in}
\section{Conclusion}
\label{sec:conclusion}

We introduced ActiveKD, integrating AL with KD by leveraging zero-/few-shot capabilities of VLMs.
We discovered VLMs exhibit \textit{structured prediction bias} that benefits student learning, and proposed \textbf{\texttt{PCoreSet}} to maximize diversity in probability space when selecting samples.
Extensive experiments across 11 datasets show the effectiveness of ActiveKD and \textbf{\texttt{PCoreSet}}.
Beyond improving the student, our analysis reveals a virtuous cycle: under the few-shot teacher mode, the samples \textbf{\texttt{PCoreSet}} selects to fill gaps in the student's probability space also tend to cover under-represented teacher modes, so the teacher itself benefits from the same selections (\Cref{fig:combined_pcoreset_teacher}).
We also find that this advantage is tied to the distillation signal: in pure AL without a VLM teacher, probability-space diversity confers no clear gain over Random (\Cref{sec:nodistill}), positioning \textbf{\texttt{PCoreSet}} specifically as a selector for the KD-with-AL regime rather than as a generic AL heuristic.
Limitations and future work are discussed in \Cref{sec:limitations}.




\bibliographystyle{unsrtnat}
\bibliography{reference}


\newpage
\appendix

\section*{Appendix Overview}

This appendix contains supplementary information organized as follows:

\begin{itemize}[itemsep=1mm,parsep=1pt,topsep=2pt,leftmargin=*]
    \item \textbf{Limitations and Future Work} (\Cref{sec:limitations}): discusses limitations of our approach and promising directions for future research.

    \item \textbf{Theoretical Analysis} (\Cref{sec:theoretical_analysis}): provides formal analysis of bias propagation in knowledge distillation, including optimal student prediction analysis (\Cref{sec:theoretical_analysis_dho}), formal characterization of teacher prediction bias (\Cref{sec:theoretical_analysis_bias}), and mathematical proof of how biases transfer from teacher to student models (\Cref{prop:bias-propagation_2}).

    \item \textbf{Implementation} (\Cref{sec:implementation}): detailed experimental setup including model configurations, training parameters.

    \item \textbf{Datasets} (\Cref{sec:datasets}): comprehensive description of the datasets used in our experiments, including ImageNet and other benchmarks used for evaluating bias propagation, as well as detailed analysis of class distributions across all datasets.

    \item \textbf{Additional Experiments} (\Cref{sec:additional_experiments}): supplementary experimental results including evaluations on alternative architectures (e.g., MobileNetV2 in \Cref{sec:architectures}), active learning without knowledge distillation (\Cref{sec:nodistill}), few-shot teacher distillation scenarios (\Cref{sec:fewshot}), and logit adjustment for calibrated teacher prediction (\Cref{sec:logit_adjustment}).

    \item \textbf{Additional Qualitative Results} (\Cref{sec:additional_results}): additional results of the proposed method, including heatmap visualizations of selected samples across datasets (\Cref{sec:heatmap_of_selected_samples}) and qualitative analysis of samples selected by our \textbf{\texttt{PCoreSet}} method (\Cref{sec:qualitative_results}).
\end{itemize}


\section{Limitations and Future Work}
\label{sec:limitations}

While our experiments focused on VLMs, we believe AL has even greater potential when leveraging the growing capabilities of generalist foundation models~\citep{wei2021finetuned,liu2023visual,bai2023qwen,achiam2023gpt,team2023gemini}.
Many task-specific applications face challenges due to limited labels, and using generalist models could provide valuable opportunities for future exploration.
Our implementation is limited to visual recognition tasks, as they are representative and provide an ideal testbed.
In addition, the zero-/few shot capabilities of VLMs in visual recognition make them natural candidates for KD.
However, extending our approach to object detection and segmentation is a promising direction that may require architectural adaptations.
Furthermore, while various KD and semi-supervised learning methods \citep[\eg,][]{wang2022debiased} are orthogonal to ActiveKD, assessing their effectiveness within this framework remains future work.





\section{Theoretical Analysis}
\label{sec:theoretical_analysis}

This section presents a formal analysis of bias propagation in knowledge distillation, examining how teacher model biases influence student model predictions.
We first analyze the optimal student prediction after distillation from a teacher model in Section~\ref{sec:theoretical_analysis_dho}, building upon the analytical framework established in DHO~\citep{DHO}.
Subsequently, in Section~\ref{sec:theoretical_analysis_bias}, we formally characterize the mechanisms through which teacher prediction biases propagate to and constrain the distilled student model, providing the theoretical foundation for our debiasing strategies.


\subsection{Optimal Student Prediction after Distillation}
\label{sec:theoretical_analysis_dho}
Let us first establish our notation. We consider a student model with a feature extractor $h: \mathcal{X} \rightarrow \mathbb{R}^d$ that maps inputs to feature representations, and a classification head $g: \mathbb{R}^d \rightarrow \mathbb{R}^C$ that maps features to logits, with the final prediction obtained by applying the softmax function $\sigma: \mathbb{R}^C \rightarrow \Delta^{C-1}$ to these logits, resulting in $f_r(x) = \sigma(g(h(x)))$, where $r$ denotes the active learning round.
Following DHO~\citep{DHO}, we consider two target probability distributions: the ground truth label distribution $y$, typically one-hot encoded vectors where $y_c = 1$ for the true class $c$ and 0 elsewhere, and the teacher's softened distribution $f(x)$ for input $x \in \mathcal{X}$.

\begin{theorem}[Optimal Distribution for Knowledge Distillation]
\label{thm:optimal_distillation}
The distribution $f^*(x)$ that minimizes the weighted combination of cross-entropy loss with respect to $y$ and Kullback-Leibler divergence with respect to $f(x)$:
\begin{equation}
\mathcal{L}(f^*(x)) = \lambda \ell(f^*(x),y) + (1-\lambda) \KL(f(x)\|f^*(x))
\end{equation}
is the weighted arithmetic mean:
\begin{equation}
f^*(x) = \lambda y + (1-\lambda) f(x)
\end{equation}
where $\lambda \in [0,1]$ is the weighting hyperparameter.
\end{theorem}

This optimal prediction $f^*(x)$ represents the theoretical target that a single-head distillation approach should converge to.

\begin{assumption}[$\varepsilon$-Convergence]
\label{assump:convergence}
We assume that after sufficient training, a student model has converged to the optimal prediction with bounded error:
\begin{equation}
\sup_x \|f_r(x) - f^*(x)\|_1 \leq \varepsilon
\end{equation}
where $f_r(x) = \sigma(g(h(x)))$ is the student model's prediction, $f^*(x) = \lambda y + (1-\lambda) f(x)$ is the optimal prediction, $\|\cdot\|_1$ denotes the L1 norm, and $\varepsilon > 0$ is a small constant representing generalization error.
\end{assumption}

This assumption is reasonable for both single-head and dual-head approaches:

For the \textbf{Single-Head Approach}, a single output head $g(h(x))$ is trained directly on the combined loss $\mathcal{L}(f^*(x))$, and with sufficient capacity and training, it approximates $f^*(x)$ with error bounded by $\varepsilon$.

For the \textbf{Dual-Head Approach}, in the Dual-Head Optimization (DHO) framework~\citep{DHO}, shared features $h(x)$ are extracted from input $x$ and two specialized classification heads are employed: $g_{\text{CE}}(h(x))$, optimized for ground truth labels using cross-entropy loss, and $g_{\text{KD}}(h(x))$, optimized for teacher predictions using KL divergence. This approach addresses the constrained optimization problem:
\begin{equation}
\min_{h,g_{\text{CE}},g_{\text{KD}}} \lambda \mathbb{E}_{x,y}[\ell(\sigma(g_{\text{CE}}(h(x))), y)] + (1-\lambda) \mathbb{E}_{x}[\KL(f(x) \| \sigma(g_{\text{KD}}(h(x))))]
\end{equation}

At inference time, their outputs are combined to form the final prediction:
\begin{equation}
f_r(x)_{\text{DHO}} = \alpha \cdot \sigma(g_{\text{CE}}(h(x))) + (1-\alpha) \cdot \sigma(g_{\text{KD}}(h(x))/\beta)
\end{equation}

The DHO framework~\citep{DHO} provides a formal proof that under appropriate optimization conditions, the dual-head inference is equivalent to single-head inference when setting $\alpha=\lambda$ and $\beta=\tau$, meaning that for any choice of distillation weight $\lambda$ and temperature $\tau$, the two approaches converge to the same prediction.

\begin{theorem}[Optimal Prediction as Linear Combination]
\label{thm:approach_equiv}
Under Assumption~\ref{assump:convergence}, the optimal prediction $f^*(x) = \lambda y + (1-\lambda) f(x)$ is a linear combination of ground truth and teacher prediction, with bounded error $\varepsilon$, regardless of whether a single-head or dual-head approach is employed.
This formulation provides the basis for analyzing bias propagation in knowledge distillation.
\end{theorem}

This theoretical foundation establishes that regardless of the architectural choice (single-head or dual-head), the optimal student prediction converges to a linear combination of the ground truth and teacher prediction, providing a principled basis for analyzing bias propagation in knowledge distillation.


\subsection{Formal Definition of Teacher Prediction Bias}
\label{sec:theoretical_analysis_bias}

\begin{definition}[Structured prediction bias]\label{def:bias_2}
Teacher predictions exhibit \textit{structured prediction bias} if there exist $K \in \mathbb{N}$, centroids $\{\mu_k\}_{k=1}^{K} \subset \Delta^{C-1}$, and radii $\{r_k\}_{k=1}^{K} \subset \mathbb{R}_{>0}$, such that:
\begin{equation}
\forall x \in \mathcal{X}, \quad f(x) \in \bigcup_{k=1}^{K} \left\{ p \in \Delta^{C-1} : \|p - \mu_k\|_2 \leq r_k \right\}.
\end{equation}
\end{definition}

This formulation captures the empirical phenomenon that foundation teacher models tend to produce predictions clustered within a limited number of regions in the probability simplex $\Delta^{C-1}$ (where $C$ is the number of classes), rather than utilizing the entire space of possible probability distributions.
The number of biased regions $K$ is typically much smaller than the model's theoretical expressivity would allow, reflecting inherent biases in the teacher's pretraining data and architecture.

\begin{proposition}[Bias propagation through KD]\label{prop:bias-propagation_2}
Let the teacher model $f$ exhibit structured prediction bias as defined in \Cref{def:bias_2}, with $\{\mu_k\}_{k=1}^K$ and $\{r_k\}_{k=1}^K$.
Assume the student model $f_r$ is trained via KD from $f$ using the loss
$\lambda\mathcal{L}_{\text{CE}} + (1 - \lambda)\mathcal{L}_{\text{KD}}$, and satisfies
$\sup_x \|f_r(x) - f^*(x)\|_1 \leq \varepsilon$ for all $x \in \mathcal{X}$, where $f^*(x) = \lambda y + (1-\lambda) f(x)$ and $y \in \{0,1\}^C$ denotes the one-hot label of $x$.
Then student predictions $f_r(x)$ also exhibit structured prediction bias as in \Cref{def:bias_2}.
Specifically, for every $x \in \mathcal{X}$, there exists $k \in [K]$ such that:
\begin{equation}
f_r(x) \in \left\{ p \in \Delta^{C-1} : \|p - \hat{\mu}_k(x)\|_2 \leq \hat{r}_k \right\},
\end{equation}
where the propagated centroid is defined as $\hat{\mu}_k(x) = \lambda y + (1 - \lambda)\mu_k$,
and the adjusted radius is $\hat{r}_k = (1 - \lambda)r_k + \varepsilon$.
Since $y$ is one-hot and $\mu_k$ is fixed, the set of possible centers $\{\hat{\mu}_k(x)\}$ is finite and contained in $\Delta^{C-1}$,
thus satisfying the condition of \Cref{def:bias_2} with at most $C \cdot K$ clusters.
\end{proposition}

\begin{proof}
From Assumption~\ref{assump:convergence}, we have:
\begin{equation}
\|f_r(x) - f^*(x)\|_1 \leq \varepsilon
\end{equation}

For any L1 norm bounded by $\varepsilon$, the maximum deviation in any single component is also bounded by $\varepsilon$, yielding:
\begin{equation}
f^*(x) - \varepsilon\mathbf{1} \leq f_r(x) \leq f^*(x) + \varepsilon\mathbf{1}
\end{equation}

From Definition~\ref{def:bias_2}, for input $x$, the teacher's prediction $f(x)$ satisfies:
\begin{equation}
\mu_i - r_i\mathbf{1} \leq f(x) \leq \mu_i + r_i\mathbf{1}
\end{equation}

for some bias region center $\mu_i$. Substituting this into the formula for $f^*(x) = \lambda y + (1-\lambda) f(x)$:
\begin{equation}
\lambda y + (1-\lambda)(\mu_i - r_i\mathbf{1}) \leq f^*(x) \leq \lambda y + (1-\lambda)(\mu_i + r_i\mathbf{1})
\end{equation}

Combining these inequalities:
\begin{equation}
\lambda y + (1-\lambda)(\mu_i - r_i\mathbf{1}) - \varepsilon\mathbf{1} \leq f_r(x) \leq \lambda y + (1-\lambda)(\mu_i + r_i\mathbf{1}) + \varepsilon\mathbf{1}
\end{equation}

To show that student predictions remain constrained, we note that for each teacher bias region centered at $\mu_k$, the student's predictions are bounded by:
\begin{equation}
\hat{\mu}_k(x) - (1-\lambda)r_k\mathbf{1} - \varepsilon\mathbf{1} \leq f_r(x) \leq \hat{\mu}_k(x) + (1-\lambda)r_k\mathbf{1} + \varepsilon\mathbf{1}
\end{equation}

where $\hat{\mu}_k(x) = \lambda y + (1-\lambda)\mu_k$. This implies that student predictions are constrained within regions centered at $\hat{\mu}_k(x)$ with radius $\hat{r}_k = (1-\lambda)r_k + \varepsilon$.

While the teacher's predictions are constrained to $K$ distinct regions, the student's predictions may not maintain exactly $K$ distinct regions due to potential merging or splitting effects when combining with ground truth labels. However, the student's predictions remain bounded and cannot freely explore the entire probability simplex.

Setting $\hat{r}_k = (1-\lambda)r_k + \varepsilon$, we can express the student's prediction constraint as:
\begin{equation}
f_r(x) \in \bigcup_{k'=1}^{K'} \{p \in \Delta^{C-1} : \|p - \hat{\mu}_k(x)\|_2 \leq \hat{r}_k\}
\end{equation}
\end{proof}

This demonstrates that regardless of the exact number of distinct regions, the student model inherits constrained prediction patterns from the teacher model, with centers shifted by the ground truth component and radii scaled by $(1-\lambda)$ plus the convergence error $\varepsilon$.

\begin{corollary}[Student Bias Inheritance]
\label{cor:student_bias}
Let $K'$ denote the number of distinct prediction regions in the student model after distillation.
The student model's predictions are constrained to these $K'$ distinct regions in the probability simplex, where:
\begin{equation}
K' \leq K \cdot C
\end{equation}
where $C$ is the number of classes.
\end{corollary}

\begin{proof}
Consider the student's optimal prediction $f^*(x) = \lambda y + (1-\lambda) f(x)$ for any input $x$. From Definition~\ref{def:bias_2}, the teacher's prediction $f(x)$ belongs to one of $K$ distinct regions in the probability simplex, each centered at some $\mu_k$ with radius $r_k$.

For each teacher bias region centered at $\mu_k$, the student's prediction center becomes $\hat{\mu}_{k,c} = \lambda e_c + (1-\lambda)\mu_k$, where $e_c$ is the one-hot vector for class $c$. This is a function of both the teacher's bias center $\mu_k$ and the ground truth label.

Since there are $K$ distinct teacher bias regions and $C$ distinct ground truth label distributions (one per class), there can be at most $K \cdot C$ distinct combinations of $(\mu_k, e_c)$. Each such combination produces a potential distinct student bias region centered at $\hat{\mu}_{k,c} = \lambda e_c + (1-\lambda)\mu_k$.

Therefore, the maximum number of distinct student prediction regions is bounded by:
\begin{equation}
K' \leq K \cdot C
\end{equation}

These student regions are convex combinations of ground truth labels and the teacher's biased regions, with an additional error margin of $\varepsilon$. Specifically, for each teacher bias region with center $\mu_k$ and radius $r_k$, and each class $c$, there corresponds a student bias region with center $\hat{\mu}_{k,c} = \lambda e_c + (1-\lambda)\mu_k$ and radius $\hat{r}_k = (1-\lambda)r_k + \varepsilon$.
\end{proof}

This establishes a comprehensive mathematical relationship between teacher bias and its propagation to the student model during knowledge distillation.
The student model inherits a structured and constrained prediction space from the teacher, which constitutes a direct and quantifiable form of bias propagation through the distillation process.



\newpage

\section{Implementation Details}
\label{sec:implementation}

\begin{algorithm}[H]
    \caption{DHO Training with zero-shot CLIP~\citep{radford2021learning} teacher}
    \label{alg:training}
    \begin{algorithmic}[1]
        \State {\bfseries Input:} labeled set $\mathcal{D}^{(l)} = \{(x^{(l)}_i, y_i)\}_{i=1}^N$, unlabeled set $\mathcal{D}^{(u)} = \{x^{(u)}_j\}_{j=1}^M$,
        \State \phantom{\bfseries Input:} student feature extractor $g$, prediction heads $h_{\text{CE}}, h_{\text{KD}}$, teacher encoders $f_{\mathcal{X}}, f_{\mathcal{T}}$,
        \State \phantom{\bfseries Input:} prompt template ``A photo of \texttt{[CLASS]}'', temperature scaling factors $\zeta, \eta$, \State\phantom{\bfseries Input:} balancing hyperparameter $\lambda$,
        \State \phantom{\bfseries Input:} supervised mini-batch size $B$, and unsupervised mini-batch size $B'$.
        \While{not converged}
            \State Sample mini-batch $\mathcal{B}^{(l)} = \{(x_b^{(l)}, y_b)\}_{b=1}^{B}$ from $\mathcal{D}^{(l)}$, $\mathcal{B}^{(u)} = \{x_{b'}^{(u)}\}_{b'=1}^{B'}$ from $\mathcal{D}^{(l)}\cup\mathcal{D}^{(u)}$.

            \State \textcolor{gray}{// Process labeled data}
            \For{each $(x_b^{(l)}, y_b) \in \mathcal{B}^{(l)}$}
                \State $z_b^{(l)} \leftarrow g(x_b^{(l)})$
                \State $\hat{p}_{\text{CE},b}^{(l)} \leftarrow \sigma(h_{\text{CE}}(z_b^{(l)}))$
                \State $\hat{p}_{\text{KD},b}^{(l)} \leftarrow \sigma(\frac{1}{\eta}h_{\text{KD}}(z_b^{(l)}))$
                \State $p_b^{(l)} \leftarrow \sigma\left(\frac{1}{\zeta\cdot\eta}[\mathtt{CosSim}(f_\mathcal{X}(x_b^{(l)}), f_\mathcal{T}(t_1)), \ldots,\mathtt{CosSim}(f_\mathcal{X}(x_b^{(l)}), f_\mathcal{T}(t_C))]^\top\right)$
            \EndFor

            \State \textcolor{gray}{// Process unlabeled data}
            \For{each $x_{b'}^{(u)} \in \mathcal{B}^{(u)}$}
                \State $z_{b'}^{(u)} \leftarrow g(x_{b'}^{(u)})$
                \State $\hat{p}_{\text{KD},b'}^{(u)} \leftarrow \sigma(\frac{1}{\eta}h_{\text{KD}}(z_{b'}^{(u)}))$
                \State $p_{b'}^{(u)} \leftarrow \sigma\left(\frac{1}{\zeta\cdot\eta}[\mathtt{CosSim}(f_\mathcal{X}(x_{b'}^{(u)}), f_\mathcal{T}(t_1)), \ldots,\mathtt{CosSim}(f_\mathcal{X}(x_{b'}^{(u)}), f_\mathcal{T}(t_C))]^\top\right)$
            \EndFor

            \State \textcolor{gray}{// Compute losses and update}
            \State $\mathcal{L}_{\text{CE}} \leftarrow \frac{1}{B}\sum_{b=1}^B\ell(\hat{p}_{\text{CE},b}^{(l)},y_b)$
            \State $\mathcal{L}_{\text{KD}} \leftarrow \frac{1}{B}\sum_{b=1}^B \KL(\hat{p}_{\text{KD},b}^{(l)}||p_{b}^{(l)}) + \frac{1}{B'}\sum_{b'=1}^{B'} \KL(\hat{p}_{\text{KD},b'}^{(u)}||p_{b'}^{(u)})$
            \State $\mathcal{L} \leftarrow \lambda \mathcal{L}_{\text{CE}} + (1 - \lambda) \mathcal{L}_{\text{KD}}$
            \State Update parameters of $g$, $h_{\text{CE}}$, $h_{\text{KD}}$ using $\nabla \mathcal{L}$
        \EndWhile
    \end{algorithmic}
\end{algorithm}

\begin{algorithm}[H]
    \caption{DHO Inference}
    \label{alg:inference}
    \begin{algorithmic}[1]
        \State {\bfseries Input:} an image $x$, feature extractor $g$, prediction heads $h_{\text{CE}}, h_{\text{KD}}$, linear coefficient $\alpha$, temperature scaling $\beta$
        \State $z \leftarrow g(x)$
        \State $\hat{p}_{\text{CE}} \leftarrow \sigma(h_{\text{CE}}(z))$
        \State $\hat{p}_{\text{KD}} \leftarrow \sigma(h_{\text{KD}}(z)/\beta)$
        \State $\hat{p} \leftarrow \alpha \cdot \hat{p}_{\text{CE}} + (1-\alpha) \cdot \hat{p}_{\text{KD}}$
        \State $\hat{y} \leftarrow \arg\max_{c}(\hat{p}_c)$
        \State {\bfseries Return:} $\hat{y}$
    \end{algorithmic}
\end{algorithm}

We employ the DHO~\citep{DHO} framework for knowledge distillation in our approach, as it provides a principled and effective mechanism for transferring knowledge from foundation teacher models to student models.
DHO's key advantage lies in its dual-head architecture that maintains separated optimization paths for labeled data and teacher knowledge, thereby preserving the distinct learning signals throughout training.
\Cref{alg:training} details the DHO training procedure with our zero-shot CLIP teacher, while \Cref{alg:inference} outlines the inference process that effectively combines both optimization pathways.

\begin{table}[htb]
    \caption{
        Implementation details for our experiments across different settings.
    }
    \label{tab:implementation_details}
    \centering
    \small
    \resizebox{\textwidth}{!}{%
    \begin{tabular}{p{0.48\textwidth}|p{0.48\textwidth}}
        \toprule
        \rowcolor{gray!20} \multicolumn{2}{c}{\textit{Active Learning on ImageNet}} \\
        \midrule
        \textbf{Model Configuration} & \textbf{Training Details} \\
        \midrule
        \begin{itemize}[leftmargin=*,nosep]
            \item \textbf{Student:} DINO self-supervised ResNet-50 \citep{caron2021emerging}
            \item \textbf{Teacher:} CLIP ResNet-50 \citep{radford2021learning}
            \item \textbf{Active learning rounds:} 8
            \item \textbf{Initial setting:} 1-shot (single image per class)
            \item \textbf{Query size:} 1,000
            \item \textbf{Unlabeled pool:} 100,000 samples randomly selected with seed
            \item \textbf{Random seeds:} 5 different seeds
            \item \textbf{KD parameters:} $\zeta=0.01$, $\eta=2$, $\lambda=0.5$
            \item \textbf{DHO parameters:} $\alpha=0.4$, $\beta=0.5$
            \item \textbf{Validation:} No validation split used
        \end{itemize}
        &
        \begin{itemize}[leftmargin=*,nosep]
            \item \textbf{Epochs:} 20 for first 7 rounds, 50 for final round
            \item \textbf{Optimizer:} AdamW ($\beta_1=0.9$, $\beta_2=0.999$)
            \item \textbf{Learning rate:} $1\times10^{-3}$
            \item \textbf{Weight decay:} $1\times10^{-2}$
            \item \textbf{Batch size:} 512 (labeled: 256, unlabeled: 256)
        \end{itemize} \\
        \midrule
        \rowcolor{gray!20} \multicolumn{2}{c}{\textit{Active Learning on 10 Additional Datasets}} \\
        \midrule
        \textbf{Model Configuration} & \textbf{Training Details} \\
        \midrule
        \begin{itemize}[leftmargin=*,nosep]
            \item \textbf{Student:} ImageNet pre-trained ResNet-18 \citep{he2016deep}, MobileNetV2 \citep{sandler2018mobilenetv2}, and ImageNet-21k \citep{deng2009imagenet} pre-trained ViT-Tiny \citep{wu2022tinyvit}
            \item \textbf{Teacher:} CLIP ResNet-50 \citep{radford2021learning}
            \item \textbf{Active learning rounds:} 16
            \item \textbf{Initial setting:} 1-shot (single image per class)
            \item \textbf{Query size:} Equal to number of classes per round
            \item \textbf{Unlabeled pool:} All training samples except labeled set
            \item \textbf{Random seeds:} 5 different seeds
            \item \textbf{KD parameters:} $\zeta=0.01$, $\eta=2$, $\lambda=0.5$
            \item \textbf{DHO parameters:} $\alpha=0.5$, $\beta=1$
            \item \textbf{Validation:} No validation split used
        \end{itemize}
        &
        \begin{itemize}[leftmargin=*,nosep]
            \item \textbf{Epochs:} 200 for all rounds
            \item \textbf{Optimizer:} AdamW ($\beta_1=0.9$, $\beta_2=0.999$)
            \item \textbf{Learning rate:} $1\times10^{-3}$
            \item \textbf{Weight decay:} $1\times10^{-2}$
            \item \textbf{Batch size:} 128 (labeled: 64, unlabeled: 64)
        \end{itemize} \\
        \midrule
        \rowcolor{gray!20} \multicolumn{2}{c}{\textit{Few-Shot Teacher Active Distillation}} \\
        \midrule
        \textbf{Model Configuration} & \textbf{Training Details} \\
        \midrule
        \begin{itemize}[leftmargin=*,nosep]
            \item \textbf{Teacher:} CLAP \citep{silva2024closer} on CLIP ResNet-50 \citep{radford2021learning}
        \end{itemize}
        &
        \begin{itemize}[leftmargin=*,nosep]
            \item \textbf{Setup:} Following original CLAP paper \citep{silva2024closer}
            \item \textbf{Modification:} Learning rate reduced from 0.1 to 0.01 for better convergence
        \end{itemize} \\
        \bottomrule
    \end{tabular}
    }
\end{table}

Table~\ref{tab:implementation_details} presents the comprehensive implementation details for our experiments across three distinct settings.
For ImageNet experiments, we employed a DINO self-supervised ResNet-50 student model with CLIP ResNet-50 as the teacher, conducting 8 active learning rounds starting from a 1-shot setting with 1,000 queries per round, using AdamW optimization with carefully tuned hyperparameters.
Our experiments on 10 additional datasets utilized various student architectures (ResNet-18, MobileNetV2, and ViT-Tiny) with CLIP ResNet-50 as the teacher, extending to 16 active learning rounds with class-count-based query sizes and longer training epochs.
For the Few-Shot Teacher Active Distillation setting, we implemented CLAP on CLIP ResNet-50 following the original paper's methodology with a reduced learning rate (0.01) to ensure better convergence in our specific experimental context.
All experiments were conducted across 5 different random seeds without utilizing validation splits to simulate realistic low-data scenarios.

\paragraph{Few-Shot Teacher Active Distillation.}
Recognizing the potential for improving teacher network performance, particularly in low-data regimes, we explored few-shot learning techniques to enhance the distillation process.
After considering various approaches \citep{jia2022visual, zhou2022learning, zhou2022conditional, khattak2023maple, zhu2023prompt, khattak2023self, menghini2023enhancing, zhao2024learning, roy2023consistency, zhang2024dept, lafon2024gallop}, we incorporated CLAP \citep{silva2024closer}, a method that notably does not require a validation set \citep{silva2024closer, murugesan2024robust, morales2024bayesadapter}, as our teacher model.
For training the teacher, we followed the setting from the original CLAP paper \citep{silva2024closer}, with only one modification:
reducing the learning rate from 0.1 to 0.01 for better convergence in our setting.
This configuration represents a more realistic scenario wherein both teacher and student networks evolve simultaneously within the active distillation framework, better reflecting practical applications where pre-trained teacher models may not be available.

\paragraph{Implementation of BADGE under DHO.}
BADGE \citep{ash2019deep} adopts a hybrid approach to maximize uncertainty and diversity simultaneously, utilizing the gradient of the final linear classifier layer which has $H \times C$ elements, where $H$ is the hidden dimension of the backbone feature and $C$ is the number of classes in the teacher model.
However, DHO \citep{DHO} extends beyond conventional single classifier architectures to dual classifiers that follow different optimization paths for labeled ground truth and teacher distillation signals.
The dual classifier approach provides complementary information from both learning objectives, enhancing model performance through this specialized learning framework.
In order to adopt DHO into the BADGE selection method, we manually calculate the gradients of both classifiers and concatenate them to form a gradient representation with $2 \times H \times C$ elements.
This concatenated gradient naturally extends the original BADGE selection algorithm while preserving its core k-means++ clustering mechanism.


\begin{algorithm}[H]
    \small
    \caption{Class-Balanced Selection}
    \label{alg:cb}
    \begin{algorithmic}[1]
        \Require Unlabeled pool $\mathcal{D}^{(u)} = \{x_1, \ldots, x_m\} \subset \mathbb{R}^d$, model outputs $f_r = \{f_r(x_1), f_r(x_2), \ldots, f_r(x_n)\}$, labeled dataset $\mathcal{D}^{(l)}$, query size $K$, number of classes $C$
        \Ensure Selected set $S \subset \mathcal{D}^{(u)}$ with $|S| = K$
        \State Initialize $S = \emptyset$
        \State Set $y_i = \arg\max_{c} \, [f_r(x_i)]_c$ for all $x_i \in \mathcal{D}^{(u)}$
        \State Count $n_c = |\{x_i \in \mathcal{D}^{(l)} : y_i = c\}|$ for each $c \in [C]$
        \State Set weights $w_c = \frac{1}{n_c}$ if $n_c > 0$, else $w_c = 1$
        \State Compute $K_c = \text{round}(\frac{w_c}{\sum_{j=1}^C w_j} \cdot K)$ for each $c$
        \State Partition $\mathcal{D}^{(u)}$ into $\mathcal{D}^{(u)}_c = \{x_i \in \mathcal{D}^{(u)} : y_i = c\}$
        \For{each class $c \in [C]$}
            \State Add $\min(K_c, |\mathcal{D}^{(u)}_c|)$ random samples from $\mathcal{D}^{(u)}_c$ to $S$
        \EndFor
        \State \Return $S$
    \end{algorithmic}
\end{algorithm}

\paragraph{Implementation of Class-Balanced Selection.}
We implement the Class-Balanced Selection algorithm (\Cref{alg:cb}) as our baseline method to address class distribution bias.
The algorithm assigns pseudo-labels to unlabeled samples based on model predictions, analyzes class distribution in the labeled set, and computes inverse weights proportional to each class's representation---giving higher weights to underrepresented classes.
It then allocates the query budget across classes according to these weights and randomly selects samples from each class partition.
This approach effectively counteracts class imbalance by prioritizing underrepresented classes.


\section{Summary of Datasets}
\label{sec:datasets}

\begin{table}[htbp]
    \caption{
        Overview of datasets used in our experiments.
        Note that validation split is not used during the active learning process, and is only shown for completeness.
    }
    \label{tab:dataset_overview}
    \centering
    \small
    \setlength{\tabcolsep}{6pt}
    \renewcommand{\arraystretch}{1.2}
    \begin{tabular}{l|ccccc}
        \toprule
        \rowcolor{gray!10} \textbf{Dataset} & \textbf{\#Classes} & \textbf{\#Train} & \textbf{\#Val} & \textbf{\#Test} & \textbf{Domain} \\
        \midrule
        Caltech101 \citep{fei2004learning} & 100 & 4,128 & 1,649 & 2,465 & General objects \\
        \rowcolor{gray!5} DTD \citep{cimpoi2014describing} & 47 & 2,820 & 1,128 & 1,692 & Textures \\
        EuroSAT \citep{helber2019eurosat} & 10 & 13,500 & 5,400 & 8,100 & Satellite \\
        \rowcolor{gray!5} FGVCAircraft \citep{maji2013fine} & 100 & 3,334 & 3,333 & 3,333 & Aircraft \\
        Food101 \citep{bossard2014food} & 101 & 50,500 & 20,200 & 30,300 & Food \\
        \rowcolor{gray!5} Flowers102 \citep{nilsback2008automated} & 102 & 4,093 & 1,633 & 2,463 & Plants \\
        OxfordPets \citep{parkhi2012cats} & 37 & 2,944 & 736 & 3,669 & Animals \\
        \rowcolor{gray!5} StanfordCars \citep{krause20133d} & 196 & 6,509 & 1,635 & 8,041 & Vehicles \\
        SUN397 \citep{xiao2010sun} & 397 & 15,880 & 3,970 & 19,850 & Scenes \\
        \rowcolor{gray!5} UCF101 \citep{soomro2012ucf101} & 101 & 7,639 & 1,898 & 3,783 & Actions \\
        ImageNet \citep{russakovsky2015imagenet} & 1,000 & 1.28M & - & 50,000 & General objects \\
        \bottomrule
    \end{tabular}
\end{table}

We provide the details of datasets we used in our experiments in this section.
Our experimental evaluation encompasses 11 diverse datasets, including ImageNet \citep{russakovsky2015imagenet} and 10 additional datasets spanning various domains and classification challenges.
These datasets represent a broad spectrum of visual recognition tasks including fine-grained classification across multiple domains: generic object recognition \citep{fei2004learning}, automobile classification \citep{krause20133d}, flower species identification \citep{nilsback2008automated}, aircraft categorization \citep{maji2013fine}, pet breed classification \citep{parkhi2012cats}, food recognition \citep{bossard2014food}, scene understanding \citep{xiao2010sun}, texture classification \citep{cimpoi2014describing}, satellite imagery analysis \citep{helber2019eurosat}, and human action recognition \citep{soomro2012ucf101}.
Details of number of samples in each training, validation and test splits are illustrated in the \cref{tab:dataset_overview}.

We adhere to the train, validation, and test splits established in prior work \citep{DHO}.
Our experimental protocol begins with a 1-shot setting, where only a single labeled example per class is available, with all remaining images in the training set treated as unlabeled.
This approach extends to subsequent active learning rounds, where the unlabeled set consists of all training images not currently in the labeled set.
To further simulate realistic constraints where validation data may be inaccessible, particularly in scenarios beginning with minimal labeled examples and progressively acquiring annotations, we conduct our experiments without utilizing the validation sets.

\begin{table}[htbp]
    \caption{
        Class balance of datasets used in our experiments.
        The total number specified in the table is the number of samples of training splits as we do active selection on them.
    }
    \label{tab:dataset_statistics}
    \centering
    \small
    \setlength{\tabcolsep}{6pt}
    \renewcommand{\arraystretch}{1.2}
    \begin{tabular}{l|ccccc}
        \toprule
        \rowcolor{gray!10} \textbf{Dataset} & \textbf{Mean} & \textbf{Min} & \textbf{Max} & \textbf{Std} & \textbf{Total} \\
        \midrule
        Caltech101 \citep{fei2004learning} & 41.28 & 16 & 400 & 56.81 & 4,128 \\
        \rowcolor{gray!5} DTD \citep{cimpoi2014describing} & 60.00 & 60 & 60 & 0.00 & 2,820 \\
        EuroSAT \citep{helber2019eurosat} & 1,350.00 & 1,000 & 1,500 & 174.80 & 13,500 \\
        \rowcolor{gray!5} FGVCAircraft \citep{maji2013fine} & 33.34 & 33 & 34 & 0.48 & 3,334 \\
        Food101 \citep{bossard2014food} & 500.00 & 500 & 500 & 0.00 & 50,500 \\
        \rowcolor{gray!5} Flowers102 \citep{nilsback2008automated} & 40.13 & 20 & 129 & 22.20 & 4,093 \\
        OxfordPets \citep{parkhi2012cats} & 79.57 & 74 & 80 & 1.26 & 2,944 \\
        \rowcolor{gray!5} StanfordCars \citep{krause20133d} & 33.21 & 19 & 54 & 3.47 & 6,509 \\
        SUN397 \citep{xiao2010sun} & 40.00 & 40 & 40 & 0.00 & 15,880 \\
        \rowcolor{gray!5} UCF101 \citep{soomro2012ucf101} & 75.63 & 58 & 97 & 10.72 & 7,639 \\
        ImageNet \citep{russakovsky2015imagenet} & 1,281.17 & 732 & 1,300 & 70.22 & 1,281,167 \\
        \bottomrule
    \end{tabular}
\end{table}

We further presents class balance statistics for each dataset in the \cref{tab:dataset_statistics} to provide additional insight into inherent class imbalances, a phenomenon that several researchers have addressed within active learning selection processes \citep{aggarwal2020active, bengar2022class, huang2024class}.
It is important to note that our research addresses a distinct problem formulation; we primarily investigate teacher model bias rather than dataset bias.
Notably, our findings demonstrate that the teacher bias emerges even in balanced datasets, showing the effectiveness of our approach in active knowledge distillation scenario.


\section{Additional Experiments}
\label{sec:additional_experiments}

\subsection{In-depth Comparison across Datasets}
\label{sec:indepth_comparison}

\begin{figure}[htbp]
    \centering
    \includegraphics[width=0.99\textwidth]{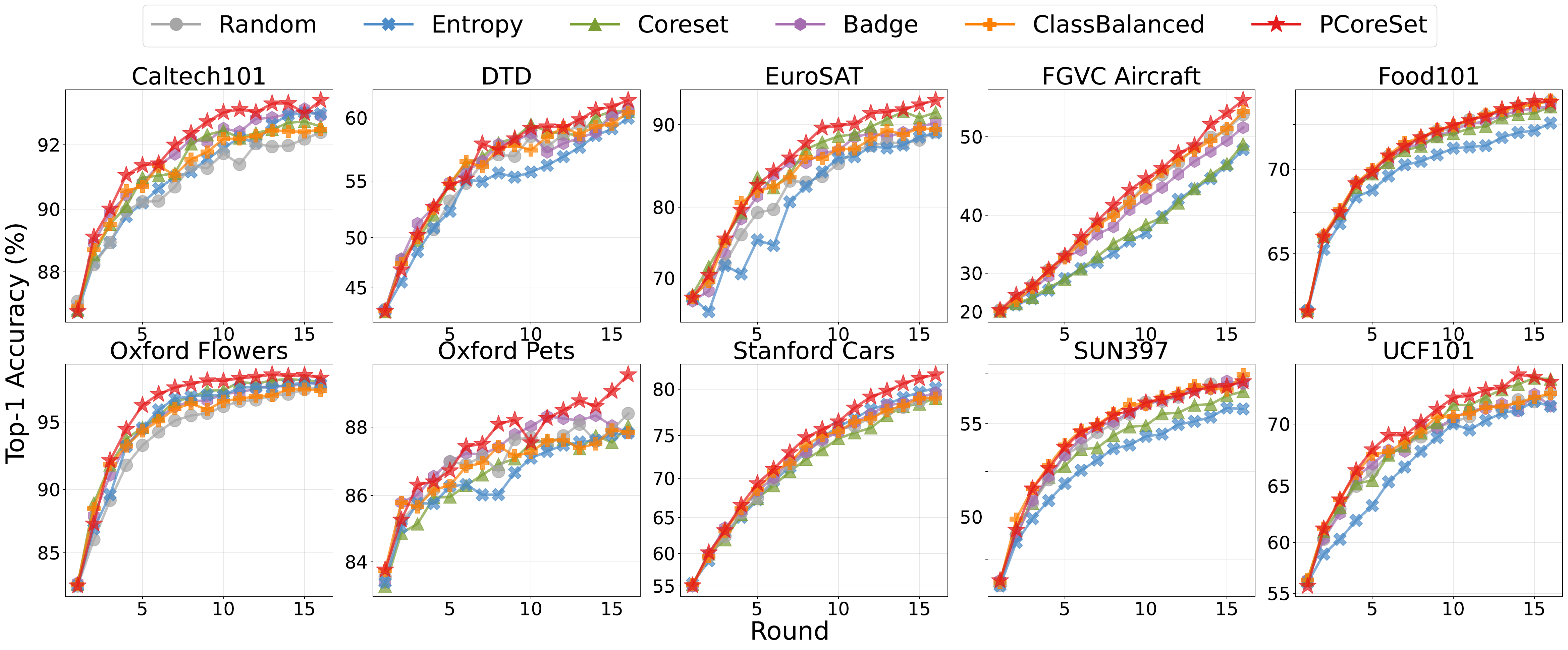}
    \vspace{-0.1in}
    \caption{
        Results on 10 datasets using \textbf{ResNet-18} with ResNet-50 teacher under ActiveKD (Zero-Shot).
    }
    \label{fig:zeroshot_10datasets}
    \vspace{-0.1in}
\end{figure}

\Cref{fig:zeroshot_10datasets} presents an in-depth evaluation of selection strategies under ActiveKD (Zero-Shot) across 10 datasets, using a ResNet-18 student and a ResNet-50 teacher.
We observe that \textbf{\texttt{PCoreSet}} achieves the best performance or is at least on par with the best across most datasets and rounds.

\subsection{Experiment on MobileNetV2}
\label{sec:architectures}

\begin{figure}[htbp]
    \centering
    \includegraphics[width=0.99\textwidth]{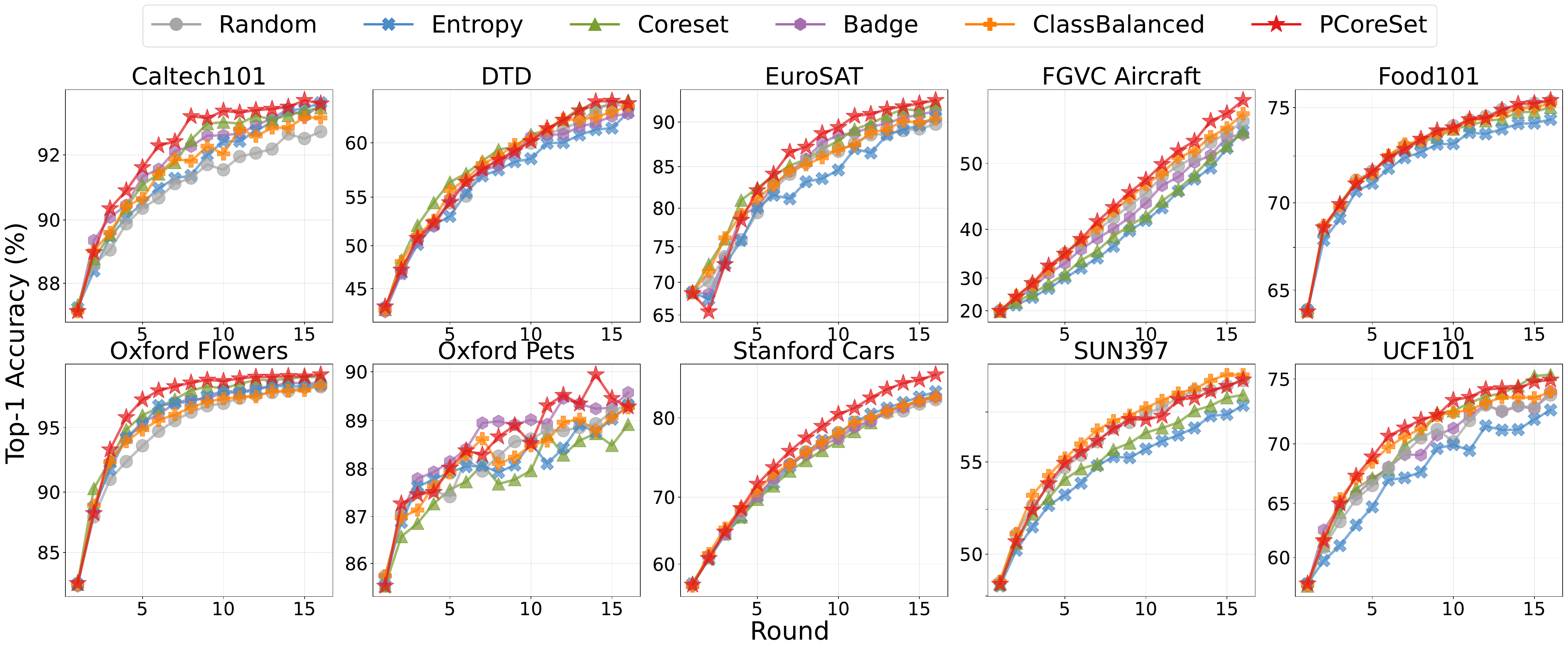}
    \caption{
    Results on 8 datasets using \textbf{MobileNetV2} under zero-shot distillation across 16 rounds.
    }
    \label{fig:mobilenet}
\end{figure}


We further validate our methodology across diverse model architectures, including the lightweight MobileNetV2 \citep{sandler2018mobilenetv2}.
The experimental results are presented in \cref{fig:mobilenet}. For these experiments, we maintained the training configuration established for ResNet-18 \citep{he2016deep}, employing a ResNet-50 CLIP teacher \citep{radford2021learning} with identical hyperparameters.
Our empirical results demonstrate that \textbf{\texttt{PCoreSet}} consistently outperforms alternative selection methods in the active distillation setting, exhibiting performance trends that align with our primary ResNet-18 \citep{he2016deep} experiments.


\subsection{Active Learning without Knowledge Distillation}
\label{sec:nodistill}

\begin{figure}[htbp]
    \centering
    \includegraphics[width=0.99\textwidth]{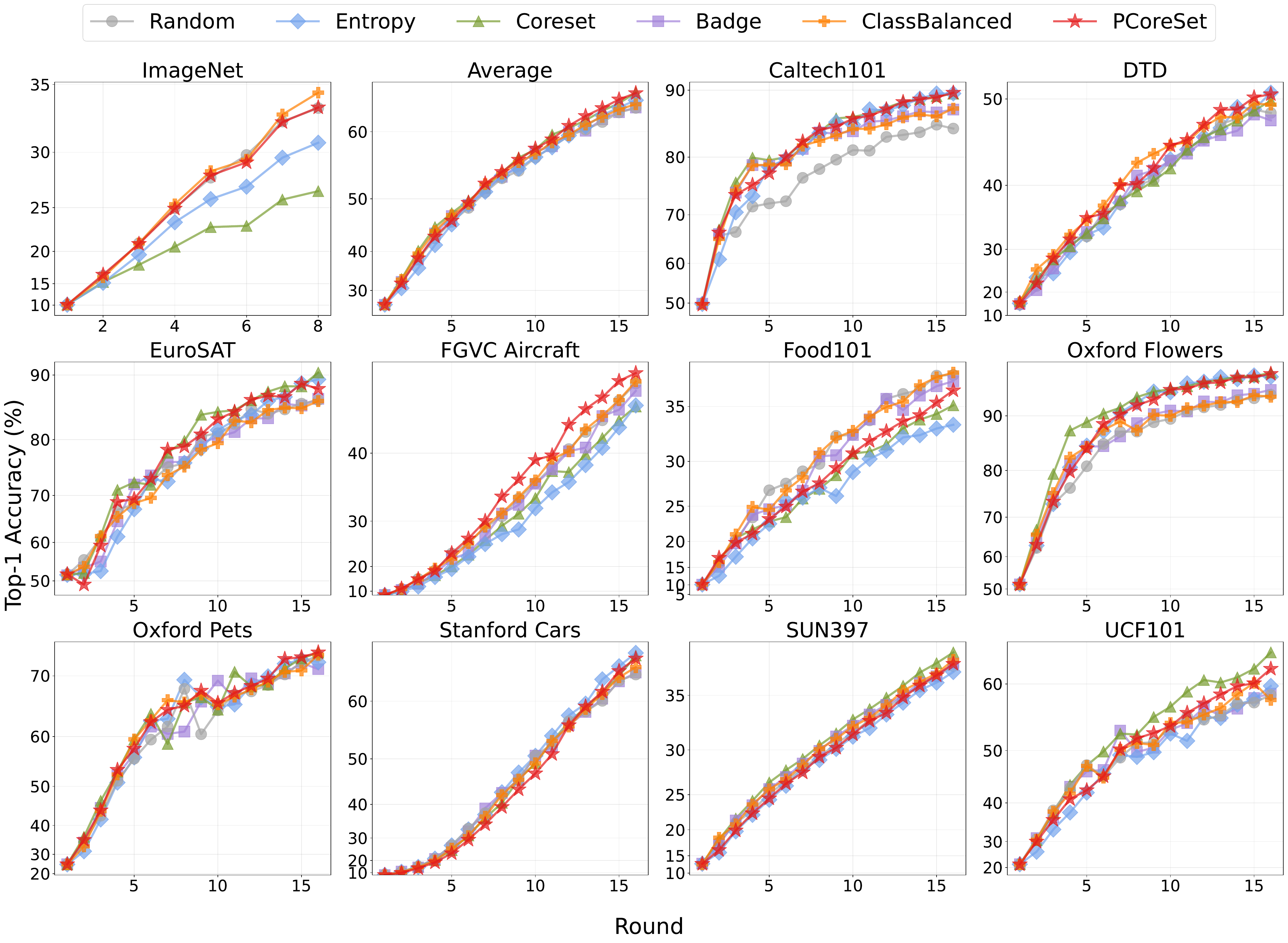}
    \caption{
        Results on 11 datasets including average of 10 datasets using \textbf{ResNet-18} without distillation.
    }
    \label{fig:nodistill}
\end{figure}

The main purpose of our work is to propose utilizing a foundational teacher model directly in the knowledge distillation process \citep{hinton2015distilling}.
Thus, we conducted active learning experiments without knowledge distillation as our baseline to enable direct comparison.
Using ResNet-18 \citep{he2016deep} and following our main experimental protocol in \cref{sec:implementation}, we performed comparative analysis of knowledge distillation effects in the active learning process.
The results are presented in \cref{fig:nodistill}.
Interestingly, widely-adopted baselines such as uncertainty \citep{holub2008entropy}, coreset \citep{sener2017active}, and badge \citep{ash2019deep} do not consistently outperform random or class-balanced selection methods \cref{alg:cb} on datasets such as FGVC \citep{maji2013fine}, with our \textbf{\texttt{PCoreSet}} \cref{alg:pcoreset} achieving superior performance.
This occurs despite the approximately equal distribution of classes in these datasets as shown in \cref{tab:dataset_statistics} in the \cref{sec:datasets}, suggesting that maintaining probabilistic balance may be beneficial even in traditional active learning scenarios.
However, \textbf{\texttt{PCoreSet}} does not demonstrate clear effectiveness in this traditional active learning setting, indicating that maintaining probabilistic diversity is specifically effective in distillation scenarios, where it serves to leverage teacher's structured bias propagation.


\subsection{Few-shot Teacher Distillation}
\label{sec:fewshot}

\begin{figure}[htbp]
    \centering
    \includegraphics[width=0.99\textwidth]{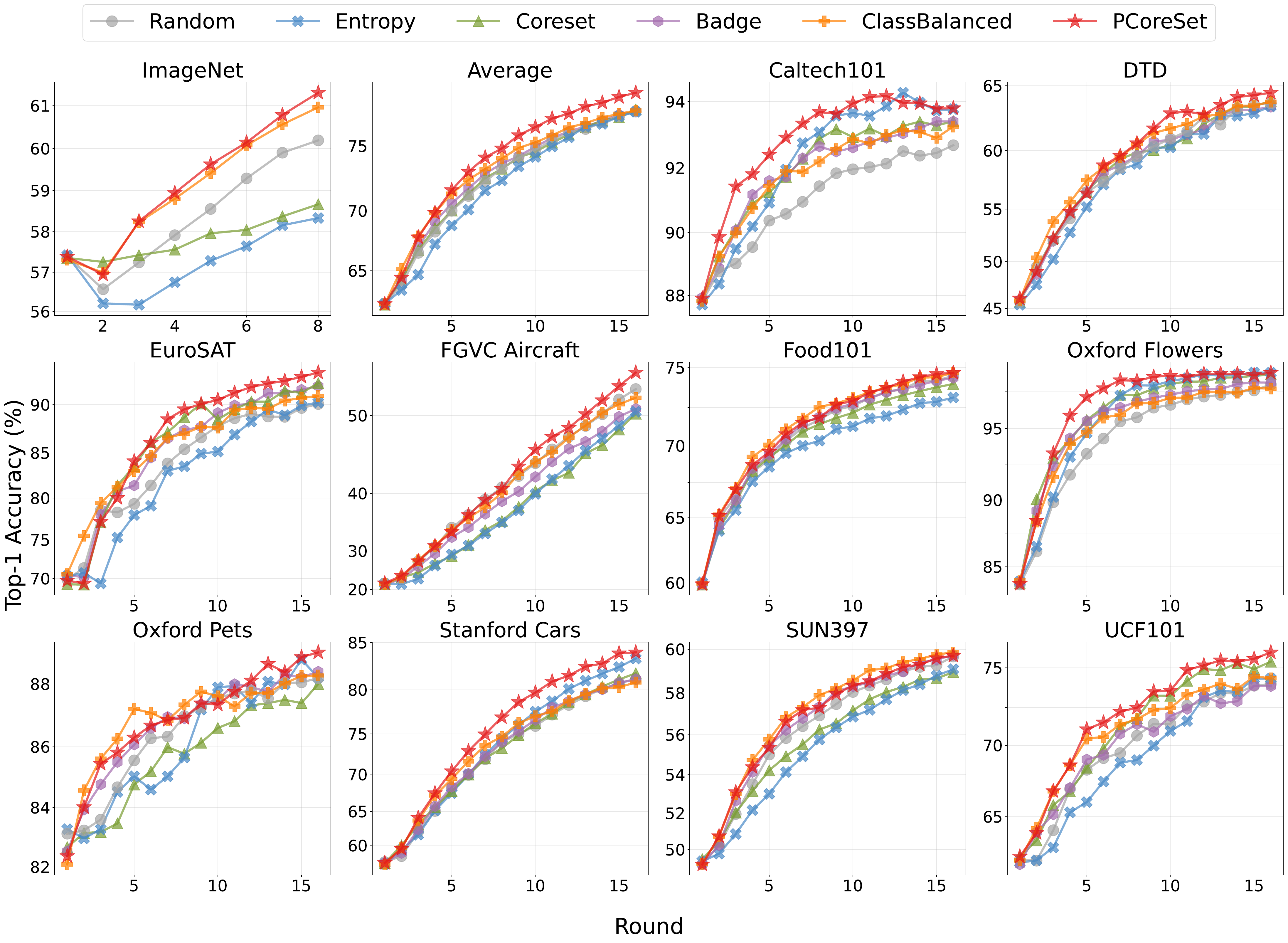}
    \caption{
        Results of \textbf{ResNet-18 student} on 11 datasets including average of 10 datasets under few-shot distillation.
    }
    \label{fig:fewshot_student}
\end{figure}

\begin{figure}[htbp]
    \centering
    \includegraphics[width=0.99\textwidth]{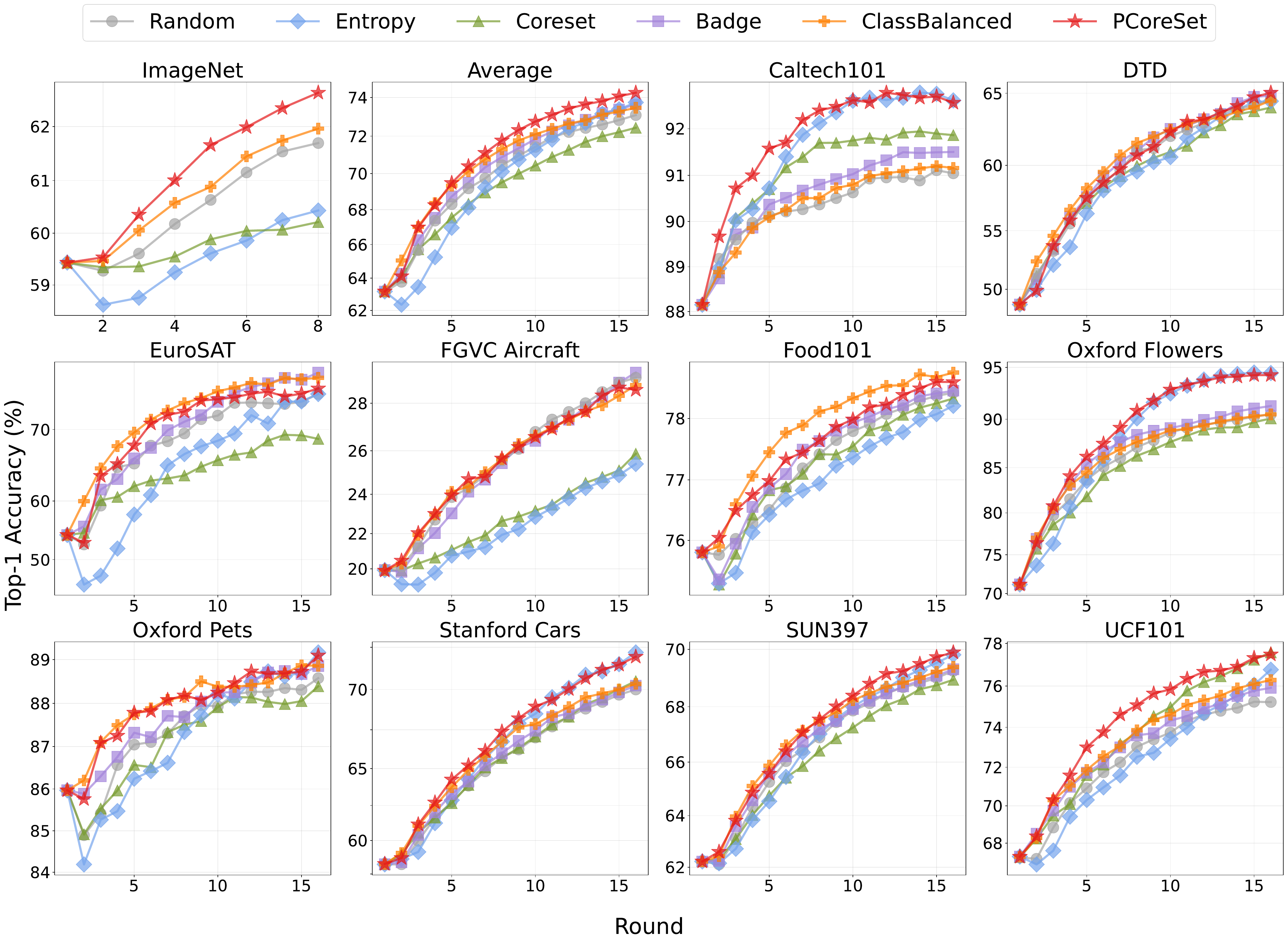}
    \caption{
        Results of \textbf{ResNet-50 few-shot teacher} on 11 datasets including average of 10 datasets under few-shot distillation.
    }
    \label{fig:fewshot_teacher}
\end{figure}

In real-world applications, it is more natural to consider scenarios where both teacher and student models evolve concurrently.
Specifically, we can apply few-shot learning methods to the generalist foundational teacher model with additional labeled samples acquired in each round, though our main experiments focused on zero-shot teachers to validate the effectiveness of integrating knowledge distillation into the active learning framework.
Therefore, we extended our experiments to incorporate few-shot teachers that are updated with newly labeled samples in each acquisition round.
For this purpose, we employed CLAP \citep{silva2024closer}, a few-shot learning method built upon CLIP that does not require a validation set---an appropriate choice for our setting where we assume extremely limited data availability (e.g., 1-shot training datasets) without access to validation data.
The experimental results for student model performance are presented in \cref{fig:fewshot_student}, with corresponding teacher model performance shown in \cref{fig:fewshot_teacher}.
In the few-shot teacher scenario, we observed consistent trends where probabilistic diversity effectively leverages bias, resulting in enhanced performance.
Moreover, examining the few-shot teacher performance in \cref{fig:fewshot_teacher} reveals that our \textbf{\texttt{PCoreSet}} outperforms alternative selection methods for teacher models as well.
We attribute this to the fact that samples selected to leverage teacher bias naturally serve as beneficial examples for improving the teacher model, simultaneously enhancing performance on both the teacher and student sides.


\subsection{Logit Adjustment for Calibrated Teacher Prediction}
\label{sec:logit_adjustment}

In this section, we investigate whether logit adjustment~\citep{menon2020long}, a technique designed for imbalanced learning, can mitigate the structured prediction bias in teacher models.
While logit adjustment (LA) provides a principled approach to handling class imbalance, we demonstrate that it is \textbf{orthogonal} to our \textbf{\texttt{PCoreSet}} method and can be used in conjunction to potentially enhance performance.

Logit adjustment~\citep{menon2020long} addresses \textbf{skewed, imbalanced} class distributions in training data by applying a \textbf{principled and effective} post-hoc adjustment based on empirical class frequencies.
Formally, $l' := l - \tau \log \pi \in \Delta^{C-1}$, where $l$ and $l'$ denote the original and adjusted logits, respectively.
$\pi \in \Delta^{C-1}$ is the class prior estimated from the empirical distribution, i.e., $\pi := \mathbb{E}_{x \sim p(x)}[p(y \mid x)] \approx \tfrac{1}{N}\sum_{n=1}^{N} y_n$, with $y_n \in \{0,1\}^C$ one-hot labels in the supervised training dataset $\mathcal{D}^{(l)}$.

Given that teacher predictions in our setting are often skewed and biased, we explore whether adjusting \textbf{teacher logits} can provide an \textbf{improved inductive bias}.
Our hypothesis is that \texttt{PCoreSet} can then leverage this calibrated teacher signal more effectively.

We consider two logit-adjustment variants:
\begin{enumerate}
    \item \textbf{Hard LA.} Let $\mathcal{D} = \mathcal{D}^{(l)} \cup \mathcal{D}^{(u)}$.
    Estimate the class prior by \textbf{pseudo-label frequencies} on $\mathcal{D}$ with teacher predictions $f(x)$:
    \begin{equation}
        \pi_c \approx \frac{1}{|\mathcal{D}|}\sum_{x \in \mathcal{D}}\mathbf{1} \left[\arg\max_{c'} [f(x)]_{c'} = c\right], \qquad \pi = [\pi_1, \ldots, \pi_C]^\top,
    \end{equation}
    where $\mathbf{1}$ is an indicator function.
    
    \item \textbf{Soft LA.} Estimate the class prior by the \textbf{average teacher probabilities}:
    \begin{equation}
        \pi \approx \frac{1}{|\mathcal{D}|}\sum_{x \in \mathcal{D}}f(x).
    \end{equation}
\end{enumerate}

We conducted experiments on 8 of the 10 datasets by integrating LA into teacher predictions.
Base results are averaged over 5 seeds, while Soft LA results are averaged over 3 seeds.

\begin{table}[t!]
\centering
\caption{\textbf{Logit adjustment (Hard / Soft) under ActiveKD on 8 datasets.} Mean accuracy across rounds; gray values in parentheses denote the difference from the no-LA baseline. Bold marks the best per round.}
\label{tab:logit_adjustment}
\scriptsize
\setlength{\tabcolsep}{4pt}
\resizebox{\textwidth}{!}{%
\begin{tabular}{@{}llccccccc@{}}
\toprule
\multirow{2}{*}{\textbf{Strategy}} & \multirow{2}{*}{\textbf{LA}} & \multicolumn{7}{c}{\textbf{Active Learning Round}} \\
\cmidrule(lr){3-9}
 & & 2 & 3 & 4 & 5 & 6 & 7 & 8 \\
\midrule
\multirow{3}{*}{Random} & --   & 65.1 & 67.4 & 69.5 & 71.3 & 72.4 & 73.9 & 74.5 \\
 & Hard & 55.1\,\textcolor{gray}{\scriptsize (-10.1)} & 59.9\,\textcolor{gray}{\scriptsize (-7.4)} & 63.1\,\textcolor{gray}{\scriptsize (-6.5)} & 66.0\,\textcolor{gray}{\scriptsize (-5.3)} & 67.5\,\textcolor{gray}{\scriptsize (-5.0)} & 69.4\,\textcolor{gray}{\scriptsize (-4.5)} & 70.6\,\textcolor{gray}{\scriptsize (-3.9)} \\
 & Soft & 64.4\,\textcolor{gray}{\scriptsize (-0.7)} & 67.6\,\textcolor{gray}{\scriptsize (+0.2)} & 69.9\,\textcolor{gray}{\scriptsize (+0.3)} & 71.5\,\textcolor{gray}{\scriptsize (+0.2)} & 72.8\,\textcolor{gray}{\scriptsize (+0.4)} & 74.0\,\textcolor{gray}{\scriptsize (+0.1)} & 75.0\,\textcolor{gray}{\scriptsize (+0.5)} \\
\midrule
\multirow{3}{*}{Coreset} & --   & 65.7 & 67.7 & 69.8 & 71.5 & 72.4 & 73.5 & 74.8 \\
 & Hard & 55.5\,\textcolor{gray}{\scriptsize (-10.2)} & 60.6\,\textcolor{gray}{\scriptsize (-7.1)} & 63.6\,\textcolor{gray}{\scriptsize (-6.2)} & 66.6\,\textcolor{gray}{\scriptsize (-4.9)} & 68.8\,\textcolor{gray}{\scriptsize (-3.6)} & 70.1\,\textcolor{gray}{\scriptsize (-3.4)} & 71.9\,\textcolor{gray}{\scriptsize (-2.9)} \\
 & Soft & 65.3\,\textcolor{gray}{\scriptsize (-0.4)} & 68.1\,\textcolor{gray}{\scriptsize (+0.4)} & 70.4\,\textcolor{gray}{\scriptsize (+0.6)} & 71.4\,\textcolor{gray}{\scriptsize (-0.1)} & 73.0\,\textcolor{gray}{\scriptsize (+0.6)} & 73.6\,\textcolor{gray}{\scriptsize (+0.2)} & 75.3\,\textcolor{gray}{\scriptsize (+0.5)} \\
\midrule
\multirow{3}{*}{Uncertainty} & --   & 63.8 & 66.5 & 67.9 & 69.8 & 71.1 & 72.4 & 73.4 \\
 & Hard & 52.6\,\textcolor{gray}{\scriptsize (-11.2)} & 57.4\,\textcolor{gray}{\scriptsize (-9.1)} & 61.0\,\textcolor{gray}{\scriptsize (-6.9)} & 63.7\,\textcolor{gray}{\scriptsize (-6.1)} & 65.6\,\textcolor{gray}{\scriptsize (-5.5)} & 68.5\,\textcolor{gray}{\scriptsize (-3.9)} & 69.9\,\textcolor{gray}{\scriptsize (-3.5)} \\
 & Soft & 63.3\,\textcolor{gray}{\scriptsize (-0.5)} & 66.3\,\textcolor{gray}{\scriptsize (-0.2)} & 68.1\,\textcolor{gray}{\scriptsize (+0.2)} & 70.0\,\textcolor{gray}{\scriptsize (+0.2)} & 71.1\,\textcolor{gray}{\scriptsize (+0.0)} & 72.3\,\textcolor{gray}{\scriptsize (-0.1)} & 73.2\,\textcolor{gray}{\scriptsize (-0.2)} \\
\midrule
\multirow{3}{*}{BADGE} & --   & 65.3 & 68.1 & 70.3 & 72.1 & 73.3 & 74.3 & 75.1 \\
 & Hard & 54.6\,\textcolor{gray}{\scriptsize (-10.7)} & 59.6\,\textcolor{gray}{\scriptsize (-8.4)} & 63.7\,\textcolor{gray}{\scriptsize (-6.6)} & 66.4\,\textcolor{gray}{\scriptsize (-5.7)} & 68.0\,\textcolor{gray}{\scriptsize (-5.2)} & 69.9\,\textcolor{gray}{\scriptsize (-4.4)} & 71.1\,\textcolor{gray}{\scriptsize (-4.0)} \\
 & Soft & 64.5\,\textcolor{gray}{\scriptsize (-0.8)} & 67.9\,\textcolor{gray}{\scriptsize (-0.2)} & 70.2\,\textcolor{gray}{\scriptsize (-0.1)} & 72.2\,\textcolor{gray}{\scriptsize (+0.1)} & 73.4\,\textcolor{gray}{\scriptsize (+0.1)} & 74.5\,\textcolor{gray}{\scriptsize (+0.2)} & 75.6\,\textcolor{gray}{\scriptsize (+0.5)} \\
\midrule
\multirow{3}{*}{ClassBalanced} & --   & 65.6 & 68.3 & 70.8 & 72.2 & 73.3 & 74.1 & 75.3 \\
 & Hard & 55.4\,\textcolor{gray}{\scriptsize (-10.2)} & 58.8\,\textcolor{gray}{\scriptsize (-9.4)} & 62.6\,\textcolor{gray}{\scriptsize (-8.2)} & 65.7\,\textcolor{gray}{\scriptsize (-6.5)} & 67.2\,\textcolor{gray}{\scriptsize (-6.1)} & 69.2\,\textcolor{gray}{\scriptsize (-4.9)} & 70.3\,\textcolor{gray}{\scriptsize (-5.0)} \\
 & Soft & 65.1\,\textcolor{gray}{\scriptsize (-0.5)} & 69.1\,\textcolor{gray}{\scriptsize (+0.8)} & 71.3\,\textcolor{gray}{\scriptsize (+0.5)} & 72.9\,\textcolor{gray}{\scriptsize (+0.7)} & 74.1\,\textcolor{gray}{\scriptsize (+0.8)} & 75.0\,\textcolor{gray}{\scriptsize (+0.9)} & 75.7\,\textcolor{gray}{\scriptsize (+0.4)} \\
\midrule
\rowcolor{gray!12}
 & --   & 65.7 & 68.6 & 71.0 & 72.8 & 74.1 & 75.3 & 76.2 \\
\rowcolor{gray!12}
 & Hard & 56.1\,\textcolor{gray}{\scriptsize (-9.5)} & 61.8\,\textcolor{gray}{\scriptsize (-6.8)} & 64.6\,\textcolor{gray}{\scriptsize (-6.4)} & 67.8\,\textcolor{gray}{\scriptsize (-5.1)} & 70.4\,\textcolor{gray}{\scriptsize (-3.6)} & 71.7\,\textcolor{gray}{\scriptsize (-3.6)} & 73.4\,\textcolor{gray}{\scriptsize (-2.8)} \\
\rowcolor{gray!12}
\multirow{-3}{*}{\textbf{\texttt{PCoreSet}}} & Soft & \textbf{65.9}\,\textcolor{gray}{\scriptsize (+0.2)} & \textbf{69.2}\,\textcolor{gray}{\scriptsize (+0.6)} & \textbf{71.5}\,\textcolor{gray}{\scriptsize (+0.5)} & \textbf{73.1}\,\textcolor{gray}{\scriptsize (+0.3)} & \textbf{74.1}\,\textcolor{gray}{\scriptsize (+0.1)} & \textbf{75.3}\,\textcolor{gray}{\scriptsize (+0.0)} & \textbf{76.3}\,\textcolor{gray}{\scriptsize (+0.1)} \\
\bottomrule
\end{tabular}
}
\end{table}

As detailed in \Cref{tab:logit_adjustment}, we observe that Soft LA \textbf{slightly improves} \textbf{\texttt{PCoreSet}} on average, even without exhaustive tuning of the hyperparameter $\tau$.
These results indicate that while logit adjustment can provide marginal benefits, the core strength of \textbf{\texttt{PCoreSet}} lies in its ability to effectively leverage \textbf{structured teacher prediction bias} within the ActiveKD framework, yielding substantial gains across diverse datasets even without additional calibration techniques.

\subsection{Ablation Study on Effectiveness of Unlabeled Data for KD}

\begin{table}[t!]
\centering
\caption{\textbf{Results on 8 datasets (average) when excluding unlabeled data from KD} (denoted as w/o unlabeled). Bold marks the best selector per round within each KD configuration.}
\label{tab:wo_unlabeled}
\scriptsize
\setlength{\tabcolsep}{4pt}
\resizebox{\textwidth}{!}{
\begin{tabular}{l l *{15}{c}}
\toprule
Method & KD & \multicolumn{15}{c}{Round} \\
\cmidrule(lr){3-17}
& & 2 & 3 & 4 & 5 & 6 & 7 & 8 & 9 & 10 & 11 & 12 & 13 & 14 & 15 & 16 \\
\midrule
\multirow{3}{*}{Random}
  & No Distill & 36.7 & 42.6 & 47.6 & 50.9 & 53.8 & 56.7 & 59.2 & 60.1 & 62.5 & 64.1 & 65.6 & 66.5 & 67.6 & 68.9 & 69.5 \\
  & \makecell[l]{ActiveKD\\(w/o unlabeled)} & 43.3 & 49.6 & 55.2 & 58.8 & 61.9 & 63.9 & 66.2 & 67.7 & 69.3 & 70.7 & 71.9 & 72.8 & 73.9 & 74.2 & 75.0 \\
  & ActiveKD & 64.5 & 66.8 & 69.0 & 70.8 & 72.0 & 73.5 & 74.1 & 74.8 & 75.8 & 76.2 & 76.9 & 77.4 & 77.6 & 78.0 & 78.7 \\
\midrule
\multirow{3}{*}{Entropy}
  & No Distill & 34.8 & 40.6 & 46.4 & 51.0 & 54.8 & 57.2 & 60.1 & 61.4 & 63.1 & 64.5 & 66.3 & 67.6 & 69.2 & 70.3 & 71.3 \\
  & \makecell[l]{ActiveKD\\(w/o unlabeled)} & 41.1 & 48.6 & 55.8 & 57.5 & 61.2 & 65.1 & 66.9 & 68.7 & 69.8 & 71.0 & 72.8 & 74.1 & 74.8 & 76.3 & 76.7 \\
  & ActiveKD & 63.8 & 66.5 & 67.9 & 69.8 & 71.1 & 72.4 & 73.4 & 74.3 & 75.1 & 75.7 & 76.5 & 76.9 & 77.5 & 78.0 & 78.4 \\
\midrule
\multirow{3}{*}{Coreset}
  & No Distill & \textbf{37.3} & \textbf{44.9} & \textbf{50.4} & \textbf{53.2} & \textbf{55.7} & 57.9 & 60.5 & \textbf{63.0} & 64.0 & \textbf{66.4} & 67.3 & 68.3 & 69.6 & 70.6 & \textbf{72.0} \\
  & \makecell[l]{ActiveKD\\(w/o unlabeled)} & \textbf{44.2} & 51.0 & 55.3 & 59.8 & 63.0 & 65.1 & 67.5 & 69.2 & 71.1 & 71.1 & 72.9 & \textbf{74.2} & 75.0 & 75.6 & 76.7 \\
  & ActiveKD & \textbf{65.7} & 67.7 & 69.8 & 71.5 & 72.4 & 73.5 & 74.8 & 75.5 & 76.3 & 76.5 & 77.0 & 77.6 & 78.3 & 78.5 & 79.0 \\
\midrule
\multirow{3}{*}{Badge}
  & No Distill & 36.6 & 42.6 & 48.7 & 52.4 & 54.7 & 57.7 & 59.4 & 61.3 & 63.5 & 64.1 & 66.2 & 66.4 & 68.0 & 69.1 & 69.6 \\
  & \makecell[l]{ActiveKD\\(w/o unlabeled)} & 43.4 & 50.4 & 56.6 & 59.2 & 61.6 & 63.9 & 66.4 & 67.3 & 68.9 & 70.5 & 71.5 & 72.2 & 73.1 & 74.5 & 74.9 \\
  & ActiveKD & 65.3 & 68.0 & 70.3 & 72.1 & 73.3 & 74.3 & 75.1 & 75.8 & 76.4 & 76.9 & 77.4 & 77.8 & 78.1 & 78.7 & 78.9 \\
\midrule
\multirow{3}{*}{ClassBalanced}
  & No Distill & 37.1 & 44.2 & 49.2 & 52.7 & 54.7 & 58.4 & 59.9 & 61.7 & 63.0 & 64.8 & 65.8 & 67.3 & 68.3 & 69.4 & 70.0 \\
  & \makecell[l]{ActiveKD\\(w/o unlabeled)} & 43.4 & \textbf{51.3} & \textbf{57.3} & 59.9 & 63.9 & 65.3 & 67.8 & 68.6 & 70.0 & 72.1 & 72.5 & \textbf{74.2} & 74.5 & 75.2 & 76.1 \\
  & ActiveKD & 65.5 & 68.3 & 70.8 & 72.2 & 73.3 & 74.1 & 75.3 & 75.8 & 76.3 & 76.9 & 77.5 & 77.8 & 78.1 & 78.6 & 79.0 \\
\midrule
\rowcolor{gray!12}
  & No Distill & 36.0 & 43.2 & 48.5 & 51.7 & 55.2 & \textbf{58.8} & \textbf{60.7} & 62.7 & \textbf{64.3} & 65.6 & \textbf{67.7} & \textbf{69.0} & \textbf{70.1} & \textbf{71.2} & 72.0 \\
\rowcolor{gray!12}
  & \makecell[l]{ActiveKD\\(w/o unlabeled)} & 43.0 & 50.1 & 56.4 & \textbf{60.2} & \textbf{64.0} & \textbf{66.1} & \textbf{68.0} & \textbf{69.7} & \textbf{71.4} & \textbf{72.9} & \textbf{73.8} & 74.0 & \textbf{76.0} & \textbf{76.7} & \textbf{77.3} \\
\rowcolor{gray!12}
\multirow{-3}{*}{\texttt{\textbf{PCoreSet}}}
  & ActiveKD & \textbf{65.7} & \textbf{68.6} & \textbf{71.0} & \textbf{72.8} & \textbf{74.0} & \textbf{75.3} & \textbf{76.2} & \textbf{77.2} & \textbf{77.6} & \textbf{78.2} & \textbf{78.8} & \textbf{79.2} & \textbf{79.8} & \textbf{80.1} & \textbf{80.4} \\
\bottomrule
\end{tabular}}
\vspace{-0.1in}
\end{table}

To investigate the contribution of unlabeled data to knowledge distillation (KD) under \textsc{ActiveKD}, we conduct additional experiments by implementing \textbf{ActiveKD (w/o unlabeled)}, which excludes the unlabeled set $\mathcal{D}^{(u)}$ from the KD loss $\mathcal{L}_{\text{KD}}$ in \Cref{eq:kd_loss}. Specifically, under this variant, the KD loss reduces to $\mathcal{L}_{\text{KD}} = \frac{1}{N}\sum_{n} \KL \!\left[f(x_n^{(l)}) \,\|\, f_r(x_n^{(l)})\right]$, without the unlabeled component $\frac{1}{M}\sum_{m} \KL \!\left[f(x_m^{(u)}) \,\|\, f_r(x_m^{(u)})\right]$.
We follow the same experimental setup as in \Cref{fig:zeroshot_10datasets}, using zero-shot teachers and 8 out of the 10 datasets. \Cref{tab:wo_unlabeled} reports the results for three configurations: (1) \textbf{No Distill}, (2) \textbf{ActiveKD (w/o unlabeled)}, and (3) \textbf{ActiveKD}. Here, (2) serves as an ablation to isolate the effect of unlabeled data during KD.
We observe that \textbf{ActiveKD (w/o unlabeled)} consistently improves over \textbf{No Distill} across all selection algorithms. However, the improvement is noticeably smaller than full \textbf{ActiveKD}. For example, at the 16th AL round with $\texttt{PCoreSet}$, the gain is 5.36 points, compared to 8.49 points with full \textsc{ActiveKD}.
More importantly, even under this restricted setting, $\texttt{PCoreSet}$ achieves the best performance in \textbf{11 out of 15} rounds (73.3\%), demonstrating its robustness and effectiveness in \emph{semi-supervised knowledge distillation using labeled samples only}.




\section{Additional Qualitative Results}
\label{sec:additional_results}


\subsection{Heatmap of Selected Samples}
\label{sec:heatmap_of_selected_samples}

We visualize the distribution of selected samples using heatmaps in \Cref{fig:10datasets_heatmap}.

\begin{figure}[htbp]
    \centering
    \includegraphics[width=0.8\textwidth]{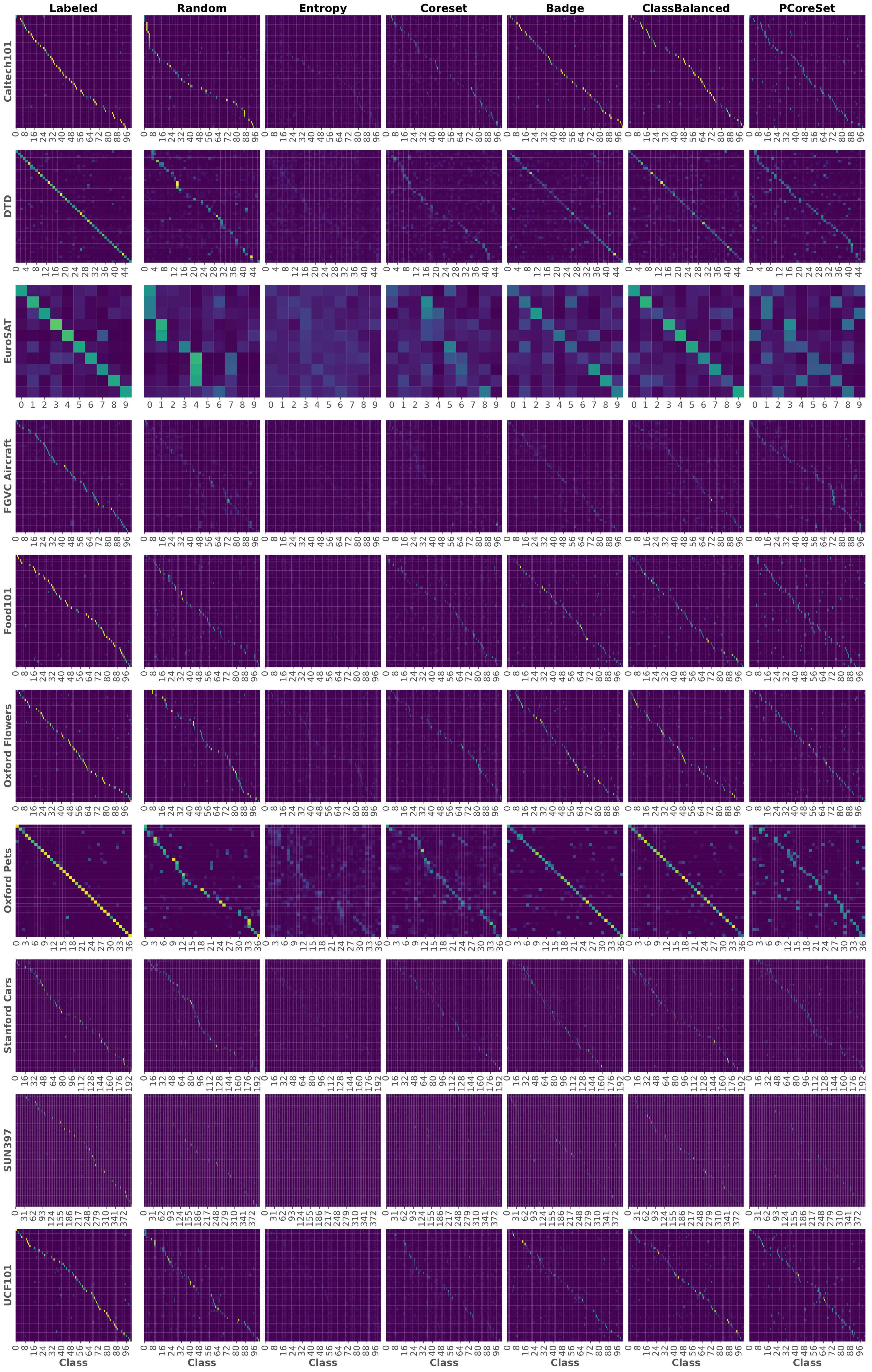}
    \caption{
        The heatmap of output probability vectors from different selection strategies in the first active learning round using 10 datasets.
    }
    \label{fig:10datasets_heatmap}
\end{figure}

\clearpage


\subsection{Qualitative Results on Selected Samples from \textbf{\texttt{PCoreSet}}}
\label{sec:qualitative_results}

We visualize samples selected by \textbf{\texttt{PCoreSet}} from various datasets in Figures \ref{fig:selected_samples_caltech} through \ref{fig:selected_samples_pets}, showcasing how \textbf{\texttt{PCoreSet}} selects diverse and representative examples across different visual domains.
These visualizations provide qualitative evidence of \textbf{\texttt{PCoreSet}}'s effectiveness in identifying informative samples that help leverage the structured bias of the teacher model.

\begin{figure}[htbp]
    \centering
    \includegraphics[width=0.99\textwidth,trim={0 0 0 0.3in},clip]{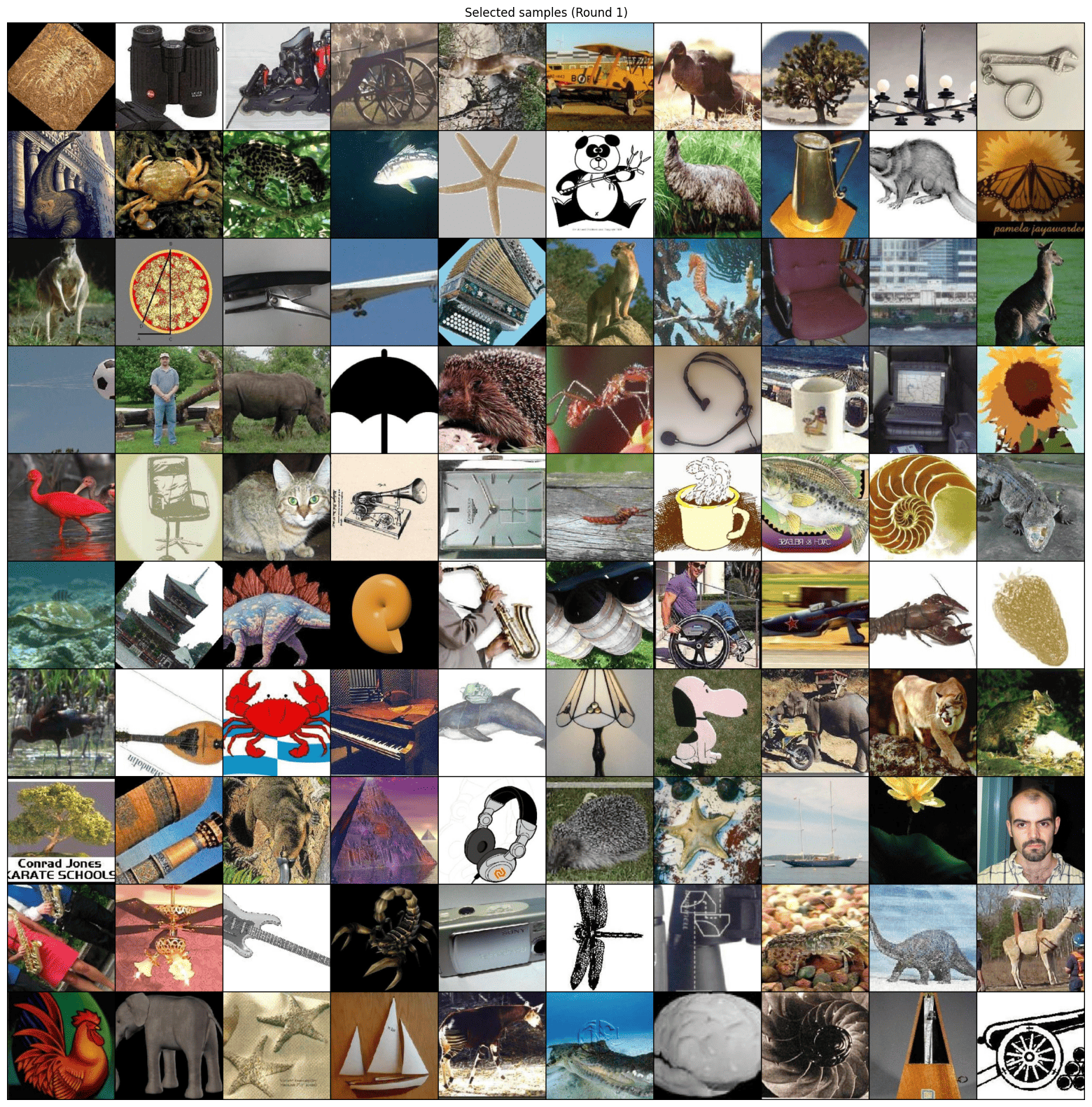}
    \caption{Visualization of samples selected by \textbf{\texttt{PCoreSet}} for Caltech101 dataset.}
    \label{fig:selected_samples_caltech}
\end{figure}

\begin{figure}[htbp]
    \centering
    \includegraphics[width=0.99\textwidth,trim={0 0 0 0.3in},clip]{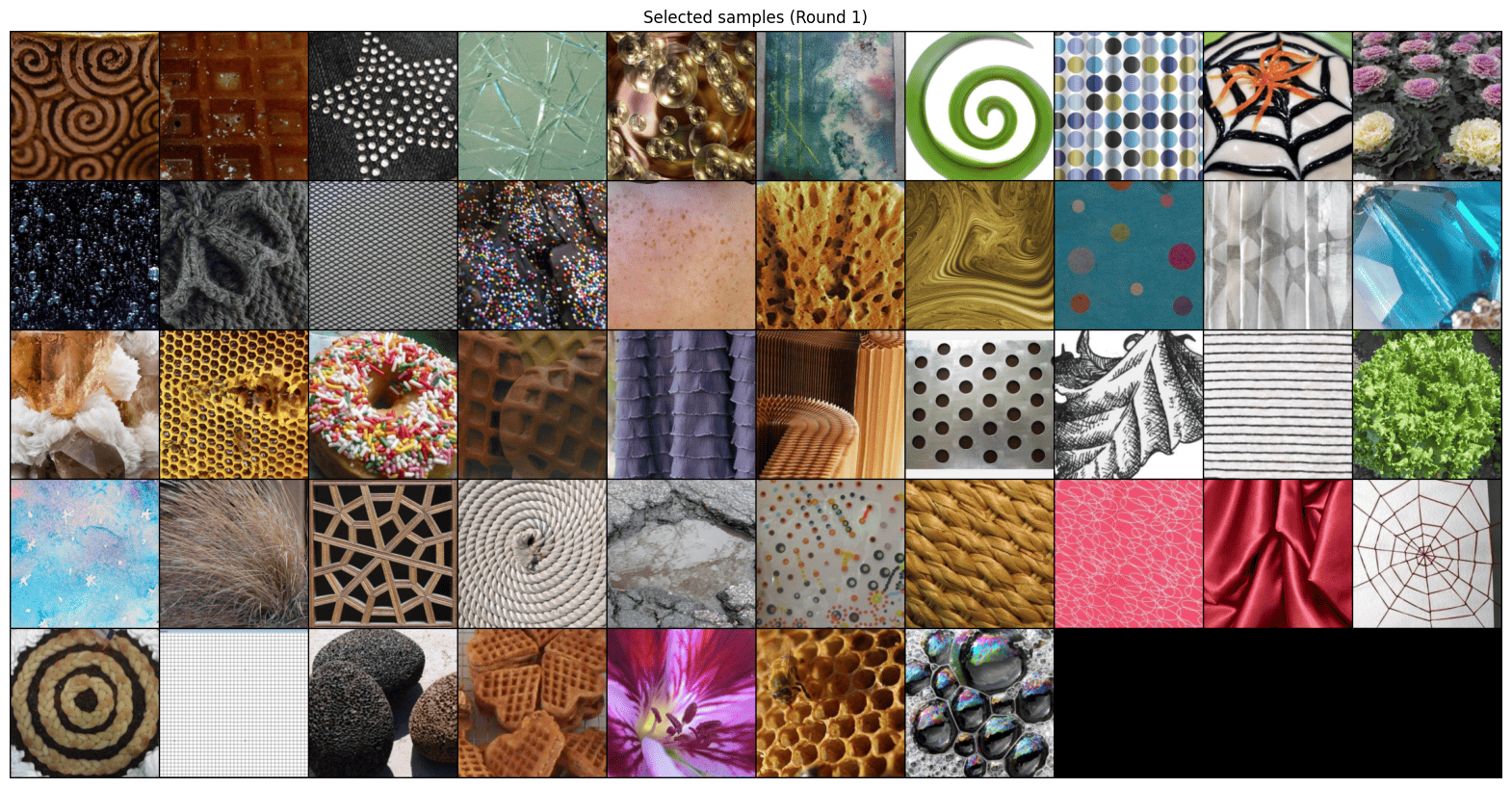}
    \caption{Visualization of samples selected by \textbf{\texttt{PCoreSet}} for DTD dataset.}
    \label{fig:selected_samples_dtd}
\end{figure}

\begin{figure}[htbp]
    \centering
    \includegraphics[width=0.99\textwidth,trim={0 0 0 0.3in},clip]{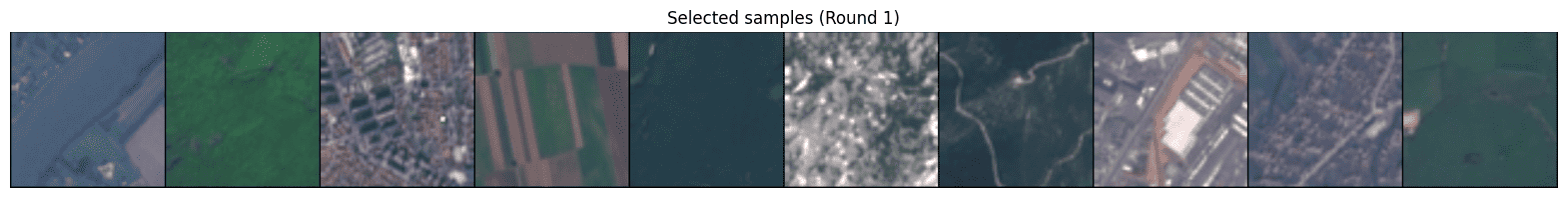}
    \caption{Visualization of samples selected by \textbf{\texttt{PCoreSet}} for EuroSAT dataset.}
    \label{fig:selected_samples_eurosat}
\end{figure}

\begin{figure}[htbp]
    \centering
    \includegraphics[width=0.99\textwidth,trim={0 0 0 0.3in},clip]{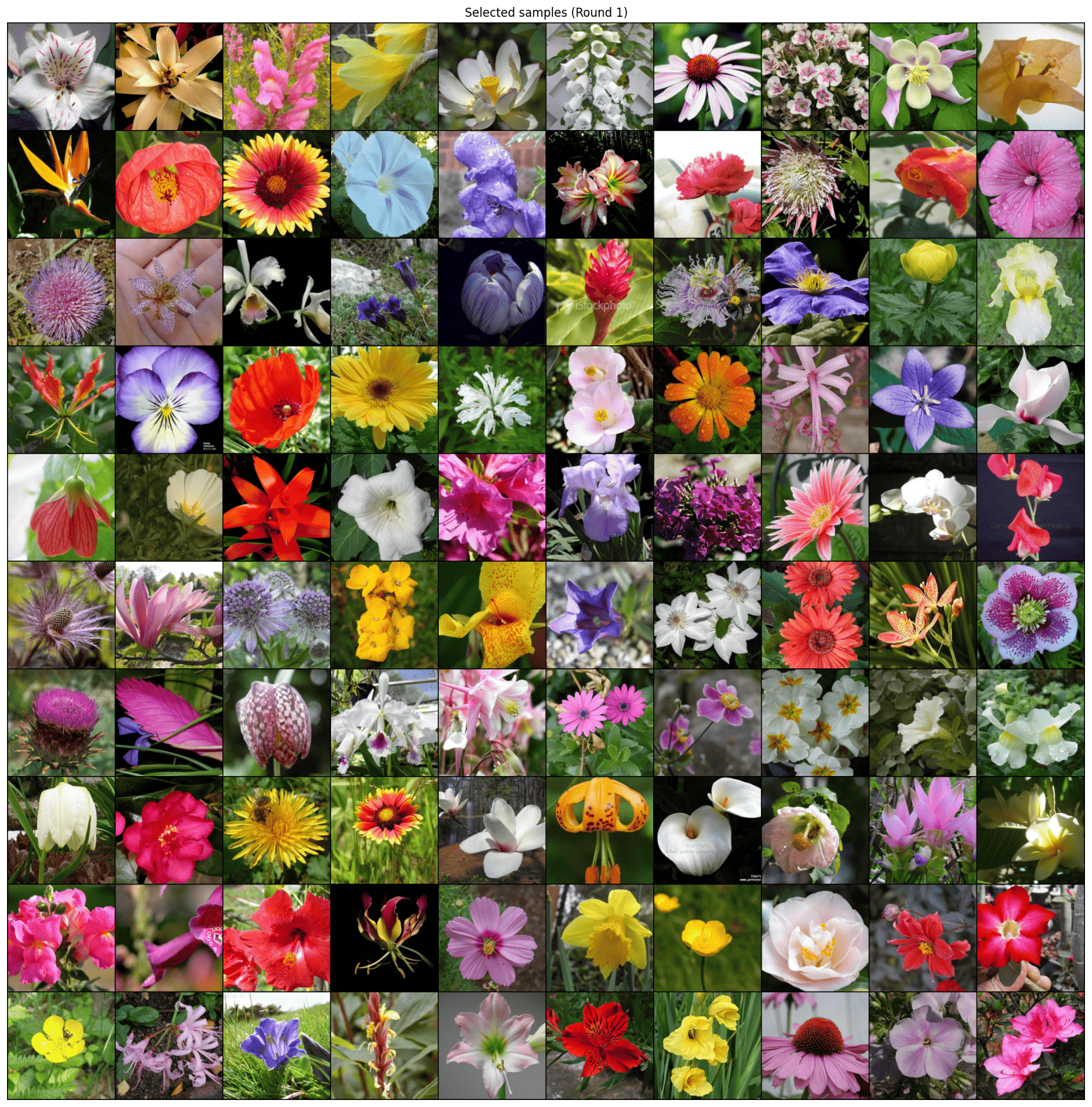}
    \caption{Visualization of samples selected by \textbf{\texttt{PCoreSet}} for Flowers102 dataset.}
    \label{fig:selected_samples_flowers}
\end{figure}

\begin{figure}[htbp]
    \centering
    \includegraphics[width=0.99\textwidth,trim={0 0 0 0.3in},clip]{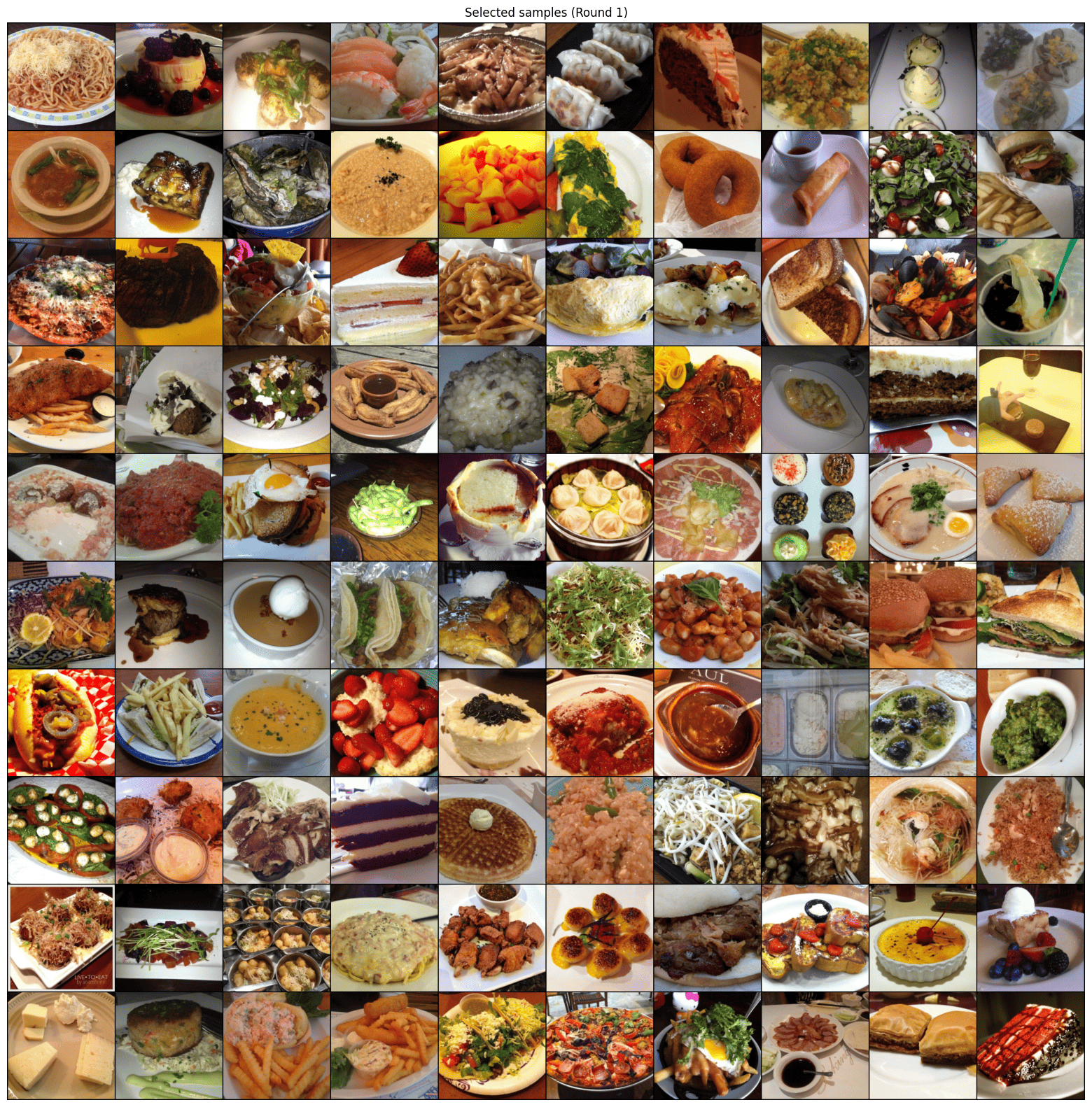}
    \caption{Visualization of samples selected by \textbf{\texttt{PCoreSet}} for Food-101 dataset.}
    \label{fig:selected_samples_food}
\end{figure}

\begin{figure}[htbp]
    \centering
    \includegraphics[width=0.99\textwidth,trim={0 0 0 0.3in},clip]{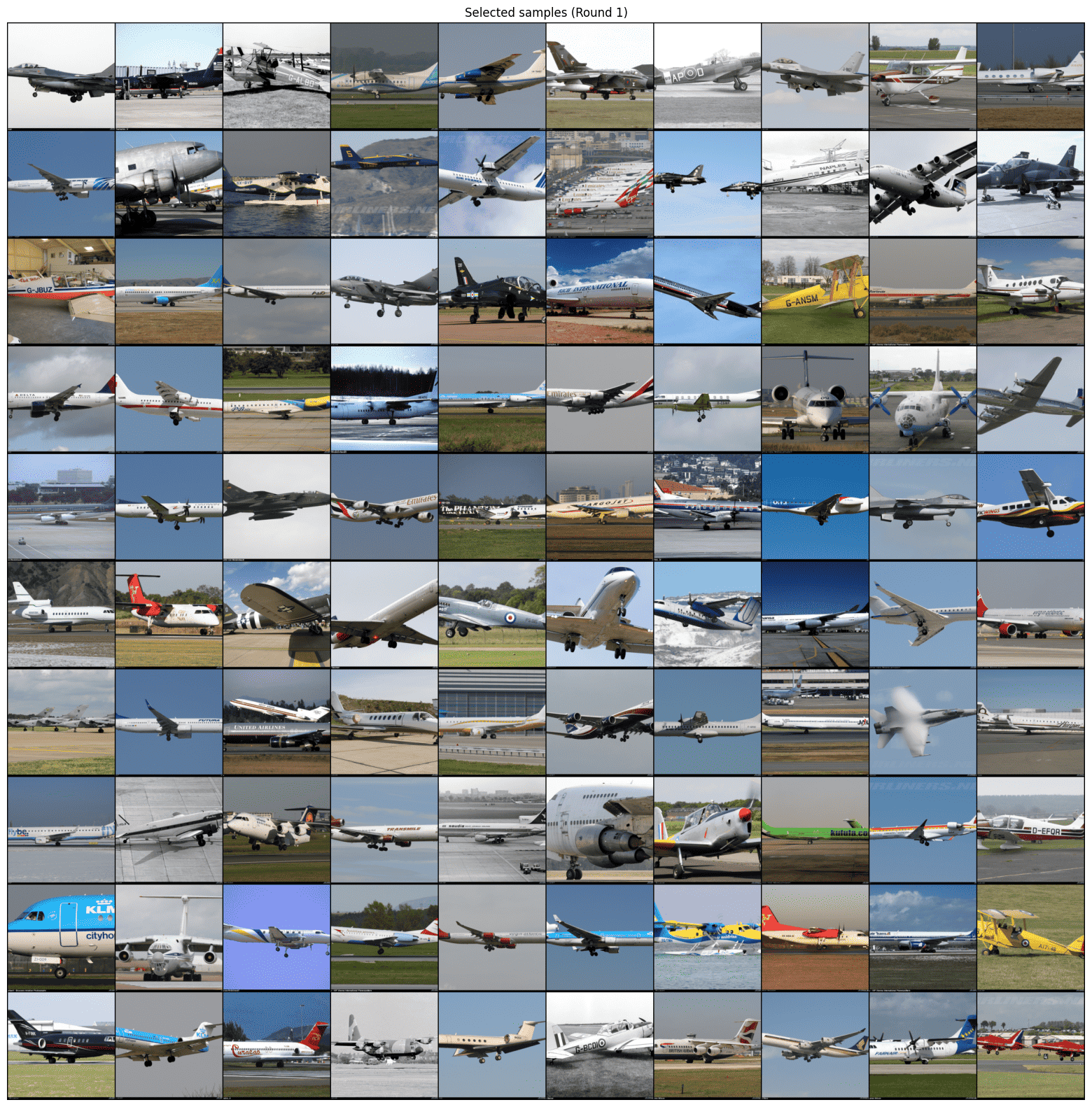}
    \caption{Visualization of samples selected by \textbf{\texttt{PCoreSet}} for FGVC dataset.}
    \label{fig:selected_samples_fgvc}
\end{figure}

\begin{figure}[htbp]
    \centering
    \includegraphics[width=0.99\textwidth,trim={0 0 0 0.3in},clip]{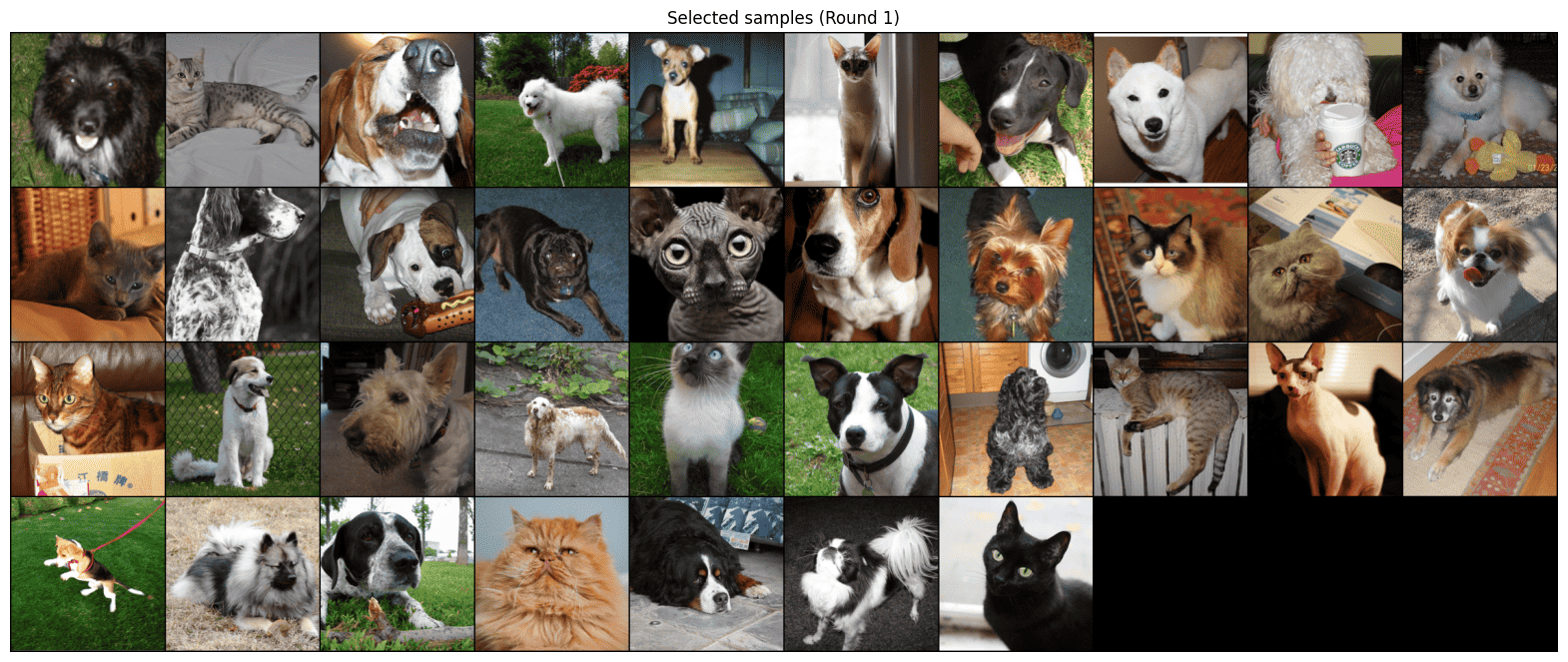}
    \caption{Visualization of samples selected by \textbf{\texttt{PCoreSet}} for Oxford-IIIT Pets dataset.}
    \label{fig:selected_samples_pets}
\end{figure}

\begin{figure}[htbp]
    \centering
    \includegraphics[width=0.99\textwidth,trim={0 0 0 0.3in},clip]{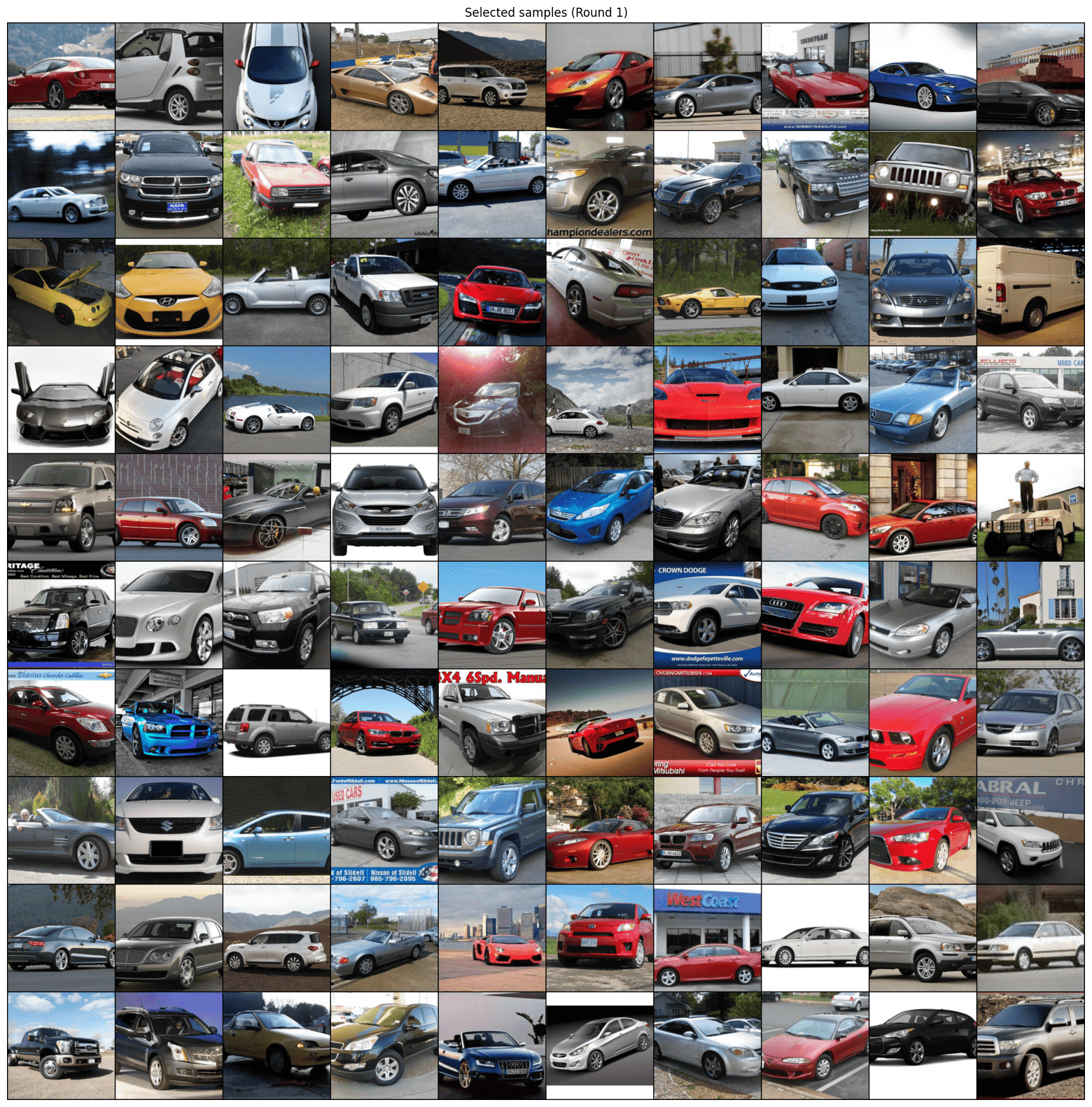}
    \caption{Visualization of samples selected by \textbf{\texttt{PCoreSet}} for Stanford Cars dataset.}
    \label{fig:selected_samples_cars}
\end{figure}

\begin{figure}[htbp]
    \centering
    \includegraphics[width=0.99\textwidth,trim={0 0 0 0.3in},clip]{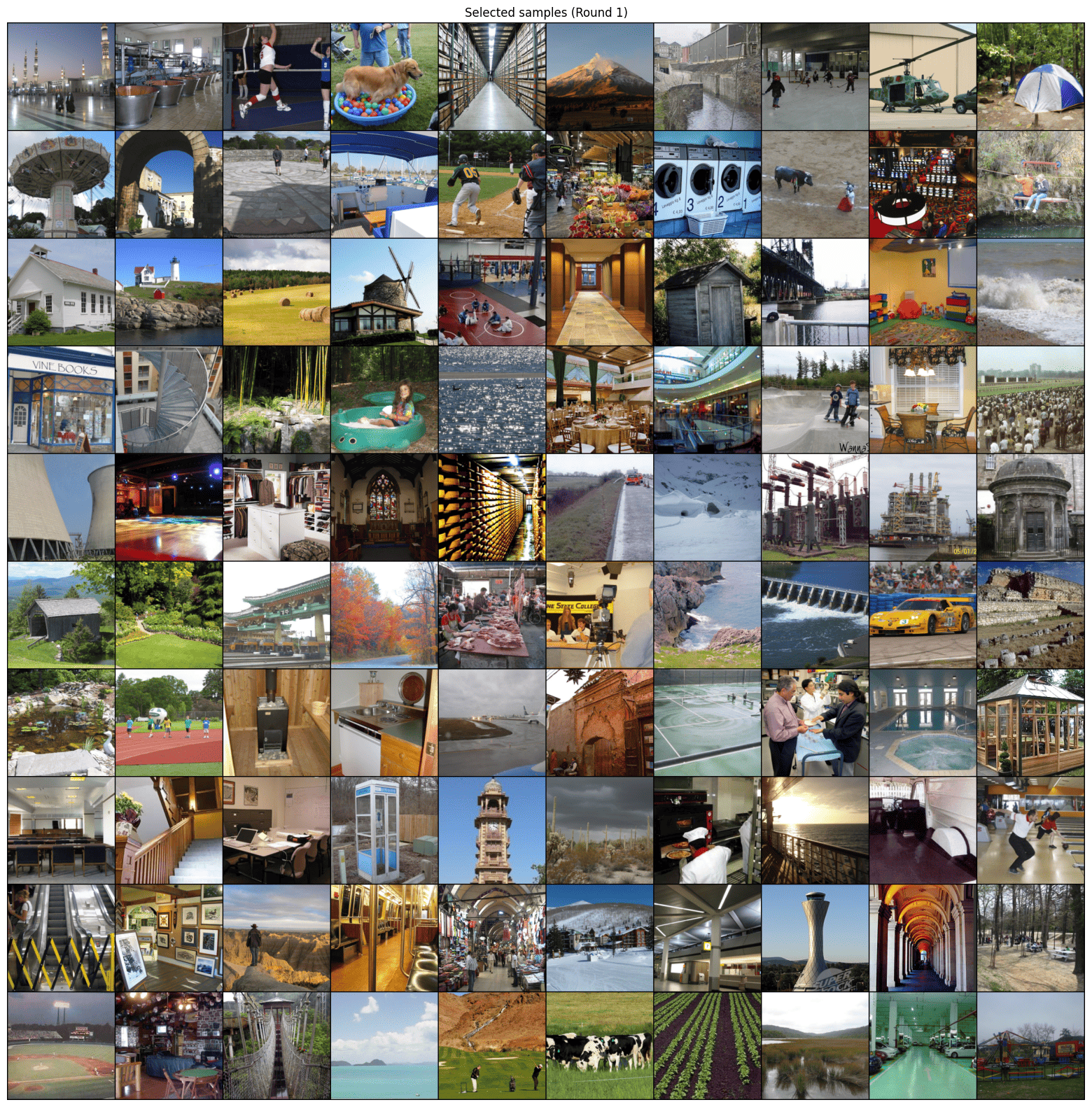}
    \caption{Visualization of samples selected by \textbf{\texttt{PCoreSet}} for SUN397 dataset.}
    \label{fig:selected_samples_sun}
\end{figure}

\begin{figure}[htbp]
    \centering
    \includegraphics[width=0.99\textwidth,trim={0 0 0 0.3in},clip]{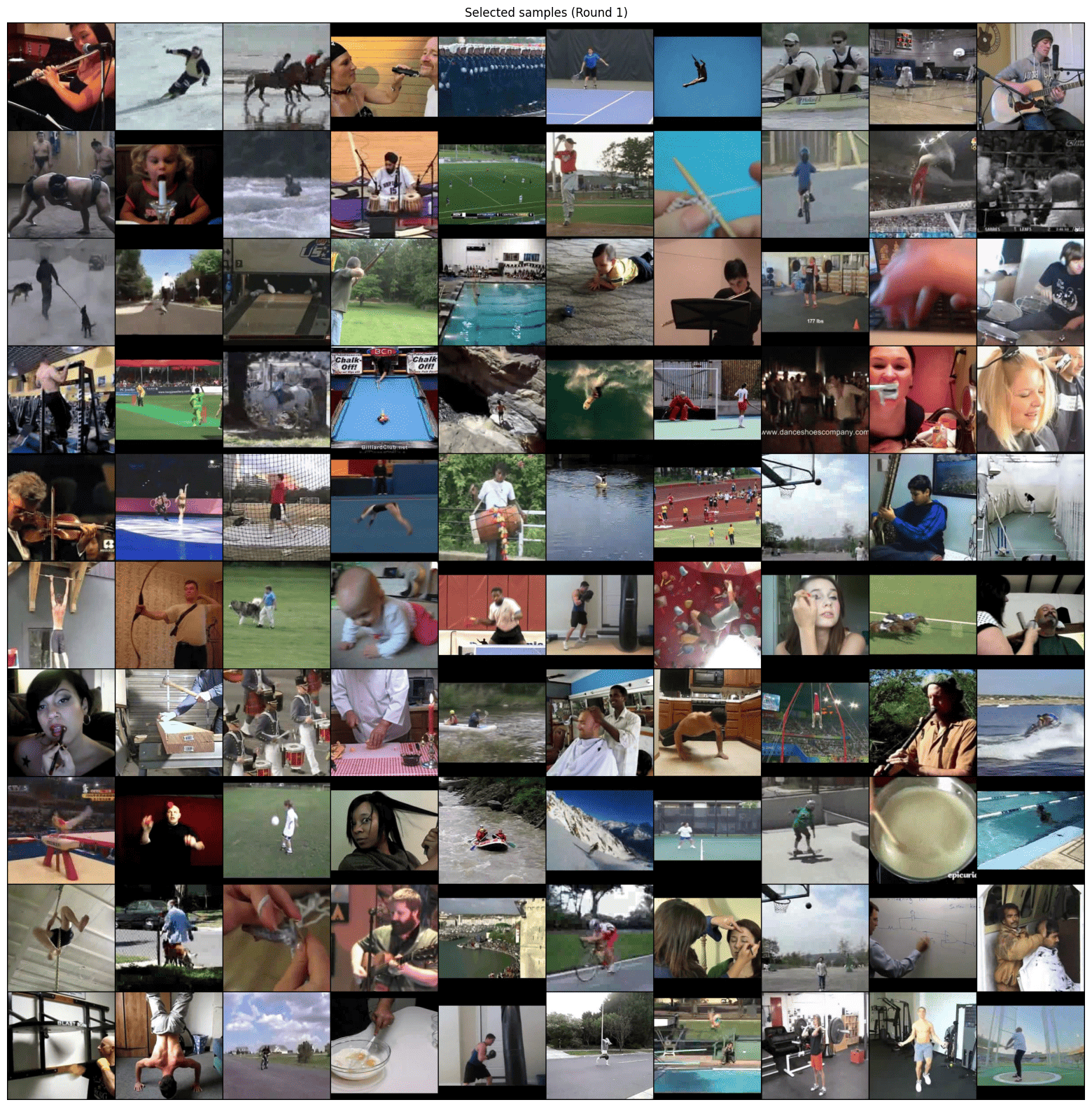}
    \caption{Visualization of samples selected by \textbf{\texttt{PCoreSet}} for UCF101 dataset.}
    \label{fig:selected_samples_ucf}
\end{figure}


\clearpage
\section{The Use of LLMs}\label{sec:llm}
We used LLMs solely for light editing such as correcting grammatical errors and polishing some words. They did not contribute to research ideation, experiments, analysis, or substantive writing.

\clearpage
\section*{NeurIPS Paper Checklist}

\begin{enumerate}

\item {\bf Claims}
    \item[] Question: Do the main claims made in the abstract and introduction accurately reflect the paper's contributions and scope?
    \item[] Answer: \answerYes{} 
    \item[] Justification: Our main claims (the ActiveKD framework, the structured prediction bias of VLM teachers, and the proposed \textbf{\texttt{PCoreSet}} selection strategy) are stated in the abstract and \S\ref{sec:introduction}, and are supported by the theoretical analysis in \Cref{sec:theoretical_analysis} and the empirical results in \Cref{sec:experiments}.

\item {\bf Limitations}
    \item[] Question: Does the paper discuss the limitations of the work performed by the authors?
    \item[] Answer: \answerYes{} 
    \item[] Justification: Limitations and future work are discussed in \Cref{sec:limitations} (deferred from \Cref{sec:conclusion} due to the page limit), covering the foundation-model scope, the visual-recognition task scope, and orthogonality to other KD/SSL methods.

\item {\bf Theory assumptions and proofs}
    \item[] Question: For each theoretical result, does the paper provide the full set of assumptions and a complete (and correct) proof?
    \item[] Answer: \answerYes{} 
    \item[] Justification: We state our key theoretical result (Proposition~\ref{prop:bias-propagation} on bias propagation through KD) in \Cref{sec:method}, and provide the full assumption (\Cref{assump:convergence}), proofs, and supporting results (Theorems~\ref{thm:optimal_distillation}--\ref{thm:approach_equiv}, Proposition~\ref{prop:bias-propagation_2}, and Corollary~\ref{cor:student_bias}) in \Cref{sec:theoretical_analysis}.

    \item {\bf Experimental result reproducibility}
    \item[] Question: Does the paper fully disclose all the information needed to reproduce the main experimental results of the paper to the extent that it affects the main claims and/or conclusions of the paper (regardless of whether the code and data are provided or not)?
    \item[] Answer: \answerYes{} 
    \item[] Justification: We used publicly available datasets and models for our experiments. Our proposed method is also simple and easy to implement, with described details in \Cref{sec:method}. We will also release our code and instructions for reproducing the results.

\item {\bf Open access to data and code}
    \item[] Question: Does the paper provide open access to the data and code, with sufficient instructions to faithfully reproduce the main experimental results, as described in supplemental material?
    \item[] Answer: \answerYes{} 
    \item[] Justification: We will provide open access to the data and code, with sufficient instructions to faithfully reproduce the main experimental results, once the blind review is over.
    \item[] Guidelines:
    \begin{itemize}
        \item The answer NA means that paper does not include experiments requiring code.
        \item Please see the NeurIPS code and data submission guidelines (\url{https://nips.cc/public/guides/CodeSubmissionPolicy}) for more details.
        \item While we encourage the release of code and data, we understand that this might not be possible, so ``No'' is an acceptable answer. Papers cannot be rejected simply for not including code, unless this is central to the contribution (e.g., for a new open-source benchmark).
        \item The instructions should contain the exact command and environment needed to run to reproduce the results. See the NeurIPS code and data submission guidelines (\url{https://nips.cc/public/guides/CodeSubmissionPolicy}) for more details.
        \item The authors should provide instructions on data access and preparation, including how to access the raw data, preprocessed data, intermediate data, and generated data, etc.
        \item The authors should provide scripts to reproduce all experimental results for the new proposed method and baselines. If only a subset of experiments are reproducible, they should state which ones are omitted from the script and why.
        \item At submission time, to preserve anonymity, the authors should release anonymized versions (if applicable).
        \item Providing as much information as possible in supplemental material (appended to the paper) is recommended, but including URLs to data and code is permitted.
    \end{itemize}

\item {\bf Experimental setting/details}
    \item[] Question: Does the paper specify all the training and test details (e.g., data splits, hyperparameters, how they were chosen, type of optimizer, etc.) necessary to understand the results?
    \item[] Answer: \answerYes{} 
    \item[] Justification: We specified all the training and test details (e.g., data splits, hyperparameters, how they were chosen, type of optimizer, etc.) necessary to understand the results in \Cref{sec:experiments_setup} and \Cref{sec:implementation} (with full configurations in \Cref{tab:implementation_details}).

\item {\bf Experiment statistical significance}
    \item[] Question: Does the paper report error bars suitably and correctly defined or other appropriate information about the statistical significance of the experiments?
    \item[] Answer: \answerYes{} 
    \item[] Justification: We report the mean performance with 95\% confidence intervals across 5 random seeds for all experiments in \Cref{sec:experiments} (see \Cref{sec:experiments_setup} for the protocol).

\item {\bf Experiments compute resources}
    \item[] Question: For each experiment, does the paper provide sufficient information on the computer resources (type of compute workers, memory, time of execution) needed to reproduce the experiments?
    \item[] Answer: \answerNo{} 
    \item[] Justification: We did not report explicit compute resources (type of compute workers, memory, total time of execution) in the paper. All implementation details necessary to reproduce our experiments---including model architectures, optimizer, batch size, training epochs per AL round, and all hyperparameters---are provided in \Cref{sec:implementation} and \Cref{tab:implementation_details}.

\item {\bf Code of ethics}
    \item[] Question: Does the research conducted in the paper conform, in every respect, with the NeurIPS Code of Ethics \url{https://neurips.cc/public/EthicsGuidelines}?
    \item[] Answer: \answerYes{} 
    \item[] Justification: We have reviewed the NeurIPS Code of Ethics and will follow it in our research.

\item {\bf Broader impacts}
    \item[] Question: Does the paper discuss both potential positive societal impacts and negative societal impacts of the work performed?
    \item[] Answer: \answerNA{} 
    \item[] Justification: Our work studies active learning with knowledge distillation on standard academic image classification datasets. We use only publicly available datasets and models, and the proposed selection strategy does not present obvious negative societal impacts.

\item {\bf Safeguards}
    \item[] Question: Does the paper describe safeguards that have been put in place for responsible release of data or models that have a high risk for misuse (e.g., pretrained language models, image generators, or scraped datasets)?
    \item[] Answer: \answerNA{} 
    \item[] Justification: Our work focuses on an active-learning sample-selection strategy combined with knowledge distillation on standard academic image classification datasets, and does not release any pretrained models or scraped data with high risk for misuse.

\item {\bf Licenses for existing assets}
    \item[] Question: Are the creators or original owners of assets (e.g., code, data, models), used in the paper, properly credited and are the license and terms of use explicitly mentioned and properly respected?
    \item[] Answer: \answerYes{} 
    \item[] Justification: We properly credit all creators of assets used in our work.

\item {\bf New assets}
    \item[] Question: Are new assets introduced in the paper well documented and is the documentation provided alongside the assets?
    \item[] Answer: \answerNA{} 
    \item[] Justification: Our work does not release any new assets such as datasets, code, or models.

\item {\bf Crowdsourcing and research with human subjects}
    \item[] Question: For crowdsourcing experiments and research with human subjects, does the paper include the full text of instructions given to participants and screenshots, if applicable, as well as details about compensation (if any)?
    \item[] Answer: \answerNA{} 
    \item[] Justification: Our work does not involve crowdsourcing or research with human subjects. We use standard academic datasets and focus on active learning methods.

\item {\bf Institutional review board (IRB) approvals or equivalent for research with human subjects}
    \item[] Question: Does the paper describe potential risks incurred by study participants, whether such risks were disclosed to the subjects, and whether Institutional Review Board (IRB) approvals (or an equivalent approval/review based on the requirements of your country or institution) were obtained?
    \item[] Answer: \answerNA{} 
    \item[] Justification: Our work does not involve crowdsourcing or research with human subjects.

\item {\bf Declaration of LLM usage}
    \item[] Question: Does the paper describe the usage of LLMs if it is an important, original, or non-standard component of the core methods in this research? Note that if the LLM is used only for writing, editing, or formatting purposes and does not impact the core methodology, scientific rigorousness, or originality of the research, declaration is not required.
    \item[] Answer: \answerNA{} 
    \item[] Justification: Our work does not involve any LLMs as important, original, or non-standard components.
    \item[] Guidelines:
    \begin{itemize}
        \item The answer NA means that the core method development in this research does not involve LLMs as any important, original, or non-standard components.
        \item Please refer to our LLM policy (\url{https://neurips.cc/Conferences/2026/LLM}) for what should or should not be described.
    \end{itemize}

\end{enumerate}

\end{document}